\documentclass{article} 
\usepackage[T1]{fontenc}
\usepackage{arxiv}
\usepackage{graphicx}
\usepackage{fullpage}
\usepackage{amsmath,bm,bbm}
\usepackage{enumerate}
\usepackage{amssymb}
\usepackage{enumitem}
\usepackage{mathrsfs}
\usepackage{mathtools}
\usepackage{amsthm}
\usepackage[sort]{natbib}
\usepackage{algorithm}
\usepackage{algorithmicx}
\usepackage[dvipsnames]{xcolor}
\usepackage{algpseudocode}
\usepackage[colorlinks,citecolor=blue,urlcolor=blue]{hyperref}
\usepackage{hyperref}
\usepackage{geometry}
\usepackage{authblk}
\usepackage{pdfpages}
\usepackage{mdframed}
\usepackage{eqparbox}
\usepackage[all]{xy}
\usepackage[table]{xcolor}
\usepackage{cleveref}
\usepackage{setspace}

\newlength{\dhatheight}

\newcommand{\h}{h} 

\DeclareSymbolFont{bbold}{U}{bbold}{m}{n}
\DeclareSymbolFontAlphabet{\mathbbold}{bbold}

\newcommand{\E}{\mathbb{E}}

\newcommand{\argmax}{\mathop{\rm argmax}}
\newcommand{\argmin}{\mathop{\rm argmin}}

\newcommand{\ignore}[1]{}
\newcommand{\oldstuff}[1]{}

\newcommand{\err}{\mathrm{er}}

\newcommand{\cX}{\mathcal{X}}

\newcommand{\ER}[1]{\mathrm{ER}\left(#1\right)}
\newcommand{\ERin}[1]{\mathrm{ER}(#1)}

\newtheorem*{itheorem}{Informal Theorem}

\newsavebox{\savepar}

\makeatletter
\newcommand{\vast}{\bBigg@{3}}
\newcommand{\Vast}{\bBigg@{4}}
\makeatother

\newtheorem{lemma}{Lemma}
\newtheorem{theorem}{Theorem}

\newtheorem{definition}{Definition}
\newtheorem{remark}{Remark}

\renewcommand{\Pr}{\mathbb{P}}

\newcommand\naturalnumber{\mathbb{N}}
\newcommand\hpc{\mathcal{H}}

\newcommand\vcdim{\mathrm{VC}(\mathcal{H})}

\def\ddefloop#1{\ifx\ddefloop#1\else\ddef{#1}\expandafter\ddefloop\fi}
\def\ddef#1{\expandafter\def\csname v#1\endcsname{\ensuremath{\boldsymbol{#1}}}}
\ddefloop abcdefghijklmnopqrstuvwxyzABCDEFGHIJKLMNOPQRSTUVWXYZ\ddefloop
\def\ddef#1{\expandafter\def\csname v#1\endcsname{\ensuremath{\boldsymbol{\csname #1\endcsname}}}}
\ddefloop {alpha}{beta}{gamma}{delta}{epsilon}{varepsilon}{zeta}{eta}{theta}{var
theta}{iota}{kappa}{lambda}{mu}{nu}{xi}{pi}{varpi}{rho}{varrho}{sigma}{varsigma}
{tau}{upsilon}{phi}{varphi}{chi}{psi}{omega}{Gamma}{Delta}{Theta}{Lambda}{Xi}{Pi
}{Sigma}{varSigma}{Upsilon}{Phi}{Psi}{Omega}{ell}\ddefloop
	
\newlength\oversetwidth
\newlength\underwidth

\title{Optimal Learning Under Tsybakov Noise}

\author{\textbf{Steve Hanneke \quad Hongao Wang \quad Mingyue Xu} \\
Department of Computer Science\\
Purdue University\\

\texttt{steve.hanneke@gmail.com}\\

\texttt{\{wang5270, xu1864\}@purdue.edu}
}
\date{}

\begin{document}
\maketitle

\begin{abstract}
Probably Approximately Correct (PAC) learning is a fundamental learning model that has been extensively investigated since it was introduced by Valiant \cite{valiant1984theory} to the theoretical computer science community. 
In this model, $\hpc$ is a concept class consisting of functions mapping from an instance space $\mathcal{X}$ to binary labels $\{0,1\}$, and $h^{*}\in\hpc$ is the target concept to be learned. 
Having access to i.i.d.\ labeled examples from a distribution $\mathcal{D}$ over $\mathcal{X}\times\{0,1\}$, which admits $h^*$ as the best concept in $\hpc$, the goal is to design a learning algorithm that outputs a hypothesis having low error competitive to $h^{*}$ with high probability.
This model was initially studied under the realizable setting, which may be considered unrealistic, as it assumes that the instances are labeled exactly by $h^*$. 
A natural relaxation is to allow label noise, that is, the true label can be flipped with probability $\eta\in(0,1/2)$. 
The simplest noise model assumes that $\eta$ is bounded away from $1/2$ on every instance, known as Massart noise. 
However, in reality, certain labels might be extremely noisy, especially for those points near the decision boundary. 
Hence, it is natural to allow very noisy points to exist, though perhaps only rarely.
This is quantified by a noise model introduced by Mammen and Tsybakov \cite{mammen1999smooth} and Tsybakov \cite{tsybakov2004optimal}, now known as Tsybakov noise.

For learning general concept classes, Massart and N\'{e}d\'{e}lec \cite{massart2006risk} gave the first general upper bound for error guarantees under Tsybakov noise, analyzing the popular ERM learning algorithm. They also proved a general lower bound on the achievable guarantees. However, their upper and lower bounds differ by a logarithmic factor. Resolving this gap has remained a well-known open question for the past twenty years.

In this work, we resolve this open question by improving the upper bound to match the best known lower bound, thus establishing the optimal error guarantee for learning under Tsybakov noise. Our learning algorithm operates by adaptively partitioning the instance space into regions, roughly corresponding to different noise levels, and returning a hypothesis in the concept class satisfying a specific error constraint for each region. Our technique shares a conceptual foundation with several recent advances in non-realizable learning, such as \cite{hanneke2024revisiting} and \cite{hanneke2025agnostic}.  
\end{abstract}

\newpage
\section{Introduction}
  \label{sec:introduction}
Probably Approximately Correct (PAC) learning, dating back to the early works of \cite{vapnik1974theory,vapnik1964class,valiant1984theory}, is one of the most fundamental models for understanding supervised statistical learning. Under this framework, one of the most basic problems is the binary classification. Specifically, we are given i.i.d.\ data $S_{n}=\{(x_{1},y_{1}),\ldots,(x_{n},y_{n})\}\sim\mathcal{D}^{n}$ generated from some unknown data distribution $\mathcal{D}$ over $\mathcal{X}\times\{0,1\}$, where $\mathcal{X}$ is called the instance space (or input domain). Equipped with the knowledge of a concept class $\hpc\subseteq\{0,1\}^{\mathcal{X}}$, the goal is to design a learning algorithm that takes the dataset $S_{n}$ as input and outputs a hypothesis $\hat{h}_{n}:\mathcal{X}\mapsto\{0,1\}$ that achieves low excess error competitive to the best concept $h^{*}_{\mathcal{D}}$ satisfies $\mathrm{er}_{\mathcal{D}}(h^{*}_{\mathcal{D}})=\inf_{h\in\hpc}\mathrm{er}_{\mathcal{D}}(h)=\inf_{h\in\hpc}\Pr_{(x,y)\sim\mathcal{D}}(h(x)\neq y)$. To formalize the problem mathematically, for any $\epsilon>0$ and $\delta\in(0,1)$, we aim to design a learning algorithm $\mathcal{A}$ to achieve $\Pr_{S_{n}\sim\mathcal{D}^{n}}(\mathrm{er}_{\mathcal{D}}(\mathcal{A}(S_{n}))-\mathrm{er}_{\mathcal{D}}(h^{*}_{\mathcal{D}})>\epsilon)<\delta$. From a statistical perspective, a learning algorithm's performance can be evaluated by its learning rates (of convergence), namely, given $n$ i.i.d.\ samples $S_{n}$ and $\delta\in(0,1)$, what is the best excess risk $\epsilon=\epsilon(n,\delta)$ such that $\Pr_{S_{n}\sim\mathcal{D}^{n}}(\mathrm{er}_{\mathcal{D}}(\mathcal{A}(S_{n}))-\mathrm{er}_{\mathcal{D}}(h^{*}_{\mathcal{D}})>\epsilon)<\delta$ holds. 

\paragraph{Realizable learning.} The realizable setting is the most basic case of PAC learning, where distribution $\mathcal{D}$ is restricted to be realizable with respect to $\hpc$, i.e., $\mathrm{er}_{\mathcal{D}}(h^{*}_{\mathcal{D}})=\inf_{h\in\hpc}\mathrm{er}_{\mathcal{D}}(h)=0$. For any distribution $\mathcal{D}$ over $\mathcal{X}\times\{0,1\}$, define the Bayes (optimal) classifier with respect to $\mathcal{D}$, denoted by $h^{*}_{\mathcal{D}}$, as $h^{*}_{\mathcal{D}}(x)=\mathbbm{1}(\Pr_{\mathcal{D}}(Y=1|X=x)\geq 1/2)$, where $\mathbbm{1}(\cdot)$ is the indicator function. Hence, the realizability assumption is equivalent to saying that the Bayes classifier $h^{*}_{\mathcal{D}}$ has zero error and lies in the concept class $\hpc$. This simplifies our goal to bounding the error rate $\mathrm{er}_{\mathcal{D}}(\mathcal{A}(S_{n}))\leq\epsilon$. 

Over the past few decades, an abundant line of works has focused on developing learning algorithms that can guarantee minimax optimal learning rates under the realizable setting, formalizing an almost exhaustive understanding of its statistical limits of convergence. Perhaps, the most natural family of learning algorithms is the celebrated Empirical Risk Minimization (ERM) rule. An ERM algorithm outputs a concept $\hat{h}_{n}\in\hpc$ that has a zero empirical error (breaking ties arbitrarily), that is, $\hat{\mathrm{er}}_{S_{n}}(\hat{h}_{n})=\frac{1}{n}\sum_{(x,y)\in S_{n}}\mathbbm{1}(\hat{h}_{n}(x)\neq y)=0$. Note that such a concept exists since any data sequence is realizable. For any concept class $\hpc$ with finite VC dimension $\vcdim=d$, which is defined as the largest number of points $x_{1},\ldots,x_{d}\in\mathcal{X}$ such that $\hpc$ contains concepts to realize all $2^{d}$ possible labelings of them, \cite{vapnik1974theory,blumer1989learnability} showed that any sample-consistent learning rules (including ERM) can at best achieve $O((d/n)\log{(n/d)}+(1/n)\log(1/\delta))$ learning rates. This upper bound was complemented by a general lower bound \citep{ehrenfeucht1989general}, saying that for any learning algorithm $\mathcal{A}$ (not necessarily sample-consistent), there exists a realizable distribution $\mathcal{D}$ such that $\mathrm{er}_{\mathcal{D}}(\mathcal{A}(S_{n}))=\Omega((d/n)+(1/n)\log(1/\delta))$, with probability greater than $\delta$. Clearly, there is a logarithmic factor gap between the lower and upper bounds, which was shown to be sometimes unavoidable by sample-consistent learning rules \citep{auer2007new} and even by proper learners \citep{bousquet2020proper}, i.e., any algorithm outputs a concept in $\hpc$. A surge of follow-up works then contributed to resolving such a gap for special cases or proposing sometimes better upper bounds (e.g., \citep{haussler1994predicting,natarajan1987learning,warmuth2004optimal,gine2006concentration,simon2015almost,hanneke2015minimax,hanneke2016refined}). We would like to mention that a complete review of this literature can be found in \cite{hanneke2015minimax}. Resolving such a gap had been a really challenging open problem for decades and was finally addressed by the seminal work of \cite{hanneke2016optimal}, which sharpened the upper bound to $O((d/n)+(1/n)\log(1/\delta))$. Hanneke's algorithm returns the majority vote of ERM classifiers trained on carefully and recursively selected sub-sampling datasets, which is indeed an improper learner. Recently, several works proposed arguably simpler learning algorithms to guarantee optimal PAC learning rates, including the classic heuristic known as Bagging or bootstrap aggregation \citep{larsen2023bagging}, a variant of the one-inclusion-graph algorithm \cite{aden2023optimal}, as well as an inspiring conjecture\footnote{\cite{aden2024majority} proved that the majority vote of three ERM classifiers achieves an optimal in-expectation error bound, and conjectured that it is in fact also optimal in the high-probability regime.} on a simple majority vote of three ERMs \citep{aden2024majority}.

\paragraph{Learning with label noise.} The assumption of realizability is usually too restrictive to hold in practice. A more general setting of learning in the presence of various classification noises is thus motivated and has become a crucial challenge in machine learning. Let us define the regression function as $\eta_{\mathcal{D}}(x)=\Pr_{\mathcal{D}}(Y=1|X=x)$. It is clear that the Bayes classifier $h^{*}_{\mathcal{D}}(x)=\max(2\eta_{\mathcal{D}}(x)-1,0)$ for any $x\in\mathcal{X}$. Now, let us introduce those noise models commonly studied in the literature (formal definitions can be found later). The simplest noise model beyond the realizable case is called the Massart (bounded) noise condition \citep{massart2006risk,sloan1988types}, guaranteeing that the probability that labels are flipped is uniformly bounded away from $1/2$, i.e., for any $x\in\mathcal{X}$, $|\eta_{\mathcal{D}}(x)-1/2|\geq\gamma$ for some $\gamma\in(0,1/2)$. While still assuming $h^{*}_{\mathcal{D}}\in\hpc$, the Massart noise induces a non-zero error to the Bayes classifier. The classic Random Classification Noise (RCN) model \citep{angluin1988learning}, described as that each label is flipped independently with probability exactly $1/2-\gamma$, is a special case of the Massart noise model. 

However, this assumption of uniformly bounded noise still fails to capture certain realistic noise sources. Consider in computer vision or NLP, it is very likely that some examples (blurry images, ambiguous text, etc.) are intrinsically much harder to annotate than others, making the bounded noise condition too rigid to hold. This motivates the proposal of the Tsybakov noise model, which was first proposed in \cite{mammen1999smooth} and subsequently refined in \cite{tsybakov2004optimal}. While keeping $h^{*}_{\mathcal{D}}\in\hpc$, the Tsybakov noise condition significantly extends the Massart noise condition by allowing a small fraction of arbitrarily large noises while having the total noise under control. Specifically, it guarantees that $\Pr_{x\sim\mathcal{D}_{X}}(|\eta_{\mathcal{D}}(x)-1/2|\leq t)=O(t^{\alpha/(1-\alpha)})$ for any $t>0$ and some $\alpha\in(0,1)$ parametrizing the level of noise in label distribution. To see how the Tsybakov noise generalizes the Massart noise, think about the classic logistic regression problem $\eta_{\mathcal{D}}(x)=\sigma(\langle w,x\rangle+b)$, where $\sigma(t)=(1+e^{-t})^{-1}$ is called the logistic (or sigmoid) function. The Tsybakov noise condition only requires $\langle w,x\rangle+b$ to have a small probability near zero, while the Massart noise condition requires $\langle w,x\rangle+b$ to stay uniformly away from zero almost surely, thus excluding a broad family of natural distributions. It is clear that the Tsybakov noise model becomes more challenging as the parameter $\alpha$ decreases. Indeed, taking $\alpha=0$ (also removing the assumption of $h^{*}_{\mathcal{D}}\in\hpc$) leads to the most general agnostic setting \citep{haussler1992decision,kearns1992toward}. The Tsybakov noise model thus interpolates between the realizable and agnostic settings. 

Later works \cite{bartlett2004local,bartlett2006convexity} further extended the Tsybakov noise condition to a more general Bernstein class condition. Since these initial works, there has been a substantial literature in learning theory and statistics built towards understanding the statistical limits of the aforementioned noise models under various scenarios (for the Massart noise condition, e.g., \cite{awasthi2015efficient,zhang2020efficient,hanneke2015minimax,hanneke2016refined,diakonikolas2019distribution,diakonikolas2020learning}, for the Tsybakov noise condition, e.g., \citep{boucheron2005theory,bartlett2006empirical,massart2006risk,balcan2007margin,hanneke2011rates,hanneke2015minimax,van2015fast,hanneke2016refined,diakonikolas2021efficiently}, and for the Bernstein class condition, e.g., \citep{bartlett2006convexity,massart2006risk,van2015fast,hanneke2019value,hanneke2022no,hanneke2026more}).

\subsection{Our contributions} 
  \label{subsec:our-contribution}
In this paper, we focus on understanding the minimax optimal learning rates when the labels are corrupted by Tsybakov noise. To formalize the problem setup, assume that i.i.d.\ samples $S_{n}=\{(x_{1}, y_{1}),\ldots,(x_{n}, y_{n})\}$ are generated from some unknown distribution $\mathcal{D}$ that satisfies a Tsybakov noise condition with parameters $(a,\alpha)$ with respect to a concept class $\hpc$. Below, we formally state its definition.

\begin{definition}  [\textbf{Tsybakov Noise Condition}]
  \label{def:Tsybakov-noise-condition}
Let $\alpha\in(0,1)$ and $a>0$. We say that a distribution $\mathcal{D}$ over $\mathcal{X}\times\{0,1\}$ satisfies a Tsybakov noise condition with parameters $(a,\alpha)$ w.r.t.\ $\hpc$, denoted by $\mathcal{D}\in\mathrm{TN}_{\hpc}(a,\alpha)$, if $h^{*}_{\mathcal{D}}\in\hpc$ and
\begin{equation*}
    \Pr_{\mathcal{D}}\left(|\eta_{\mathcal{D}}(x)-1/2| \leq t\right) \leq a't^{\frac{\alpha}{1-\alpha}} ,\;\; \forall t>0 ,
\end{equation*}
where $a' = (1-\alpha)(2\alpha)^{\alpha/(1-\alpha)}a^{1/(1-\alpha)}$\footnote{We will use the definition of $a'$ in all of the following sections of this paper.}.
\end{definition}

In this paper, we are particularly interested in passively learning a concept class $\hpc$ with $\vcdim=d$. The best known lower bound is $\mathrm{er}_{\mathcal{D}}(\mathcal{A}(S_{n}))-\mathrm{er}_{\mathcal{D}}(h^{*}_{\mathcal{D}})=\Omega((d/n+\log(1/\delta)/n)^{1/(2-\alpha)})$, with probability at least $1-\delta$. Such a type of lower bounds has been developed in \cite{tsybakov2004optimal,massart2006risk,audibert2007fast} from a statistical/nonparametric perspective. A clear statement of the lower bound for learning a VC class can be found in \cite[Section 6]{hanneke2015minimax}. We include it here for completeness:
\begin{equation*}
    \frac{1}{\epsilon^{2-\alpha}}\left(d+\log\left(\frac{1}{\delta}\right)\right) \lesssim \mathcal{M}_{\mathrm{TN}(a,\alpha)}(\epsilon,\delta) \lesssim \frac{1}{\epsilon^{2-\alpha}}\left(d\log\left(\frac{1}{a\epsilon^{\alpha}}\right)+\log\left(\frac{1}{\delta}\right)\right) ,
\end{equation*}
where $\mathcal{M}_{\mathrm{TN}(a,\alpha)}(\epsilon,\delta)$ denote the sample complexity for passively learning $\hpc$ to guarantee $(\epsilon,\delta)$ excess risk under Tsybakov noise. Solving the above for $\epsilon$ directly yields that
\begin{equation*}
    \left(\frac{d}{n}+\frac{1}{n}\log\left(\frac{1}{\delta}\right)\right)^{\frac{1}{2-\alpha}} \lesssim \min_{\mathcal{A}}\max_{\mathcal{D}\in\mathrm{TN}_{\hpc}(a,\alpha)}\mathrm{er}_{\mathcal{D}}(\mathcal{A}(S_{n}))-\mathrm{er}_{\mathcal{D}}(h^{*}_{\mathcal{D}}) \lesssim \left(\frac{d}{n}\log\left(\frac{n}{d}\right)+\frac{1}{n}\log\left(\frac{1}{\delta}\right)\right)^{\frac{1}{2-\alpha}} ,
\end{equation*}
with probability at least $1-\delta$. In fact, the upper bound is realized by the natural ERM learning algorithm Let $\mathrm{ERM}(\hpc;S)=\{h\in\hpc: h=\argmin_{h\in\hpc}\hat{\mathrm{er}}_{S}(h)\}$ denote the class of empirical risk minimizers of $\hpc$ on sample $S$. We have the following result (see Lemma~\ref{lem:excess-risk-under-tsybakov-noise} for a formal version and its proof).

\begin{lemma}  [\textbf{ERM Bound under Tsybakov Noise, \cite{massart2006risk}}]
Let $\mathcal{X}$ be an instance space, $\hpc\subseteq\{0,1\}^{\mathcal{X}}$ be a concept class with $\vcdim=d$. Let $a>0$ and $\alpha\in(0,1)$. Assume that $\mathcal{D}$ satisfies a Tsybakov noise condition with parameters $(a,\alpha)$ w.r.t.\ $\hpc$. Let $h^{*}_{\mathcal{D}}\in\hpc$ be the Bayes-optimal classifier w.r.t.\ $\mathcal{D}$. Let $n\in\naturalnumber$ and $S_{n}=\{(x_{i}, y_{i})\}_{i=1}^{n}\sim\mathcal{D}^{n}$ be i.i.d.\ samples. Then, for any $\delta\in(0,1)$, we have
\begin{equation*}
    \sup_{h\in\mathrm{ERM}(\hpc;S_{n})}\mathrm{er}_{\mathcal{D}}(h) - \mathrm{er}_{\mathcal{D}}(h^{*}_{\mathcal{D}}) \leq \left(a\left(\frac{d}{n}\log\left(\frac{n}{a^2d}\right)+\frac{1}{n}\log\left(\frac{1}{\delta}\right)\right)\right)^{\frac{1}{2-\alpha}} ,
\end{equation*}
for some constant $c=c(a,\alpha)>0$, with probability at least $1-\delta$.
\end{lemma}

Clearly, the sharpest known lower and upper bounds on the minimax optimal learning rates have a logarithmic factor gap. Indeed, this question of what the minimax optimal learning rates are and what learning algorithms can achieve such rates has remained open for decades.

In this paper, we resolve this open question by providing a learning algorithm $\mathcal{A}$ that guarantees an excess risk bound of $\mathrm{er}_{\mathcal{D}}(\mathcal{A}(S_{n}))-\mathrm{er}_{\mathcal{D}}(h^{*}_{\mathcal{D}})=O((d/n+\log(1/\delta)/n)^{1/(2-\alpha)})$, with probability at least $1-\delta$. Our upper bound, together with the best-known lower bound, yields that the optimal minimax optimal learning rates $\mathrm{er}_{\mathcal{D}}(\mathcal{A}(S_{n}))-\mathrm{er}_{\mathcal{D}}(h^{*}_{\mathcal{D}})=\Theta((d/n+\log(1/\delta)/n)^{1/(2-\alpha)})$, or equivalently, the optimal sample complexity is $\mathcal{M}_{\mathrm{TN}(a,\alpha)}(\epsilon,\delta)=\Theta(\epsilon^{\alpha-2}(d+\log(1/\delta)))$. Specifically, we show that

\begin{itheorem}
Let $\mathcal{X}$ be an instance space, $\hpc\subseteq\{0,1\}^{\mathcal{X}}$ be a concept class with $\vcdim=d$. Let $\alpha\in(0,1)$, $a>0$. For any $\mathcal{D}\in\mathrm{TN}_{\hpc}(a,\alpha)$ and any $\delta\in(0,1)$, there exists a learning algorithm $\mathcal{A}$ such that, for sufficiently large $n\in\naturalnumber$ and $S_{n}=\{(x_i,y_i)\}_{i=1}^{n}\sim\mathcal{D}^{n}$, with probability at least $1-\delta$,
\begin{equation*}
    \err_{\mathcal{D}}\left(\mathcal{A}(S_{n})\right) - \err_{\mathcal{D}}\left(h^{*}_\mathcal{D}\right) \leq C\left(\frac{1}{n}\left(d+\log\left(\frac{1}{\delta}\right)\right)\right)^{\frac{1}{2-\alpha}},
\end{equation*}
for some constant $C=C(a,\alpha)>0$.
\end{itheorem}

Perhaps surprisingly, our new learning algorithm is a \textbf{\emph{proper}} learner. For comparison, optimal excess risks can not be achieved by any proper learner under the realizable setting as discussed before, but can indeed be achieved by proper learners (e.g., ERM) under the agnostic setting. The core behind our new algorithm, named `\textbf{M}assart'' \textbf{E}rror \textbf{R}egions \textbf{I}solation under \textbf{T}sybakov noise (\textbf{MERIT}), is to recursively isolate error regions that nearly satisfy Massart noise conditions with a carefully quantified scheduling of noise levels. A similar strategy of adaptively isolating variance regions has been adopted in the work of \cite{hanneke2025agnostic}, where the author proved the minimax optimal first-order query complexity of agnostic active learning. We will elaborate on this idea in Section~\ref{sec:technical-overview} and hope it finds further applications in learning theory and statistics.

\subsection{Paper outline} 
  \label{subsec:paper-outline}
The rest of the paper is organized as follows. In Section~\ref{sec:preliminaries}, we formally describe the problem setup of learning under Tsybakov noise and introduce required definitions. In Section~\ref{sec:technical-overview}, we present a thorough technical overview of our main results, containing a conceptual proof strategy as well as certain insights into our algorithm design. In Section~\ref{sec:main-results}, we present our optimal learning algorithm under Tsybakov noise, together with a complete proof of its performance guarantee. Finally, in Section~\ref{sec:conclusion}, we conclude the paper's contribution and propose some interesting future research directions. Indirectly relevant proofs and useful concentration inequalities are deferred to Appendices~\ref{sec:omitted-proofs} and \ref{sec:concentration-inequalities}.

\section{Preliminaries}
  \label{sec:preliminaries}
Before proceeding to our main results and technical overview, we first introduce a few additional definitions and convenient notational conventions.
  
For any function $h:\mathcal{X}\mapsto\{0,1\}$, we define its error region as $\mathrm{ER}(h)=\{(x,y)\in\mathcal{X}\times\{0,1\}:h(x)\neq y\}$. It follows immediately that the error rate of $h$ under distribution $\mathcal{D}$ is $\mathrm{er}_{\mathcal{D}}(h)=\Pr_{\mathcal{D}}(\mathrm{ER}(h))$. For any two functions $f,g:\mathcal{X}\mapsto\{0,1\}$, we define their pairwise disagreement region as $\{f\neq g\}=\{x\in\mathcal{X}: f(x)\neq g(x)\}$. The probability of their pairwise disagreement region under distribution $\mathcal{D}$ is $\Pr_{\mathcal{D}_{X}}(\{f\neq g\})$ where $\mathcal{D}_{X}$ is the marginal of $X$. We adopt the following notational conventions. For any $\Delta\subseteq\mathcal{X}$, we define $\Pr_{\mathcal{D}}(\Delta)=\Pr_{\mathcal{D}}(\Delta\times\{0,1\})=\Pr_{\mathcal{D}_{\mathcal{X}}}(\Delta)$. Hence, we abbreviate $\Pr_{\mathcal{D}}(\{f\neq g\})=\Pr_{\mathcal{D}_{X}}(\{f\neq g\})$. Given sample $S=\{(x_{1}, y_{1}),(x_{2}, y_{2}),\ldots\}\in(\mathcal{X}\times\{0,1\})^{*}$, we denote by $\hat{\Pr}_{S}(\mathrm{ER}(h))=\frac{1}{|S|}\sum_{(x,y)\in S}\mathbbm{1}(h(x)\neq y)=\widehat{\mathrm{er}}_{S}(h)$ the empirical error rate of $h$ on $S$ and $\hat{\Pr}_{S}(\{f\neq g\})=\frac{1}{|S|}\sum_{(x,y)\in S}\mathbbm{1}(f(x)\neq g(x))$. Furthermore, we also apply conventions in combination: we define $S\cap\Delta$ as the subsample $\{(x,y)\in S: x\in\Delta\}$, $\mathrm{ER}(h)\cap\Delta=\{(x,y)\in\mathcal{X}\times\{0,1\}:h(x)\neq y \land x\in\Delta\}$, and then define $\Pr_{\mathcal{D}}(\mathrm{ER}(h)\cap\Delta)$, $\Pr_{\mathcal{D}}(\{f\neq g\}\cap\Delta)$, $\hat{\Pr}_{S}(\mathrm{ER}(h)\cap\Delta)$ and $\hat{\Pr}_{S}(\{f\neq g\}\cap\Delta)$ accordingly. Finally, we define the ``decision-list'' aggregation of functions, e.g. $h=f\mathbbm{1}_{\Delta}+g\mathbbm{1}_{\mathcal{X}\setminus\Delta}$ implies $h(x)=f(x)$ for any $x\in\Delta$ and $h(x)=g(x)$ for any $x\in(\mathcal{X}\setminus\Delta)$.

We redefine $\log(x)=\log_{2}(x \lor 2)$. We adopt convenient asymptotic notations. For any real-valued functions $f,g$, we use standard asymptotic notations $O(\cdot)$ and $\Omega(\cdot)$ to hide numerical constant factors (e.g. $f(z)=O(g(z))$). 
Sometimes, we use $\lesssim$ and $\gtrsim$ instead (e.g. $f(z) \lesssim g(z)$). 

Finally, we also include the definitions for the Massart noise condition and the Bernstein class condition. 

\begin{definition}  [\textbf{Massart Noise Condition}]
  \label{def:Massart-noise-condition}
Let $\gamma\in(0,1/2)$. We say that a distribution $\mathcal{D}$ over $\mathcal{X}\times\{0,1\}$ satisfies a Massart noise condition with parameter $\gamma$ w.r.t.\ $\hpc$ if $h^{*}_{\mathcal{D}}\in\hpc$ and
\begin{equation*}
    |\eta_{\mathcal{D}}(x)-1/2| \geq \gamma ,\;\; \forall x\in\mathcal{X} .
\end{equation*}
\end{definition}

\begin{definition}  [\textbf{Bernstein Class Condition}]
  \label{def:Bernstein-class-condition}
Let $\alpha\in[0,1]$ and $a\geq1$. We say that a distribution $\mathcal{D}$ over $\mathcal{X}\times\{0,1\}$ satisfies a Bernstein class condition with parameters $(a,\alpha)$ w.r.t.\ $\hpc$ if $h^{*}_{\mathcal{D}}\in\hpc$ and
\begin{equation*}
    \Pr_{\mathcal{D}}\big(\{h\neq h^{*}_{\mathcal{D}}\}\big) \leq a\Big(\Pr_{\mathcal{D}}\big(\mathrm{ER}(h)\big)-\Pr_{\mathcal{D}}\big(\mathrm{ER}(h^{*}_{\mathcal{D}})\big)\Big)^{\alpha} ,\;\; \forall h\in\hpc.
\end{equation*}
\end{definition}

\section{Technical Overview}
  \label{sec:technical-overview}
In this section, we present a high-level idea and technical review of our new PAC learning algorithm under Tsybakov noise. Before proceeding to the technical overview of our new algorithm, let us briefly recall that our problem is to learn a concept class $\hpc$ with $\vcdim=d$ under a distribution $\mathcal{D}$ over $\mathcal{X}\times\{0,1\}$ that satisfies a Tsybakov noise condition. The goal is to design a learning algorithm $\mathcal{A}$ that takes i.i.d.\ samples $\{(x_{1}, y_{1}),(x_{2}, y_{2}),\ldots,(x_{n}, y_{n})\}\sim\mathcal{D}^{n}$ and outputs a hypothesis $\hat{h}_{\mathcal{A}}$ that achieves the optimal excess risk $\Pr_{\mathcal{D}}(\ERin{\hat{h}_{\mathcal{A}}})-\Pr_{\mathcal{D}}(\ERin{h^{*}_{\mathcal{D}}})=O((d/n+\log(1/\delta)/n)^{1/(2-\alpha)})$, with probability at least $1-\delta$. Perhaps surprisingly, our new PAC learning algorithm, achieving optimal excess risk under Tsybakov noise, is a \textbf{\emph{proper}} learner. Throughout this section, let us ignore the $\delta$-dependence in the bounds for the simplicity of presentation.

\subsection{Proof of concepts} 
  \label{subsec:proof-of-concepts}
We first explain conceptually what inspires our design of a new algorithm that achieves an optimal excess risk of $O((d/n)^{1/(2-\alpha)})$. Recall that the natural (proper) learning algorithm of Empirical Risk Minimization (ERM) can only guarantee an excess risk of $O(((d/n)\log(n/d))^{1/(2-\alpha)})$ (\cite{massart2006risk}, see also our Lemma~\ref{lem:excess-risk-under-tsybakov-noise}), which is of course a $\log(n/d)$ factor worse than what we are aiming for. Inspired by a general principle of techniques in \cite{hanneke2025agnostic,bousquet2021fast,puchkin2021exponential}, our new algorithm adopts a paradigm of adaptively partitioning the instance space $\mathcal{X}$ into regions $(\mathcal{X}\setminus\Delta_{1}),(\Delta_{1}\setminus\Delta_{2}),\ldots,(\Delta_{T-1}\setminus\Delta_{T})$ and $\Delta_{T}$ with increasing noise levels and then learning using ERM within each region separately. This scheduling of noise resolutions gradually prunes out sub-optimal concepts, improves the performance of ERM on each region, and leads to a learning algorithm $\mathcal{A}$ that returns a hypothesis being a decision-list aggregation of ERM predictors, that is, $\hat{h}_{\mathcal{A}}=\sum_{t=1}^{T}\hat{h}_{t-1}\mathbbm{1}(\Delta_{t-1}\setminus\Delta_{t})+\hat{h}_{T}\mathbbm{1}(\Delta_{T})$,\footnote{We write $\Delta_{0}=\mathcal{X}$ for simplicity.} where $\hat{h}_{t-1}\in\mathrm{ERM}(\hpc;S_{n}\cap(\Delta_{t-1}\setminus\Delta_{t}))$ for any $1\leq t\leq T$ and $\hat{h}_{T}\in\mathrm{ERM}(\hpc;S_{n}\cap\Delta_{T})$. Let us denote by $\mathcal{E}(\hpc;S)$ the upper bound on the excess risk of ERM for learning $\hpc$ given data $S$, i.e., $\sup_{h\in\mathrm{ERM}(\hpc;S)}\Pr_{\mathcal{D}}(\mathrm{ER}(h))-\Pr_{\mathcal{D}}(\mathrm{ER}(h^{*}_{\mathcal{D}}))\leq\mathcal{E}(\hpc;S)$. We have that $\Pr_{\mathcal{D}}(\ERin{\hat{h}_{\mathcal{A}}})-\Pr_{\mathcal{D}}(\ERin{h^{*}_{\mathcal{D}}})=$
\begin{align*}
    &\sum_{t=1}^{T}\left(\Pr_{\mathcal{D}}(\ERin{\hat{h}_{t-1}}\cap(\Delta_{t-1}\setminus\Delta_{t}))-\Pr_{\mathcal{D}}(\ERin{h^{*}_{\mathcal{D}}}\cap(\Delta_{t-1}\setminus\Delta_{t}))\right) \\
    &\hspace{2em}+\Pr_{\mathcal{D}}(\ERin{\hat{h}_{T}}\cap\Delta_{T})-\Pr_{\mathcal{D}}(\ERin{h^{*}_{\mathcal{D}}}\cap\Delta_{T}) \\
    \leq &\sum_{t=1}^{T}\mathcal{E}(\hpc;S_{n}\cap(\Delta_{t-1}\setminus\Delta_{t}))+\mathcal{E}(\hpc;S_{n}\cap\Delta_{T}) = O((d/n)^{1/(2-\alpha)}) .
\end{align*}
So far, we have already got an improper learning algorithm achieving optimal excess risks. Now, we explain how we can design a proper learner built upon this improper strategy. For any $1\leq t\leq T$, we define $\hat{\hpc}_{t-1}=\{h\in\hpc: \hat{\Pr}_{S_{n}\cap(\Delta_{t-1}\setminus\Delta_{t})}(\mathrm{ER}(h))-\hat{\Pr}_{S_{n}\cap(\Delta_{t-1}\setminus\Delta_{t})}(\mathrm{ER}(\hat{h}_{t-1})) \leq \mathcal{E}(\hpc;S_{n}\cap(\Delta_{t-1}\setminus\Delta_{t}))\}$, and also $\hat{\hpc}_{T}=\{h\in\hpc: \hat{\Pr}_{S_{n}\cap\Delta_{T}}(\mathrm{ER}(h))-\hat{\Pr}_{S_{n}\cap\Delta_{T}}(\mathrm{ER}(\hat{h}_{T})) \leq \mathcal{E}(\hpc;S_{n}\cap\Delta_{T})\}$. The desired proper learner then returns any concept $\hat{h}_{\mathcal{A}}$ in $\hat{\hpc}=(\bigcap_{1\leq t\leq T}\hat{\hpc}_{t-1})\cap\hat{\hpc}_{T}\subseteq\hpc$. Note that such a concept exists since $h^{*}_{\mathcal{D}}$ is always therein. Moreover, for any concept $\hat{h}_{\mathcal{A}}\in\hat{\hpc}$, we have $\Pr_{\mathcal{D}}(\ERin{\hat{h}_{\mathcal{A}}})-\Pr_{\mathcal{D}}(\ERin{h^{*}_{\mathcal{D}}})=$
\begin{align*}
    &\sum_{t=1}^{T}\left(\Pr_{\mathcal{D}}(\ERin{\hat{h}_{\mathcal{A}}}\cap(\Delta_{t-1}\setminus\Delta_{t}))-\Pr_{\mathcal{D}}(\ERin{\hat{h}_{t-1}}\cap(\Delta_{t-1}\setminus\Delta_{t}))\right) \\
    &\hspace{2em}+\Pr_{\mathcal{D}}(\ERin{\hat{h}_{\mathcal{A}}}\cap\Delta_{T})-\Pr_{\mathcal{D}}(\ERin{\hat{h}_{T}}\cap\Delta_{T}) \\
    + &\sum_{t=1}^{T}\left(\Pr_{\mathcal{D}}(\ERin{\hat{h}_{t-1}}\cap(\Delta_{t-1}\setminus\Delta_{t}))-\Pr_{\mathcal{D}}(\ERin{h^{*}_{\mathcal{D}}}\cap(\Delta_{t-1}\setminus\Delta_{t}))\right) \\
    &\hspace{2em}+\Pr_{\mathcal{D}}(\ERin{\hat{h}_{T}}\cap\Delta_{T})-\Pr_{\mathcal{D}}(\ERin{h^{*}_{\mathcal{D}}}\cap\Delta_{T}) \\ 
    \leq &2\sum_{t=1}^{T}\mathcal{E}(\hpc;S_{n}\cap(\Delta_{t-1}\setminus\Delta_{t}))+2\mathcal{E}(\hpc;S_{n}\cap\Delta_{T}) = O((d/n)^{1/(2-\alpha)}) .
\end{align*}
In the remaining part of this subsection, we provide a conceptual proof of the aforementioned improper algorithm of decision-list aggregation of ERMs. The main intuition stems from some basic understandings of the underlying Tsybakov noise model. Unlike the Massart noise condition, the Tsybakov noise condition $\Pr_{\mathcal{D}}(|\eta_{\mathcal{D}}(x)-1/2| \leq t) \leq at^{\alpha/(1-\alpha)}$ does not guarantee uniformly bounded noise over then entire instance space, but only restricts the noise in a probabilistic sense, that is, a noisier region has a lower probability mass assigned. Specifically, if $\Delta\subseteq\mathcal{X}$ is some region such that $\Pr_{\mathcal{D}}(\{h\neq h^{*}_{\mathcal{D}}\}\cap(\mathcal{X}\setminus\Delta))\leq \beta(\Pr_{\mathcal{D}}(\ERin{h}\cap(\mathcal{X}\setminus\Delta))-\Pr_{\mathcal{D}}(\ERin{h^{*}_{\mathcal{D}}}\cap(\mathcal{X}\setminus\Delta)))$ holds for any $h\in\hpc$ with some $\beta>0$, one can show that $\Pr_{\mathcal{D}}(\Delta)=O(\beta^{-\alpha/(1-\alpha)})$ (this can be easily proved following the same idea of proving Lemma~\ref{lem:upper-bound-of-p-Delta}). In particular, for an extremely noisy region $\Delta_{T}$ under Tsybakov noise, the optimal learning guarantee one could hope to achieve is no better than performing agnostic learning, for which an ERM learner $\hat{h}_{T}$ yields an excess risk $\Pr_{\mathcal{D}}(\ERin{\hat{h}_{T}}\cap\Delta_{T})-\Pr_{\mathcal{D}}(\ERin{h^{*}_{\mathcal{D}}}\cap\Delta_{T})=O(\sqrt{(d/n)\Pr_{\mathcal{D}}(\Delta_{T})})$. These together convince us to set up our last resolution $\beta_{T}=(n/d)^{(1-\alpha)/(2-\alpha)}$ so that to derive an optimal excess risk $O(\sqrt{(d/n)\Pr_{\mathcal{D}}(\Delta_{T})})=O((d/n)^{1/2}\beta_{T}^{-\alpha/2(2-\alpha)})=O((d/n)^{1/(2-\alpha)})$. Next, to figure out which resolution $\beta_{1}$ to start, consider some region $\Delta_{1}$ such that $\Pr_{\mathcal{D}}(\{h\neq h^{*}_{\mathcal{D}}\}\cap(\mathcal{X}\setminus\Delta_{1}))\leq \beta_{1}(\Pr_{\mathcal{D}}(\ERin{h}\cap(\mathcal{X}\setminus\Delta_{1}))-\Pr_{\mathcal{D}}(\ERin{h^{*}_{\mathcal{D}}}\cap(\mathcal{X}\setminus\Delta_{1})))$ holds for any $h\in\hpc$. This is indeed saying that $\mathcal{D}$ satisfies a linear Bernstein class condition restricting onto $(\mathcal{X}\setminus\Delta_{1})$. Under that, we have a tight ERM bound (Lemma~\ref{lem:excess-risk-under-linear-BCC}) yields an excess risk $\Pr_{\mathcal{D}}(\ERin{\hat{h}_{0}}\cap(\mathcal{X}\setminus\Delta_{1}))-\Pr_{\mathcal{D}}(\ERin{h^{*}_{\mathcal{D}}}\cap(\mathcal{X}\setminus\Delta_{1}))=O((\beta_{1}d/n)\log(n\Pr_{\mathcal{D}}(\mathcal{X}\setminus\Delta_{1})/\beta_{1}^{2}d))=O((\beta_{1}d/n)\log(n/\beta_{1}^{2}d))$. This motivates us to set up our first resolution as $\beta_{1}=(n/d)^{(1-\alpha)/(2-\alpha)}/\log(n/d)$ in order to achieve $O((\beta_{1}d/n)\log(n/\beta_{1}^{2}d))=O((d/n)^{1/(2-\alpha)})$ excess risk on $(\mathcal{X}\setminus\Delta_{1})$. What remains is to design a scheduling of noise resolutions from $\beta_{1}$ to $\beta_{T}$ such that, all intermediate regions $\{(\Delta_{t-1}\setminus\Delta_{t}), 2\leq t\leq T\}$ satisfying $\Pr_{\mathcal{D}}(\{h\neq h^{*}_{\mathcal{D}}\}\cap(\Delta_{t-1}\setminus\Delta_{t}))\leq \beta_{t}(\Pr_{\mathcal{D}}(\ERin{h}\cap(\Delta_{t-1}\setminus\Delta_{t}))-\Pr_{\mathcal{D}}(\ERin{h^{*}_{\mathcal{D}}}\cap(\Delta_{t-1}\setminus\Delta_{t}))), \forall h\in\hpc$, together admit an optimal excess risk by running ERMs separately and then aggregrate. Our approach realizes this by setting $\beta_{t}=2^{t-1}\beta_{1}$ for every $1\leq t\leq \lceil\log\log(n/d)\rceil+1$. Indeed, for any such region $(\Delta_{t-1}\setminus\Delta_{t})$, Lemma~\ref{lem:excess-risk-under-linear-BCC} guarantees that $\Pr_{\mathcal{D}}(\ERin{\hat{h}_{t-1}}\cap(\Delta_{t-1}\setminus\Delta_{t}))-\Pr_{\mathcal{D}}(\ERin{h^{*}_{\mathcal{D}}}\cap(\Delta_{t-1}\setminus\Delta_{t}))=O((\beta_{t}d/n)\log(n\Pr_{\mathcal{D}}(\Delta_{t-1}\setminus\Delta_{t})/\beta_{t}^{2}d))$. Note that $\Pr_{\mathcal{D}}(\Delta_{t-1}\setminus\Delta_{t})\leq\Pr_{\mathcal{D}}(\Delta_{t-1})=O(\beta_{t-1}^{-\alpha/(1-\alpha)})$. This yields $O((\beta_{t}d/n)\log(n\Pr_{\mathcal{D}}(\Delta_{t-1}\setminus\Delta_{t})/\beta_{t}^{2}d))=O((\beta_{t}d/n)\log(n\beta_{t-1}^{-\alpha/(1-\alpha)}/\beta_{t}^{2}d))$. Now, we can bound the aggregated excess risk by
\begin{align*}
    &\sum_{t=2}^{T}\Pr_{\mathcal{D}}\left(\ERin{\hat{h}_{t-1}}\cap(\Delta_{t-1}\setminus\Delta_{t})\right)-\Pr_{\mathcal{D}}\left(\ERin{h^{*}_{\mathcal{D}}}\cap(\Delta_{t-1}\setminus\Delta_{t})\right) = O\left(\sum_{t=2}^{T}\frac{\beta_{t}d}{n}\log\left(\frac{n\beta_{t-1}^{-\alpha/(1-\alpha)}}{\beta_{t}^{2}d}\right)\right) \\
    = &O\left(\left(\frac{d}{n}\right)^{\frac{1}{2-\alpha}}\log^{-1}\left(\frac{n}{d}\right)\sum_{t=2}^{\lceil\log\log(n/d)\rceil+1}2^{t-1}\left(\frac{2-\alpha}{1-\alpha}\log\log\left(\frac{n}{d}\right)-\frac{t-2\alpha}{1-\alpha}\right)\right) \\
    = &O\left(\left(\frac{d}{n}\right)^{\frac{1}{2-\alpha}}\sum_{t=2}^{\lceil\log\log(n/d)\rceil+1}2^{t-T}\left(\frac{(2-\alpha)T}{1-\alpha}-\frac{t-2\alpha}{1-\alpha}\right)\right)\\
    = &O\left(\left(\frac{d}{n}\right)^{\frac{1}{2-\alpha}}\sum_{t=2}^{T}2^{t-T}\frac{T-t}{1-\alpha}\right) = O\left(\left(\frac{d}{n}\right)^{\frac{1}{2-\alpha}}\right) .
\end{align*}
Putting together, we can design a conceptual algorithm that outputs a hypothesis being a decision-list aggregation of ERM predictors, that is, $\hat{h}_{\mathcal{A}}=\sum_{t=1}^{T}\hat{h}_{t-1}\mathbbm{1}(\Delta_{t-1}\setminus\Delta_{t}) + \hat{h}_{T}\mathbbm{1}(\Delta_{T})$, where $\Delta_{0}=\mathcal{X}$. By the law of total probability, we can bound its excess risk by $\Pr_{\mathcal{D}}(\ERin{\hat{h}_{\mathcal{A}}})-\Pr_{\mathcal{D}}(\ERin{h^{*}_{\mathcal{D}}})=\sum_{t=1}^{T}(\Pr_{\mathcal{D}}(\ERin{\hat{h}_{t-1}}\cap(\Delta_{t-1}\setminus\Delta_{t}))-\Pr_{\mathcal{D}}(\ERin{h^{*}_{\mathcal{D}}}\cap(\Delta_{t-1}\setminus\Delta_{t})))+\Pr_{\mathcal{D}}(\ERin{\hat{h}_{T}}\cap\Delta_{T})-\Pr_{\mathcal{D}}(\ERin{h^{*}_{\mathcal{D}}}\cap\Delta_{T})=O((d/n)^{1/(2-\alpha)})$, where $O$ hides a constant that depends on the Tsybakov noise model parameter $(a,\alpha)$. As mentioned previously, this can be translated into a proper learner.

\subsection{Algorithmic design} 
  \label{subsec:algorithmic-design}
Recall that in our conceptual analysis above, the main idea is to adaptively separate the instance space $\mathcal{X}$ into regions satisfying $\Pr_{\mathcal{D}}(\{h\neq h^{*}_{\mathcal{D}}\}\cap(\Delta_{t-1}\setminus\Delta_{t})) \leq \beta_{t}(\Pr_{\mathcal{D}}(\ERin{h}\cap(\Delta_{t-1}\setminus\Delta_{t}))-\Pr_{\mathcal{D}}(\ERin{h^{*}_{\mathcal{D}}}\cap(\Delta_{t-1}\setminus\Delta_{t}))), \forall h\in\hpc$, with different noise levels $\{\beta_{t}, 1\leq t\leq T\}$. This is roughly saying that the distribution $\mathcal{D}$ satisfies a Massart noise condition with parameter $1/\beta_{t}$ on the region $(\Delta_{t-1}\setminus\Delta_{t})$ since a Massart noise condition with margin parameter $\gamma>0$ implies a Bernstein class condition with $a=1/\gamma$ and $\alpha=1$ (but not vice versa) \citep{boucheron2005theory,bousquet2021fast}. Our new algorithm, concretizes this conceptual idea via approximating those near-Massart's regions statistically, and is thus named ``\textbf{M}assart'' \textbf{E}rror \textbf{R}egions \textbf{I}solation under \textbf{T}sybakov noise (\textbf{MERIT}), which is presented as in \Cref{alg:main-algo}. Below, we illustrate several key insights regarding our algorithmic design.

A natural idea of directly approximating the ratio between the excess risk $(\Pr_{\mathcal{D}}(\ERin{h}\cap(\Delta_{t-1}\setminus\Delta_{t}))-\Pr_{\mathcal{D}}(\ERin{h^{*}_{\mathcal{D}}}\cap(\Delta_{t-1}\setminus\Delta_{t})))$ and the $L_{1}$ pseudo distance $\Pr_{\mathcal{D}}(\{h\neq h^{*}_{\mathcal{D}}\}\cap(\Delta_{t-1}\setminus\Delta_{t}))$. Algorithmically, this should be verified on their empirical estimates, i.e., to find the region $(\Delta_{t-1}\setminus\Delta_{t})$ such that $(\hat{\Pr}_{S_{n}}(\ERin{h}\cap(\Delta_{t-1}\setminus\Delta_{t}))-\hat{\Pr}_{S_{n}}(\ERin{\hat{h}^{*}_{t-1}}\cap(\Delta_{t-1}\setminus\Delta_{t}))) \geq \beta_{t}^{-1}\hat{\Pr}_{S_{n}}(\{h\neq \hat{h}^{*}_{t-1}\}\cap(\Delta_{t-1}\setminus\Delta_{t})), \forall h\in\hpc$, where we use an ERM $\hat{h}^{*}_{t-1}$ to approximate the Bayes classifier on $(\Delta_{t-1}\setminus\Delta_{t})$. Unfortunately, this approach does not work out. A quick intuition can be derived from the fact that one cannot guarantee $\hat{\Pr}_{S_{n}}(\{h\neq \hat{h}^{*}_{t-1}\}\cap(\Delta_{t-1}\setminus\Delta_{t}))-\Pr_{\mathcal{D}}(\{h\neq h^{*}_{\mathcal{D}}\}\cap(\Delta_{t-1}\setminus\Delta_{t}))=O(\Pr_{\mathcal{D}}(\{h\neq h^{*}_{\mathcal{D}}\}\cap(\Delta_{t-1}\setminus\Delta_{t})))$ for any $h\in\hpc$, and thus a triangle inequality might fail to bridge the boundednesses between $\hat{\Pr}_{S_{n}}(\{h\neq \hat{h}^{*}_{t-1}\}\cap(\Delta_{t-1}\setminus\Delta_{t}))$ and $\Pr_{\mathcal{D}}(\{h\neq h^{*}_{\mathcal{D}}\}\cap(\Delta_{t-1}\setminus\Delta_{t}))$. Inspired by this, our algorithm applies a remedy of picking out those pairs of nearly-optimal functions $f,g$ with $(\hat{\Pr}_{S_{n}}(\ERin{f}\cap(\Delta_{t-1}))-\hat{\Pr}_{S_{n}}(\ERin{g}\cap(\Delta_{t-1})))\leq\gamma_{t}$ that are far apart from the each other in a sense of $\hat{\Pr}_{S_{n}}(\{f \neq g\}\cap(\Delta_{t-1}))\geq\beta_{t}\gamma_{t}$. It is clear that after extracting the disagreement regions of them into $\Delta_{t}$, an empirical linear Bernstein class condition with parameter $\beta_{t}^{-1}$ holds on the remain region $(\Delta_{t-1}\setminus\Delta_{t})$, leading to a population linear Bernstein class condition with high probability. 

Furthermore, our algorithm design leverages an idea of gradually pruning the concept class, that is, adaptively choosing a decreasing schedule of $\gamma_{t}$'s and gradually pruning the surviving concept class $\hpc_{t-1}=\{h\in\hpc:\Pr_{\mathcal{D}}(\ERin{h}\cap\Delta_{t-1})-\Pr_{\mathcal{D}}(\ERin{h^{*}_{\mathcal{D}}}\cap\Delta_{t-1})\leq\gamma_{t}\}$. A similar idea has been adopted in \cite{hanneke2025agnostic}. This pruning schedule $\{\gamma_{t},1\leq t\leq T\}$ together with the scheduling of the noise resolutions $\{\beta_{t},1\leq t\leq T\}$ allow us to guarantee optimal excess risks using ERM by applying a strong uniform Bernstein inequality.\footnote{We would like to emphasize that all our presented concentration inequalities in Section~\ref{sec:concentration-inequalities} are direct implications of the uniform Bernstein inequality. In other words, using Lemma~\ref{lem:uniform-bernstein-inequality} does not degrade any bound compared to using Lemma~\ref{lem:excess-risk-under-linear-BCC}, which has been shown to provide optimal guarantee in Subsection~\ref{subsec:proof-of-concepts}.} It is worth mentioning that such types of inequalities have been established in the literature for standard i.i.d.\ data \citep{koltchinskii2006local,gine2006concentration,hanneke2025universal,hanneke2025agnostic}, as well as for non-i.i.d.\ data \citep{hanneke2022no}. In particular, our presenting Lemma~\ref{lem:uniform-bernstein-inequality} gives out its most general form for i.i.d.\ data. Technically, the decreasing resolutions $\beta_{t}$'s allow us to have tight bounds $\Pr_{\mathcal{D}}(\Delta_{t-1})=O(\beta_{t-1}^{-\alpha/(1-\alpha)})$ to plug into Lemma~\ref{lem:uniform-bernstein-inequality} rather than using trivial bounds $\Pr_{\mathcal{D}}(\Delta_{t-1})\leq 1$. Collaborating together with $\beta_{t}$'s, pruning the concept class with $\gamma_{t}$'s then yields tight bounds $\Pr_{\mathcal{D}}(\{h\neq h^{*}_{\mathcal{D}}\}\cap(\Delta_{t-1}\setminus\Delta_{t}))=O(\beta_{t}\gamma_{t})$ rather than $\Pr_{\mathcal{D}}(\{h\neq h^{*}_{\mathcal{D}}\})\leq 1$. To formalize it, since $\Pr_{\mathcal{D}}(\{h\neq h^{*}_{\mathcal{D}}\}\cap(\Delta_{t-1}\setminus\Delta_{t}))=O(\beta_{t}\gamma_{t})$, Lemma~\ref{lem:uniform-bernstein-inequality} yields that the excess risk of ERM $\hat{h}_{t-1}$\footnote{We emphasize that $\hat{h}_{t-1}$ and the previous $\hat{h}^{*}_{t-1}$ are ERMs trained on independent datasets and for different purposes in the algorithm. Therefore, we assign different notations to distinguish them.} trained on $S_{n}\cap(\Delta_{t-1}\setminus\Delta_{t})$ satisfies
\begin{equation*}
    \Pr_{\mathcal{D}}(\ERin{\hat{h}_{t-1}}\cap(\Delta_{t-1}\setminus\Delta_{t}))-\Pr_{\mathcal{D}}(\ERin{h^{*}_{\mathcal{D}}}\cap(\Delta_{t-1}\setminus\Delta_{t}))=O(\sqrt{(d\beta_{t}\gamma_{t}/n)\log(\Pr_{\mathcal{D}}(\Delta_{t-1}\setminus\Delta_{t})/(\beta_{t}\gamma_{t}))}) .
\end{equation*} 
Since $\beta_{t}=2^{t-1}(n/d)^{(1-\alpha)/(2-\alpha)}/\log(n/d)$ and $\Pr_{\mathcal{D}}(\Delta_{t-1}\setminus\Delta_{t})\leq\Pr_{\mathcal{D}}(\Delta_{t-1})=O(\beta_{t-1}^{-\alpha/(1-\alpha)})$, plugging into the above bound yields that 
\begin{align*}
    &\Pr_{\mathcal{D}}(\ERin{\hat{h}_{t-1}}\cap(\Delta_{t-1}\setminus\Delta_{t}))-\Pr_{\mathcal{D}}(\ERin{h^{*}_{\mathcal{D}}}\cap(\Delta_{t-1}\setminus\Delta_{t})) \\
    = &O(\sqrt{(d\beta_{t}\gamma_{t}/n)\log(\Pr_{\mathcal{D}}(\Delta_{t-1}\setminus\Delta_{t})/(\beta_{t}\gamma_{t}))}) \\
    = &O(\sqrt{(d/n)^{1/(2-\alpha)}2^{t-1}\log^{-1}(n/d)\gamma_{t}\log(\beta_{t-1}^{-\alpha/(1-\alpha)}/\beta_{t}\gamma_{t})}) .
\end{align*}
Now, choosing $\gamma_{t}=O((d/n)^{1/(2-\alpha)}(\log((n/d)^{\alpha/(2-\alpha)}\beta_{t-1}^{-\alpha/(1-\alpha)}))^{1/(2-\alpha)})$ would suffice to guarantee an optimal overall excess risk $O((d/n)^{1/(2-\alpha)})$ on $(\mathcal{X}\setminus\Delta_{T})$ by summing over $1\leq t\leq T$ following the same calculation as in Subsection~\ref{subsec:proof-of-concepts}. 

A natural question is whether we are allowed to produce such a pruning schedule $\{\gamma_{t},1\leq t\leq T\}$ algorithmically. In other words, we need to guarantee that such $\gamma_{t}$-pruned classes will not degenerate to empty throughout running the algorithm. Indeed, we are able to show that for any $1\leq t\leq T$, a pruned class $\hpc_{t-1}=\{h\in\hpc:\Pr_{\mathcal{D}}(\ERin{h}\cap\Delta_{t-1})-\Pr_{\mathcal{D}}(\ERin{h^{*}_{\mathcal{D}}}\cap\Delta_{t-1})\leq\gamma_{t}\}$ is attainable by simply running ERM on sample $S_{n}\cap\Delta_{t-1}$.\footnote{Note that we also need to approximate the unknown $h^{*}_{\mathcal{D}}$ on $\Delta_{t-1}$ by an ERM predictor $\hat{h}^{*}_{t-1}\in\mathrm{ERM}(\hpc;S_{n}\cap\Delta_{t-1})$ in the algorithm. Here, we ignore this estimation for the simplicity of presentation.} Note that it suffices to show that the magnitude of $\gamma_{t}$ is larger than the bound in Lemma~\ref{lem:excess-risk-under-tsybakov-noise}, which provides the optimal excess risk upper bound of ERM. This is clear since applying Lemma~\ref{lem:excess-risk-under-tsybakov-noise} yields a bound of $\mathcal{E}(\hpc;S_{n}\cap(\Delta_{t-1}\setminus\Delta_{t}))=O(((d/n)\log((n/d)^{\alpha/(2-\alpha)}\Pr_{\mathcal{D}}(\Delta_{t-1})))^{1/(2-\alpha)})=O(((d/n)\log((n/d)^{\alpha/(2-\alpha)}\beta_{t-1}^{-\alpha/(1-\alpha)}))^{1/(2-\alpha)})=O(\gamma_{t})$. We would like to emphasize that all concentration bounds used in the proof collaborate perfectly (we notice little to no freedom in tightening the analysis), and allow us to establish the proof recursively: in each round $1\leq t\leq T$, the uniform Bernstein inequality yields an upper bound on the excess risk of ERM trained on $S_{n}\cap(\Delta_{t-1}\setminus\Delta_{t})$ as
\begin{equation*}
    O\left(\sqrt{(d\Pr_{\mathcal{D}}(\hat{h}_{t-1}\neq h^{*}_{\mathcal{D}}\cap(\Delta_{t-1}\setminus\Delta_{t}))/n)\log(\Pr_{\mathcal{D}}(\Delta_{t-1}\setminus\Delta_{t})/\Pr_{\mathcal{D}}(\hat{h}_{t-1}\neq h^{*}_{\mathcal{D}}\cap(\Delta_{t-1}\setminus\Delta_{t})))}\right) ,
\end{equation*}
where $\Pr_{\mathcal{D}}(\hat{h}_{t-1}\neq h^{*}_{\mathcal{D}}\cap(\Delta_{t-1}\setminus\Delta_{t}))=O(\beta_{t}\gamma_{t})$ is realized via pruning and the bound $\Pr_{\mathcal{D}}(\Delta_{t-1}\setminus\Delta_{t})\leq\Pr_{\mathcal{D}}(\Delta_{t-1})=O(\beta_{t-1}^{-\alpha/(1-\alpha)})$ is guaranteed by the previous round noise resolution.

Finally, we describe how to allocate $n$ samples to meet all the requirements of the aforementioned estimations throughout the algorithm. The entire dataset $S_{n}$ is first partitioned into three disjoint subsets $S_{1},S_{2}$ and $S_{3}$ with $|S_1| = n/2$, and $|S_2|=|S_3|=n/4$. For each round $1\leq t\leq T$, we first allocate $n_{t}=2^{t-T}n/4$ samples $S_{1,t}\subset S_{1}$ (with $\bigcap_{1\leq t\leq T}S_{1,t}=\emptyset$) to estimate the pruned concept class $\hpc_{t-1}$ consisting of surviving concepts $h$ satisfying $\hat{\Pr}_{S_{1,t}}(\mathrm{ER}(h)\cap\Delta_{t-1}) \leq \hat{\Pr}_{S_{1,t}}(\mathrm{ER}(h^{*}_{\mathcal{D}})\cap\Delta_{t-1})+\gamma_{t}$.\footnote{Note that in the algorithm we approximate the unknown $h^{*}_{\mathcal{D}}$ on $\Delta_{t-1}$ by an ERM predictor $\hat{h}^{*}_{t-1}\in\mathrm{ERM}(\hpc;S_{1,t}\cap\Delta_{t-1})$. Here, we ignore this estimation for simplicity.} Note that this allocation satisfies $\sum_{t=1}^{T}n_{t}\leq\sum_{t=0}^{\infty}2^{-t}n/4\leq n/2$. Next, we need to isolate $\Delta_t$ from $\Delta_{t-1}$. Recall that this has been done via recursively extracting pairs of nearly-optimal functions $f,g\in\hpc_{t-1}$ having large pseudo-distance. For this purpose, we allocate samples $S_{2,t}\subset S_{2}$ to estimate the pseudo-distance between $f,g \in \hpc_{t-1}$ satisfying $\hat{\Pr}_{S_{2,t}}(\{f \neq g\}\cap(\Delta_{t-1}\setminus\Delta_{t}))\geq\beta_{t}\gamma_{t}$ and updating $\Delta_{t}$ which is initialized empty. Note that we will add the entire disagreement region $\{f\neq g\}$ to $\Delta_t$, implying that $\Pr_{\mathcal{D}}(\Delta_t)$ increases at least $\Omega(\beta_t\gamma_t)$ at each iteration. Hence, the recursive process will terminate in at most $s_{t}= O(\Pr_{\mathcal{D}}(\Delta_t)/ \beta_t\gamma_t) = O(\beta_{t}^{-1/(1-\alpha)}\gamma_{t}^{-1})$ steps. Moreover, to estimate the pseudo-distance within $O(\beta_t\gamma_t)$ error for any pair of functions, we need to use $m_{t}=O((d/\beta_{t}\gamma_{t})\log(1/\beta_{t}\gamma_{t}))$ fresh samples for each recursive step. Putting together, the total sample size is $\sum_{t=1}^{T}|S_{2,t}|=\sum_{t=1}^{T}s_{t}m_{t}=O(\sum_{t=1}^{T}d\beta_{t}^{-(2-\alpha)/(1-\alpha)}\gamma_{t}^{-2}\log(1/\beta_{t}\gamma_{t}))$. When setting $\beta_{t}=2^{t-T}(n_{t}/d)^{(1-\alpha)/(2-\alpha)}$ and $\gamma_{t}=O((d/n_{t})^{1/(2-\alpha)}(\log(2^{-t}\log(n_{t}/d)))^{1/(2-\alpha)})$, we have the following hold: (i) $\sum_{t=1}^{T}s_{t}m_{t}=O(\sum_{t=1}^{T}d\beta_{t}^{-(2-\alpha)/(1-\alpha)}\gamma_{t}^{-2}\log(1/\beta_{t}\gamma_{t}))=O(n)$. Indeed, when $n$ is sufficiently large, we can bound $\sum_{t=1}^{T}s_{t}m_{t}\leq n/4$; (ii) we are able to guarantee optimal excess risk of a decision-list aggregation of ERM predictors on regions $(\mathcal{X}\setminus\Delta_{1}),(\Delta_{1}\setminus\Delta_{2}),\ldots,(\Delta_{T-1}\setminus\Delta_{T})$ and $\Delta_{T}$ trained over fresh samples $S_{3}$. As discussed, it can be transformed into a proper learner.

\section{Optimal PAC Learner under Tsybakov Noise Condition}
  \label{sec:main-results}
In this section, we present the main results of this paper. The entire section is organized as follows. We first present our optimal learning algorithm as \Cref{alg:main-algo} below. Next, we state an upper bound on the excess risk of the predictor returned by \Cref{alg:main-algo} in \Cref{thm:optimal-rates-main}. It claims that for any general concept class $\mathcal{H}$ with $\vcdim=d$, any $\delta\in(0,1)$ and any distribution $\mathcal{D}$ satisfies a Tsybakov noise condition with respect to $\hpc$ with parameters $(a,\alpha)$, our algorithm reaches the optimal asymptotic guarantee on the excess risk. To describe and prove the bound, we need to define the following constants, including numerical constants and constants related to parameters $a, \alpha$. Let $c_0$ be a numerical constant from the uniform Bernstein inequality, $c_1 = 1/(2^{\frac{1-\alpha}{2-\alpha}}+ 2^{-\frac{1-\alpha}{2-\alpha}}-2)$, $a' = (1-\alpha)(2\alpha)^{\alpha/(1-\alpha)}a^{1/(1-\alpha)}$, $c = 16ac_0^2 +2c_0+2$, $c_2 = 2^{(1+2\alpha)/(1-\alpha)}a'c^{\alpha/(1-\alpha)}+1$, $C_1 = c_2/4 + 2c_0a^{1/(2-\alpha)}$ and $C_2 = 4c_0C_1/(1-\alpha)$. 
We include the proof of \Cref{thm:optimal-rates-main} in the rest of this section, based on the correctness of certain technical lemmas. We also provide high-level ideas for proving those technical lemmas, where their formal proofs are deferred to Appendix~\ref{sec:omitted-proofs}.

In the algorithm, several parameters need to be defined in advance. First, the number of rounds of the \textbf{for} loop, which is $T = \lceil\log\log(n/d)\rceil+1$. Then, on each round $1\leq t\leq T$, we need to define the confidence, $\delta_t$, and the number of samples required to estimate the excess risk, $n_t$, based on the input of confidence, $\delta$, and sample size, $4n$, $\delta_{t}=(\delta/5)^{T+1-t}$, $n_t = n 2^{t-T}$. Then we can define 
\begin{align*}
    a_t&=2^{t-T}\left(\frac{n_t}{d+\log\left(1/\delta\right)}\right)^{\frac{1-\alpha}{2-\alpha}}, \\
    \gamma_t&=C_1\left(\frac{d}{n_t}\log\left(\left(\frac{n_t}{d+\log\left(1/\delta\right)}\right)^{\frac{\alpha}{2-\alpha}}(a_{t-1})^{-\frac{\alpha}{1-\alpha}}\right)+\frac{1}{n_t}\log\left(\frac{2}{\delta_{t}}\right)\right)^{\frac{1}{2-\alpha}}, \\
    m_t &= \frac{8c_0^2}{a_t\gamma_t}\left(d\log\left(\frac{1}{a_t\gamma_t}\right)+\log\left(\frac{2}{\delta_t}\right)\right) ,\\
    e_t &= C_2 \left(\frac{1}{2}\right)^{\frac{1-\alpha}{2-\alpha}(T-t)}(T-t)\left(\frac{1}{n}\left(d+\log\left(\frac{1}{\delta}\right)\right)\right)^{\frac{1}{2-\alpha}}.
\end{align*}
After that, we can provide our learning algorithm as follows. 
\begin{algorithm}[H]
\caption{``\textbf{M}assart'' \textbf{E}rror \textbf{R}egions \textbf{I}solation under \textbf{T}sybakov noise (\textbf{MERIT})}
\label{alg:main-algo}
\begin{algorithmic}[1]
    \Require Sample $S_{4n}=\{(x_{i}, y_{i})\}_{i=1}^{4n}\sim\mathcal{D}^{4n}$, Tsybakov noise parameter $(a,\alpha)$, VC dimension $d$, confidence $\delta\in(0,1)$.
    \State $\Delta_{0} \gets \mathcal{X}$.
    \State $S_3 \gets n$ fresh samples from $S_{4n}$.
    \For{$t = 1,\ldots,T$}
        \State $\Delta_{t} \gets \emptyset$.
        \State $S_{1,t}\leftarrow n_t$ fresh samples from $S_{4n}$.
        \State $\hat{h}^{*}_{t-1}\gets$ any $h\in\mathrm{ERM}(\hpc;S_{1,t}\cap\Delta_{t-1})$.
        \State $\hpc_{t-1} \gets \{h\in\hpc: \hat{\Pr}_{S_{1,t}}(\mathrm{ER}(h)\cap\Delta_{t-1}) \leq \hat{\Pr}_{S_{1,t}}(\mathrm{ER}(\hat{h}^{*}_{t-1})\cap\Delta_{t-1}) + \gamma_{t}\}$.
        \State $S_{2,t}\gets m_t$ fresh data from $S_{4n}$.
        \While{$\exists f,g\in\hpc_{t-1}$, such that $\hat{\Pr}_{S_{2,t}}\left(\{f\neq g\}\cap\left(\Delta_{t-1}\setminus\Delta_{t}\right)\right) \geq a_{t}\gamma_{t}$}
            \State $\Delta_{t} \leftarrow \Delta_{t}\cup\{f\neq g\}$.
            \State $S_{2,t}\gets m_t$ fresh samples from $S_{4n}$.
        \EndWhile
        \State $\hat{h}_{t-1}\gets$ any $h\in\mathrm{ERM}(\hpc_{t-1};S_{3}\cap(\Delta_{t-1}\setminus\Delta_{t}))$.
        \State $\hat{\hpc}_{t-1} \gets \{h\in \hpc_{t-1}:\hat{\Pr}_{S_{3}}(\mathrm{ER}(h)\cap(\Delta_{t-1}\setminus\Delta_t)) \leq \hat{\Pr}_{S_{3}}(\mathrm{ER}(\hat{h}_{t-1})\cap(\Delta_{t-1}\setminus\Delta_t)) + e_{t-1}\}$.
    \EndFor
    \State $\hat{h}_{T}\gets$ any $h\in\mathrm{ERM}(\hpc;S_{3}\cap\Delta_{T})$.
    \State $\hat{\hpc}_{T} \gets \{h\in \hpc:\hat{\Pr}_{S_{3}}(\mathrm{ER}(h)\cap\Delta_{T}) \leq \hat{\Pr}_{S_{3}}(\mathrm{ER}(\hat{h}_{T})\cap\Delta_{T}) + C_2(d/n+\log(1/\delta)/n)^{1/(2-\alpha)}\}$.
    \Ensure $h_{\text{ALG}}\gets$ any $h$ in $\bigcap_{t=0}^{T}\hat{\hpc}_{t}$.
\end{algorithmic}
\end{algorithm}

The main idea of this algorithm is to decompose the instance space $\mathcal{X}$ into $T+1$ different regions, which are $\mathcal{X}\setminus\Delta_1, \ldots,\Delta_{T-1}\setminus\Delta_T$ and $\Delta_T$, such that on every region $\Delta_{t-1}\setminus\Delta_{t}$, the pseudo-distance between any two concepts is bounded. To reach this goal, we directly isolate the disagreement region between two concepts from $\Delta_{t-1}$ if the pseudo-distance between them is large. 

We also notice that all concepts in $\hpc_{t-1}$ have a low excess risk on the region $\Delta_{t-1}$. Therefore, if two concepts in $\hpc_{t-1}$ have a large pseudo-distance, the noise level at their disagreement region will be high. Then, because the distribution $\mathcal{D}$ satisfies the Tsybakov noise condition, the size of the region $\Delta_t$ is also bounded. This separation process is shown as the \textbf{while} loop in \Cref{alg:main-algo}. We also notice that because the size of $\Delta_t$ is upper-bounded, and the size of the region added to $\Delta_t$ in each round is lower-bounded. The number of pairs of concepts with a large pseudo-distance, $s_t$, is bounded. 

In the meantime, as we use the refined uniform Bernstein inequality (\Cref{lem:uniform-bernstein-inequality}) to get the accuracy of our estimation of the excess risk and pseudo-distance on different regions, the separation of the regions should be independent of the samples used for estimation. Therefore, at every round $t$, we need fresh samples to finish the following estimations: 1) use $n_t$ fresh samples to estimate the excess risk, and 2) after every update of $\Delta_t$, use $m_t$ fresh samples to estimate the pseudo-distance between two concepts. We also notice that $\Delta_t$ will only be updated at most $s_t$ times. At last, we require $n$ new samples to run ERM on each region and get the output. Combining all those facts above, we have an upper bound on the number of required samples, $N$, which is linear in $n$. This is shown as \Cref{lem:number-of-samples}.

After the region separation, we can apply the refined uniform Bernstein inequality to each region separately and show that, with probability at least $1-3\delta/5$, the following statements hold.
\begin{itemize}
    \item For every $t = 0,1,\ldots,T-1$, for every $h\in\hat{\hpc}_t$, $\Pr_{\mathcal{D}}(\ERin{h}\cap(\Delta_t\setminus\Delta_{t+1})) - \Pr_{\mathcal{D}}(\ERin{h^*_{\mathcal{D}}}\cap(\Delta_t\setminus\Delta_{t+1})) \leq e_t$.
    \item For every $h\in\hat{\hpc}_T$, $\Pr_{\mathcal{D}}(\ERin{h}\cap\Delta_T) - \Pr_{\mathcal{D}}(\ERin{h^*_{\mathcal{D}}}\cap\Delta_T) \leq 2C_2((d+\log(1/\delta))/n)^{\frac{1}{2-\alpha}}$.
    \item For every $t = 0,1,\ldots,T$, $h^*_{\mathcal{D}}\in\hat{\hpc}_t$.
\end{itemize}
From the last bullet, we know that $h^*_{\mathcal{D}}\in \bigcap_{t=0}^{T}\hat{\hpc}_{t}$, therefore, we can always generate such a concept $h$. From the first bullet and the definition of $e_t$ according to \Cref{alg:main-algo}, we know that $\sum_{t = 0}^{T-1} e_t \leq c_1C_2 ((d+\log(1/\delta))/n)^{1/(2-\alpha)}$. Combining with the second bullet, we know that the excess risk of the chosen $h$ is in $O(((d+\log(1/\delta))/n)^{1/(2-\alpha)})$, with probability at least $1-3\delta/5$, that is the optimal rates we are aiming for. 

Along the way to proving the main theorem, several high-probability bounds on the estimation of the population errors and the pseudo-distance between two functions are required, shown as \Cref{lem:prune}, \Cref{lem:uniform-chernoff-pseudo-distance-estimation}, and \Cref{lem:approximate-linear-BCC}. All of those lemmas are based on the (refined) uniform Bernstein inequality, proved as \Cref{lem:uniform-bernstein-inequality}.

After explaining the main ideas of our algorithm, we can now start the formal proofs. First, we formally provide the main theorem as follows.
\begin{theorem}
\label{thm:optimal-rates-main}
    Let $\hpc$ be a concept class with $\vcdim=d$. Let $\mathcal{D}$ be a distribution that satisfies a Tsybakov noise condition with parameter $(a,\alpha)$ with respect to $\hpc$. Let $\delta\in(0,1)$. For any $n =\Omega((d+\log(1/\delta))(4(2-\alpha)/(1-\alpha))^{(8(2-\alpha)/(1-\alpha))})$, let $S_n=\{(x_i,y_i)\}_{i=1}^{n}\sim\mathcal{D}^{n}$ be i.i.d.\ samples. With probability at least $1-\delta$, the predictor $h_{\mathrm{ALG}}$ returned by \Cref{alg:main-algo} satisfies
    \begin{equation}
        \Pr_{\mathcal{D}}\left(\ERin{h_{\mathrm{ALG}}}\right) - \Pr_{\mathcal{D}}\left(\ERin{h^{*}_\mathcal{D}{}}\right) \leq 2(c_1+1)C_2\left(\frac{1}{n}\left(d+\log\left(\frac{1}{\delta}\right)\right)\right)^{\frac{1}{2-\alpha}} , 
    \end{equation}
    where $c_1, C_2$ are constants defined at the beginning.
\end{theorem}

\begin{remark}
We would like to emphasize that our algorithm does not serve as a base for designing potential optimal learners under the realizable setting. Specifically, taking $\alpha\rightarrow1$ to approximate the realizable case, our algorithm results in running a single ERM on the entire $\mathcal{X}$, which is, of course, sub-optimal. This is thus compatible with the fact that optimal rates cannot be admitted by any proper learner for the realizable case. Also, one should not hope to apply our algorithm, which is in its nature proper, to the agnostic setting to achieve the optimal first-order excess risk bound, since this has been proven to be only attainable by improper learners \cite{hanneke2024revisiting}. However, we conjecture that our techniques might be useful to obtain an optimal lower-order term in that first-order excess risk bound (see a further discussion in Section~\ref{sec:conclusion}).
\end{remark}

As we mentioned before, the algorithm consists of two parts. One is the process of space separation, and the other is generating a good concept as the output. To simplify the proof, we divide the proof into two parts as well. We first focus on the proof of the property of the good concept under the assumption that the space separation part works as desired, and then prove that the space separation actually works as desired. 

We first define the following good events, which describe the goal of the space separation process as follows. For each round $t\in[T]$ in the \textbf{for} loop of \Cref{alg:main-algo}, we define the event $\mathcal{E}_{t}$ as, on termination of the \textbf{while} loop in round $t$, the following holds:
\begin{itemize}
    \item $\Pr_{\mathcal{D}}(\mathrm{ER}(h)\cap\Delta_{t-1}) - \Pr_{\mathcal{D}}(\mathrm{ER}(h^{*}_{\mathcal{D}})\cap\Delta_{t-1})\leq 2c\gamma_t, \;\; \forall h\in\hpc_{t-1}$.
    \item $\Pr_{\mathcal{D}}(\{f \neq g\}\cap(\Delta_{t-1}\setminus\Delta_{t})) \leq 2a_t\gamma_t, \;\; \forall f,g\in\hpc_{t-1}$.
    \item $\Pr(\Delta_t) \leq c_2(a_t)^{-\frac{\alpha}{1-\alpha}}$.
    \item $s_t \leq 2c_2 a_t^{-\frac{1}{1-\alpha}} \gamma_t^{-1}$, where $s_t$ is the number of pairs $f,g$ found on the termination of \textbf{while} loop.
    \item $h^*_{\mathcal{D}} \in \hpc_{t-1}$.
\end{itemize}
Here, $c, c_2$ are constants defined as above.

For the simplicity of the proof, we first build our main results under the assumption that $\mathcal{E}_t$ happens for all $t$, then we prove that $\mathcal{E}_t$ happens for all $t$ with high probability. After that, the union bound yields the final result. 

Before diving into the proof of \Cref{thm:optimal-rates-main}, we first show that the total number of fresh samples required by \Cref{alg:main-algo} is $4n$ under the assumption that $\mathcal{E}_t$ happens for all $t$, which is stated as the following lemma.

\begin{lemma}
\label{lem:number-of-samples}
    For any $n = \Omega((d+\log(1/\delta))(4(2-\alpha)/(1-\alpha))^{(8(2-\alpha)/(1-\alpha))})$, the total number of fresh samples required by \Cref{alg:main-algo} is at most $4n$, if $\mathcal{E}_t$ happens for all $t$.
\end{lemma}

To prove this lemma, we need to bound the number of fresh samples required by the algorithm. Note that the algorithm requires three main groups of fresh samples, denoted as $S_1$, $S_2$, and $S_3$. $S_1$ contains the samples used to estimate the excess risk during the space separation process, and the number of samples required is $\sum_{t= 1^T} n 2^{t-T} \leq 2n$. $S_2$ contains the samples used to estimate the pseudo-distance between two concepts during the space separation process. At round $t$, as $\mathcal{E}_t$ is defined, there are at most $s_t$ updates of $\Delta_t$ and every update will consume $m_t$ fresh samples. Therefore, the number of samples in $S_2$ is upper-bounded by $\sum_{t = 1}^T m_t s_t$. By careful computation, it is smaller than or equal to $n$, if $n = \Omega((d+\log(1/\delta))(4(2-\alpha)/(1-\alpha))^{(8(2-\alpha)/(1-\alpha))})$. At last, $S_3$ contains the fresh samples required for generating the good concept after the space separation process, the number of which is $n$. Combining all of those bounds, \Cref{lem:number-of-samples} holds. For brevity, the complete proof is postponed to \Cref{sec:omitted-proofs}.

We can now prove \Cref{thm:optimal-rates-main} under the assumption that for all $t$, $\mathcal{E}_t$ happens. The main idea of the proof is to show that the following statements hold by the uniform Bernstein inequality.
\begin{itemize}
    \item For every $t = 0,1,\ldots,T-1$, for every $h\in\hat{\hpc}_t$, $\Pr_{\mathcal{D}}(\ERin{h}\cap(\Delta_t\setminus\Delta_{t+1})) - \Pr_{\mathcal{D}}(\ERin{h^*_{\mathcal{D}}}\cap(\Delta_t\setminus\Delta_{t+1})) \leq e_t$.
    \item For every $h\in\hat{\hpc}_T$, $\Pr_{\mathcal{D}}(\ERin{h}\cap\Delta_T) - \Pr_{\mathcal{D}}(\ERin{h^*_{\mathcal{D}}}\cap\Delta_T) \leq C_2((d+\log(1/\delta))/n)^{\frac{1}{2-\alpha}}$.
    \item For every $t = 0,1,\ldots,T$, $h^*_{\mathcal{D}}\in\hat{\hpc}_t$.
\end{itemize}
\begin{proof}[Proof of \Cref{thm:optimal-rates-main} (Under the assumption that for all $t$, $\mathcal{E}_t$ happens)]
    By \Cref{lem:number-of-samples}, we know that if for all $t$, $\mathcal{E}_t$ happens and $n = \Omega((d+\log(1/\delta))(4(2-\alpha)/(1-\alpha))^{(8(2-\alpha)/(1-\alpha))})$, the algorithm gets enough fresh samples to assign. Then we can focus on the proof of the generalization bound.
    
    According to \Cref{alg:main-algo}, and the fact that $\Delta_t\setminus\Delta_{t+1}$ are disjoint from each other for all $t = 0,\ldots, T$, due to the total probability law, we have
    \begin{equation*}
        \Pr_{\mathcal{D}}\left(\ERin{h_{\mathrm{ALG}}}\right) - \Pr_{\mathcal{D}}\left(\ERin{h^*_{\mathcal{D}}}\right) = \sum_{t = 0}^T \Pr_{\mathcal{D}}\left(\ERin{h_{\mathrm{ALG}}}\cap(\Delta_t\setminus\Delta_{t+1})\right) - \Pr_{\mathcal{D}}\left(\ERin{h^*_{\mathcal{D}}}\cap(\Delta_t\setminus\Delta_{t+1})\right),
    \end{equation*}
    where $\Delta_0 = \cX$ and $\Delta_{T+1} = \emptyset$ and $\Delta_t$ is generated by \Cref{alg:main-algo} for $t = 1,\ldots,T$. Let $\delta_t$ be chosen as defined by \Cref{alg:main-algo} and $\delta_{T+1} = \delta_T$. Therefore, we need to upper bound the excess risk for each $t$ separately. 
    For every $t$, from \Cref{lem:uniform-bernstein-inequality}, with probability at least $1-\delta_{t+1}$, for every $h_{1,t},h_{2,t} \in \hpc_t$, we have 
    \begin{align}
    &\left|\Pr_{\mathcal{D}}\left(\ERin{h_{1,t}}\cap(\Delta_t\setminus\Delta_{t+1})\right) - \Pr_{\mathcal{D}}\left(\ERin{h_{2,t}}\cap(\Delta_t\setminus\Delta_{t+1})\right)\right. \nonumber\\
    &-\left(\left.\hat{\Pr}_{S_3}\left(\ERin{h_{1,t}}\cap(\Delta_t\setminus\Delta_{t+1})\right)
    -\hat{\Pr}_{S_3}\left(\ERin{h_{2,t}}\cap(\Delta_t\setminus\Delta_{t+1})\right)\right)\right|\; \nonumber\\
    \leq &2c_0\sqrt{\frac{\Pr_{\mathcal{D}}(\{h_{1,t}\neq h_{2,t}\}\cap(\Delta_t\setminus\Delta_{t+1}))}{n}\left(d\log\left(\frac{\Pr_{\mathcal{D}}(\Delta_t\setminus\Delta_{t+1})}{\Pr_{\mathcal{D}}(\{h_{1,t}\neq h_{2,t}\}\cap(\Delta_t\setminus\Delta_{t+1}))}\right)+\log\left(\frac{1}{\delta_{t+1}}\right)\right)}.\label{eq:final-bound-rhs}
    \end{align}
    The inequality comes from \Cref{lem:uniform-bernstein-inequality}, by omitting the lower order term and $(\Pr_{\mathcal{D}}(\Delta_t\setminus\Delta_{t+1})n)/d$ in the logarithmic function. 
    
    Then by the assumption, we know that $\mathcal{E}_t$ happens for all $t$, we have the following facts hold for all $t$:
    \begin{itemize}
        \item $\Pr_{\mathcal{D}}(\{f \neq g\}\cap(\Delta_{t}\setminus\Delta_{t+1})) \leq 2a_{t+1}\gamma_{t+1}, \;\; \forall f,g\in\hpc_{t}$.
        \item $\Pr(\Delta_t) \leq c_2(a_t)^{-\frac{\alpha}{1-\alpha}}$.
        \item $h^*_{\mathcal{D}} \in \hpc_{t}$.
    \end{itemize}
    Notice that $\Delta_t\setminus\Delta_{t+1} \subseteq \Delta_t$, thus, we have $\Pr_{\mathcal{D}}(\Delta_t\setminus\Delta_{t+1}) \leq \Pr_{\mathcal{D}}(\Delta_t) \leq c_2(a_t)^{-\alpha/(1-\alpha)}$. Combining with the first bullet and \Cref{eq:final-bound-rhs}, for every $h_{1,t},h_{2,t} \in \hpc_t$ we have, 
    \begin{align}
    \label{eq:final-bound-rhs-2}
        &\left|\Pr_{\mathcal{D}}\left(\ERin{h_{1,t}}\cap(\Delta_t\setminus\Delta_{t+1})\right) - \Pr_{\mathcal{D}}\left(\ERin{h_{2,t}}\cap(\Delta_t\setminus\Delta_{t+1})\right)\right. \nonumber\\
    &-\left(\left.\hat{\Pr}_{S_3}\left(\ERin{h_{1,t}}\cap(\Delta_t\setminus\Delta_{t+1})\right)
    -\hat{\Pr}_{S_3}\left(\ERin{h_{2,t}}\cap(\Delta_t\setminus\Delta_{t+1})\right)\right)\right|\; \nonumber\\ 
    \leq &4c_0\sqrt{\frac{a_{t+1}\gamma_{t+1}}{n}\left(d\log\left(\frac{(a_t)^{-\frac{\alpha}{1-\alpha}}}{a_{t+1}\gamma_{t+1}}\right)+\log\left(\frac{1}{\delta_{t+1}}\right)\right)}.
    \end{align}

    By definition, we have $\gamma_t > ((d+\log(1/\delta))/n_{t})^{1/(2-\alpha)}$ for every $t = 1,2,\ldots,T$, thus, we have $a_{t+1}\gamma_{t+1} > 2^{t+1-T}((d+\log(1/\delta))/n_{t+1})^{\alpha/(2-\alpha)}$. By the definition of $n_t=2^{t-T}n$, we have
    \begin{align*}
        \log\left(\frac{(a_t)^{-\alpha/(1-\alpha)}}{a_{t+1}\gamma_{t+1}}\right) \leq &\log\left(\frac{(a_t)^{-\alpha/(1-\alpha)}}{2^{t+1-T}\left((d+\log(1/\delta))/n_{t+1}\right)^{\frac{\alpha}{2-\alpha}}}\right) \\
        = &T-t-1 +\frac{\alpha}{2-\alpha}+ \log\left(\left(\frac{n_{t}}{d+\log(1/\delta)}\right)^{\frac{\alpha}{2-\alpha}}(a_t)^{-\frac{\alpha}{1-\alpha}}\right).
    \end{align*}
    By definition, we have $a_{t}=2^{t-T}(n_t/(d+\log(1/\delta)))^{(1-\alpha)/(2-\alpha)}$. Hence, we can get $\log((n_{t}/(d+\log(1/\delta)))^{\alpha/(2-\alpha)}(a_t)^{-\alpha/(1-\alpha)})=\alpha(T-t)/(1-\alpha)$. 
    Therefore, $\log((a_t)^{-\alpha/(1-\alpha)}/{a_{t+1}\gamma_{t+1}})<(T-t)/(1-\alpha)$ and $\log((n_t/(d+\log(1/\delta)))^{\alpha/(2-\alpha)}(a_t)^{-\alpha/(1-\alpha)}) < \alpha(T-t)/(1-\alpha)$. Then, $\delta_t = (\delta/5)^{T+1-t}$ yields that $\log(1/\delta_{t+1})$ and $\log(2/\delta_{t+1})$ are both $O((T-t)\log(1/\delta))$, here we only omit numerical constants. Combining these facts and the definition of $\gamma_t$, we have $\gamma_{t+1} \leq (C_1/(1-\alpha)) ((T-t)(d+\log(1/\delta))/n_{t+1})^{1/(2-\alpha)}$. Putting all into \Cref{eq:final-bound-rhs-2}, we have
    \begin{align*}
        &\left|\Pr_{\mathcal{D}}\left(\ERin{h_{1,t}}\cap(\Delta_t\setminus\Delta_{t+1})\right) - \Pr_{\mathcal{D}}\left(\ERin{h_{2,t}}\cap(\Delta_t\setminus\Delta_{t+1})\right)\right. \nonumber\\
        &-\left(\left.\hat{\Pr}_{S_3}\left(\ERin{h_{1,t}}\cap(\Delta_t\setminus\Delta_{t+1})\right)
        -\hat{\Pr}_{S_3}\left(\ERin{h_{2,t}}\cap(\Delta_t\setminus\Delta_{t+1})\right)\right)\right|\; \nonumber\\
        \leq &\frac{4c_0C_1}{1-\alpha}\sqrt{2^{t-T}\left(\frac{n_{t+1}}{d+\log(1/\delta)}\right)^{\frac{1-\alpha}{2-\alpha}}\left(\frac{(T-t)}{n_{t+1}}\left(d+\log\left(\frac{1}{\delta}\right)\right)\right)^{\frac{1}{2-\alpha}}\left(\frac{(T-t)}{n}\left(d+\log\left(\frac{1}{\delta}\right)\right)\right)} .
    \end{align*}

    Then notice that $t = 0,1,\ldots, T-1$ in the above inequalities, so $T-t \geq 1$. We also know that $\alpha \in (0,1)$, so $(T-t)^{1+1/(2-\alpha)} \leq (T-t)^2$, which yields that 
    \begin{align}
        &\left|\Pr_{\mathcal{D}}\left(\ERin{h_{1,t}}\cap(\Delta_t\setminus\Delta_{t+1})\right) - \Pr_{\mathcal{D}}\left(\ERin{h_{2,t}}\cap(\Delta_t\setminus\Delta_{t+1})\right)\right. \nonumber\\
        &-\left(\left.\hat{\Pr}_{S_3}\left(\ERin{h_{1,t}}\cap(\Delta_t\setminus\Delta_{t+1})\right)
        -\hat{\Pr}_{S_3}\left(\ERin{h_{2,t}}\cap(\Delta_t\setminus\Delta_{t+1})\right)\right)\right|\; \nonumber\\
        \leq &\frac{4c_0C_1}{1-\alpha}\sqrt{2^{t-T}(T-t)^2\left(\frac{n_{t+1}}{d+\log\left(1/\delta\right)}\right)^{\frac{-\alpha}{2-\alpha}}\left(\frac{1}{n}\left(d+\log\left(\frac{1}{\delta}\right)\right)\right)}\nonumber.
    \end{align}
    Then, by the definition of $n_t = n 2^{t-T}$, we have
    \begin{align*}
        &\left|\Pr_{\mathcal{D}}\left(\ERin{h_{1,t}}\cap(\Delta_t\setminus\Delta_{t+1})\right) - \Pr_{\mathcal{D}}\left(\ERin{h_{2,t}}\cap(\Delta_t\setminus\Delta_{t+1})\right)\right. \\
        &-(\hat{\Pr}_{S_3}\left(\ERin{h_{1,t}}\cap(\Delta_t\setminus\Delta_{t+1})\right)
        -\hat{\Pr}_{S_3}\left(\ERin{h_{2,t}}\cap(\Delta_t\setminus\Delta_{t+1})\right))|\;
    \end{align*}
    \begin{align}
        \leq &\frac{4c_0C_1}{1-\alpha} \left(\frac{1}{2^{(1-\alpha)/(2-\alpha)}}\right)^{T-t}(T-t)\left(\frac{1}{n}\left(d+\log\left(\frac{1}{\delta}\right)\right)\right)^{\frac{1}{2}(1+\frac{1}{2-\alpha}-\frac{1-\alpha}{2-\alpha})}\nonumber\\
        = &\frac{4c_0C_1}{1-\alpha} \left(\frac{1}{2^{(1-\alpha)/(2-\alpha)}}\right)^{T-t}(T-t)\left(\frac{1}{n}\left(d+\log\left(\frac{1}{\delta}\right)\right)\right)^{\frac{1}{2-\alpha}}. \label{eq:uniform-upper-bound-rhs}
    \end{align}
    According to \Cref{alg:main-algo}, by the definition of $e_t$, we have
    \begin{align*}
        &\left|\Pr_{\mathcal{D}}\left(\ERin{h_{1,t}}\cap(\Delta_t\setminus\Delta_{t+1})\right) - \Pr_{\mathcal{D}}\left(\ERin{h_{2,t}}\cap(\Delta_t\setminus\Delta_{t+1})\right)\right. \\
        &-(\hat{\Pr}_{S_3}\left(\ERin{h_{1,t}}\cap(\Delta_t\setminus\Delta_{t+1})\right)
        -\hat{\Pr}_{S_3}\left(\ERin{h_{2,t}}\cap(\Delta_t\setminus\Delta_{t+1})\right))|\; \leq e_t. 
    \end{align*}
    Then because $\hat{h}_t$ and $h^*_{\mathcal{D}}$ are both in $\hpc_t$, we have
    \begin{equation*}
        \left|\Pr_{\mathcal{D}}(\ERin{\hat{h}_{t}}\cap(\Delta_t\setminus\Delta_{t+1})) - \Pr_{\mathcal{D}}\left(\ERin{h^*_{\mathcal{D}}}\cap(\Delta_t\setminus\Delta_{t+1})\right)\right| \leq e_t, 
    \end{equation*}
    and
    \begin{equation*}
        \left|\hat{\Pr}_{S_3}\left(\ERin{h^*_{\mathcal{D}}}\cap(\Delta_t\setminus\Delta_{t+1})\right)
        -\hat{\Pr}_{S_3}(\ERin{\hat{h}_{t}}\cap(\Delta_t\setminus\Delta_{t+1}))\right| \leq e_t.
    \end{equation*}
    Therefore, according to \Cref{alg:main-algo}, $h^*_{\mathcal{D}} \in \hat{\hpc}_t$, with probability at least $1-\delta_t$, for every $t = 0,\ldots,T-1$. On the other hand, for every $h \in \hat{\hpc}_t$, we have
    \begin{align*}
        &\left|\Pr_{\mathcal{D}}\left(\ERin{h}\cap(\Delta_t\setminus\Delta_{t+1})\right) - \Pr_{\mathcal{D}}\left(\ERin{h^*_{\mathcal{D}}}\cap(\Delta_t\setminus\Delta_{t+1})\right)\right. \\
        &-(\hat{\Pr}_{S_3}\left(\ERin{h}\cap(\Delta_t\setminus\Delta_{t+1})\right)
        -\hat{\Pr}_{S_3}\left(\ERin{h^*_{\mathcal{D}}}\cap(\Delta_t\setminus\Delta_{t+1})\right))|\; \leq e_t.
    \end{align*}
    Therefore, by using the triangle inequality and the fact that $\hat{h}_t$ is the ERM on $S_3\cap (\Delta_t\setminus\Delta_{t+1})$, we have 
    \begin{align*}
        &\left|\Pr_{\mathcal{D}}\left(\ERin{h}\cap(\Delta_t\setminus\Delta_{t+1})\right) - \Pr_{\mathcal{D}}\left(\ERin{h^*_{\mathcal{D}}}\cap(\Delta_t\setminus\Delta_{t+1})\right)\right| \nonumber\\
        \leq &|\hat{\Pr}_{S_3}(\ERin{h}\cap(\Delta_t\setminus\Delta_{t+1}))
        -\hat{\Pr}_{S_3}(\ERin{h^*_{\mathcal{D}}}\cap(\Delta_t\setminus\Delta_{t+1}))| + e_t \nonumber\\
        = &|\hat{\Pr}_{S_3}\left(\ERin{h}\cap(\Delta_t\setminus\Delta_{t+1})\right)
        - \hat{\Pr}_{S_3}(\ERin{\hat{h}_t}\cap(\Delta_t\setminus\Delta_{t+1}))\nonumber\\
        &+\hat{\Pr}_{S_3}(\ERin{\hat{h}_t}\cap(\Delta_t\setminus\Delta_{t+1})) -\hat{\Pr}_{S_3}\left(\ERin{h^*_{\mathcal{D}}}\cap(\Delta_t\setminus\Delta_{t+1})\right)| + e_t \nonumber\\
        \leq &|\hat{\Pr}_{S_3}\left(\ERin{h}\cap(\Delta_t\setminus\Delta_{t+1})\right)
        - \hat{\Pr}_{S_3}(\ERin{\hat{h}_t}\cap(\Delta_t\setminus\Delta_{t+1}))|+ e_t \nonumber.
    \end{align*}
    According to \Cref{alg:main-algo}, we have $|\hat{\Pr}_{S_3}(\ERin{h}\cap(\Delta_t\setminus\Delta_{t+1}))
        - \hat{\Pr}_{S_3}(\ERin{\hat{h}_t}\cap(\Delta_t\setminus\Delta_{t+1}))| \leq e_t$. Thus, 
    for every $h\in \hat{\hpc}_t$, with probability at least $1-\delta_t$, we have
    \begin{equation}
    \label{eq:upper-bound-rhs-final-part-1}
        \left|\Pr_{\mathcal{D}}\left(\ERin{h}\cap(\Delta_t\setminus\Delta_{t+1})\right) - \Pr_{\mathcal{D}}\left(\ERin{h^*_{\mathcal{D}}}\cap(\Delta_t\setminus\Delta_{t+1})\right)\right| \leq 2e_t.
    \end{equation}
    
    Then, for the final part $\Delta_T$, by the inequality \ref{eq:final-bound-rhs} on $\Delta_T$, we have that, with probability at least $1-\delta_T$, for every $h_{1,T}, h_{2,T} \in \hpc$, 
    \begin{align*}
        &\left|\Pr_{\mathcal{D}}\left(\ERin{h_{1,T}}\cap\Delta_T\right) - \Pr_{\mathcal{D}}\left(\ERin{h_{2,T}}\cap\Delta_T\right) -(\hat{\Pr}_{S_3}\left(\ERin{h_{1,T}}\cap\Delta_T\right)
        -\hat{\Pr}_{S_3}\left(\ERin{h_{2,T}}\cap\Delta_T\right))\right|\; \nonumber\\
        \leq &2c_0\sqrt{\frac{\Pr_{\mathcal{D}}(\Delta_T)}{n}\left(d+\log\left(\frac{1}{\delta_T}\right)\right)}.
    \end{align*}

    Notice that $\Pr_{\mathcal{D}}(\Delta_T) \leq c_2 (a_T)^{-\alpha/(1-\alpha)} = c_2((d+\log(1/\delta))/n)^{\alpha/(2-\alpha)}$, because $\mathcal{E}_T$ happens. We also have $\delta_T = \delta/5$ according to \Cref{alg:main-algo}, thus, we have 
    \begin{align}
    \label{eq:final-bound-part-2}
    &|\Pr_{\mathcal{D}}\left(\ERin{h_{1,T}}\cap\Delta_T\right) - \Pr_{\mathcal{D}}\left(\ERin{h_{2,T}}\cap\Delta_T\right) -(\hat{\Pr}_{S_3}\left(\ERin{h_{1,T}}\cap\Delta_T\right)
        -\hat{\Pr}_{S_3}\left(\ERin{h_{2,T}}\cap\Delta_T\right))|\; \nonumber\\
    \leq &2c_0c_2\sqrt{\left(\frac{d+\log(1/\delta)}{n}\right)^{\frac{\alpha}{2-\alpha}}\frac{1}{n}\left(d+\log\left(\frac{1}{\delta}\right)\right)} = 2c_0c_2 \left(\frac{d+\log(1/\delta)}{n}\right)^{\frac{1}{2-\alpha}},
    \end{align}
    because $c_2\geq 1$, which implies $c_2\geq c_2^{1/2}$. Therefore, we have 
    \begin{equation*}
        \left|\Pr_{\mathcal{D}}\left(\ERin{\hat{h}_{T}}\cap\Delta_T\right) - \Pr_{\mathcal{D}}\left(\ERin{h^*_{\mathcal{D}}}\cap\Delta_T\right)\right| \leq 2c_0c_2 \left(\frac{d+\log(1/\delta)}{n}\right)^{\frac{1}{2-\alpha}}, 
    \end{equation*}
    and
    \begin{equation*}
        \left|\hat{\Pr}_{S_3}\left(\ERin{h^*_{\mathcal{D}}}\cap\Delta_T\right)
        -\hat{\Pr}_{S_3}\left(\ERin{\hat{h}_{T}}\cap\Delta_T\right)\right| \leq 2c_0c_2 \left(\frac{d+\log(1/\delta)}{n}\right)^{\frac{1}{2-\alpha}}.
    \end{equation*}
    Thus, $\h^*_{\mathcal{D}} \in \hat{\hpc}_T$, with probability at least $1-\delta_T$ as well. On the other hand, according to \Cref{alg:main-algo} and the definition of $C_1$, for the same reason, with probability at least $1-\delta_T$, for every $h \in \hat{\hpc}_T$, we have
    \begin{equation}
    \label{eq:upper-bound-rhs-final-part-2}
        \left|\Pr_{\mathcal{D}}\left(\ERin{h}\cap\Delta_T\right) - \Pr_{\mathcal{D}}\left(\ERin{h^*_{\mathcal{D}}}\cap\Delta_T\right)\right| \leq 2C_2 \left(\frac{d+\log(1/\delta)}{n}\right)^{\frac{1}{2-\alpha}}.
    \end{equation}
    
    Combining all the inequalities for $t = 0,1,\ldots,T-1$, i.e., inequality \ref{eq:upper-bound-rhs-final-part-1}, by union bound, we have 
    \begin{align*}
        &\sum_{t = 0}^{T-1} \Pr_{\mathcal{D}}\left(\ERin{h_{\textrm{ALG}}}\cap(\Delta_t\setminus\Delta_{t+1})\right) - \Pr_{\mathcal{D}}\left(\vphantom{\hat{h}}\ERin{h^*_{\mathcal{D}}}\cap(\Delta_t\setminus\Delta_{t+1})\right) \\ 
        \leq &2\sum_{t = 0}^{T-1} e_t \\
        \leq &2C_2\sum_{t = 0}^{T-1} \left(\frac{1}{2^{\frac{1-\alpha}{2-\alpha}}}\right)^{T-t}(T-t)\left(\frac{1}{n}\left(d+\log\left(\frac{1}{\delta}\right)\right)\right)^{\frac{1}{2-\alpha}} \\
        = &2C_2\left(\sum_{t = 0}^{T-1} \left(\frac{1}{2^{\frac{1-\alpha}{2-\alpha}}}\right)^{T-t}(T-t)\right)\left(\frac{1}{n}\left(d+\log\left(\frac{1}{\delta}\right)\right)\right)^{\frac{1}{2-\alpha}},
    \end{align*}
    holds with probability at least $1-\sum_{t=1}^{T}\delta_t \geq 1-2\delta/5$. 
    Then notice that $\sum_{t = 0}^{T-1} (1/2^{(1-\alpha)/(2-\alpha)})^{T-t}(T-t)$ is smaller than $\lim_{n\rightarrow\infty}\sum_{i = 1}^{n}i/(2^{\frac{1-\alpha}{2-\alpha}})^i = 1/(2^{\frac{1-\alpha}{2-\alpha}}+ 2^{-\frac{1-\alpha}{2-\alpha}}-2) = c_1$. Therefore, we have 
    \begin{equation}
    \label{eq:final-bound-rhs-part-1}
        \sum_{t = 0}^{T-1} \Pr_{\mathcal{D}}\left(\ERin{h_{\textrm{ALG}}}\cap(\Delta_t\setminus\Delta_{t+1})\right) - \Pr_{\mathcal{D}}\left(\vphantom{\hat{h}}\ERin{h^*_{\mathcal{D}}}\cap(\Delta_t\setminus\Delta_{t+1})\right) \leq 2c_1C_2 \left(\frac{1}{n}\left(d+\log\left(\frac{1}{\delta}\right)\right)\right)^{\frac{1}{2-\alpha}},
    \end{equation}
    with probability at least $1-2\delta/5$. 
    
    
    Finally, combining the inequality \ref{eq:final-bound-rhs-part-1} and the inequality \ref{eq:upper-bound-rhs-final-part-2}, and the union bound, we have that, with probability at least $1-2\delta/5-\delta/5 = 1-3\delta/5$, 
    \begin{align}
        \Pr_{\mathcal{D}}\left(\ERin{h_{\mathrm{ALG}}}\right) - \Pr_{\mathcal{D}}\left(\ERin{h^*_{\mathcal{D}}}\right) = &\sum_{t = 0}^T \Pr_{\mathcal{D}}\left(\ERin{h_{\mathrm{ALG}}}\cap(\Delta_t\setminus\Delta_{t+1})\right) - \Pr_{\mathcal{D}}\left(\vphantom{\hat{h}}\ERin{h^*_{\mathcal{D}}}\cap(\Delta_t\setminus\Delta_{t+1})\right)\nonumber\\
        \leq &2c_1C_2 \left(\frac{1}{n}\left(d+\log\left(\frac{1}{\delta}\right)\right)\right)^{\frac{1}{2-\alpha}} + 2C_2 \left(\frac{1}{n}\left(d+\log\left(\frac{1}{\delta}\right)\right)\right)^{\frac{1}{2-\alpha}}\nonumber \\
        = &2(c_1+1)C_2 \left(\frac{1}{n}\left(d+\log\left(\frac{1}{\delta}\right)\right)\right)^{\frac{1}{2-\alpha}},
    \end{align}
    as desired. 
\end{proof}
After that, we only need to show that, with probability at least $1-2\delta/5$, for all $t$, $\mathcal{E}_t$ happens. After that, the union bound yields \Cref{thm:optimal-rates-main}. First, we need to prove multiplicative bounds on the error of the estimation of the excess risk and the pseudo-distance as \Cref{lem:prune} and \Cref{lem:uniform-chernoff-pseudo-distance-estimation}, respectively. 

\begin{lemma}
  \label{lem:prune}
Let $\hpc$ be a concept class with $\vcdim=d$ and $h^*_{\mathcal{D}}\in \hpc$. Let $\mathcal{D}$ be a distribution that satisfies a Tsybakov noise condition with parameter $(a,\alpha)$ w.r.t. $\hpc$. Let $m\in\naturalnumber$, $\delta_0,\delta\in(0,1)$ and $\delta < \delta_0$. Let $\Delta\subseteq\mathcal{X}$ and $\gamma = ((d\log((m/(d+\log(1/\delta_0)))^{\alpha/(2-\alpha)}\Pr_{\mathcal{D}}(\Delta))+\log(1/\delta))/m)^{1/(2-\alpha)}$. Let $\mathrm{ERM}(\hpc;S_{m})=\{h\in\hpc: h=\argmin_{h\in\hpc}\hat{\Pr}_{S_{m}}(\mathrm{ER}(h))\}$ denote the class of empirical risk minimizers with i.i.d.\ samples $S_{m}=\{(x_{i}, y_{i})\}_{i=1}^{m}\sim\mathcal{D}^{m}$. Let $\hat{h}\in\mathrm{ERM}(\hpc;S_{m}\cap\Delta)$. Then, for any $h\in\hpc$, with probability at least $1-\delta$, 
\begin{align*}
    &\hat{\Pr}_{S_m}(\ER{h}\cap\Delta) - \hat{\Pr}_{S_m}(\ERin{\hat{h}}\cap\Delta) \leq \gamma \;\;\Longrightarrow\;\;\Pr_{\mathcal{D}}\left(\ER{h}\cap\Delta\right) - \Pr_{\mathcal{D}}\left(\ER{h^*_{\mathcal{D}}}\cap\Delta\right) \leq c\gamma\; ,\\
    &\Pr_{\mathcal{D}}\left(\ER{h}\cap\Delta\right) - \Pr_{\mathcal{D}}\left(\ER{h^*_{\mathcal{D}}}\cap\Delta\right) \leq \gamma\;\;\Longrightarrow\;\;\left|\hat{\Pr}_{S_m}(\ER{h}\cap\Delta) - \hat{\Pr}_{S_m}(\ERin{h^*_{\mathcal{D}}}\cap\Delta)\right| \leq c\gamma,
\end{align*}
for a constant $c=16ac_{0}^2+2c_{0}+2$ as defined at the beginning. 
\end{lemma}
The proof of this lemma is based on the uniform Bernstein inequality (\Cref{lem:uniform-bernstein-inequality}) and the well-known fact that the Tsybakov noise condition with parameter $(a,\alpha)$ implies the Bernstein class condition with parameter $(a,\alpha)$.\footnote{This is a well-known fact in statistical theory, c.f. \cite{massart2006risk}.} 

To prove the first statement, by the triangle inequality, we have 
\begin{align}
\label{eq:excess-risk-on-delta-t}
    &\left|\Pr_{\mathcal{D}}\left(\ER{h}\cap \Delta\right)-\Pr_{\mathcal{D}}\left(\ER{h^*_{\mathcal{D}}}\cap\Delta\right)\right| \nonumber\\
    \leq &\left|\hat{\Pr}_{S_m}\left(\ER{h}\cap \Delta\right) - \hat{\Pr}_{S_m}\left(\ER{h^*_{\mathcal{D}}}\cap\Delta\right)\right| \nonumber\\
    & +\left|(\hat{\Pr}_{S_m}\left(\ER{h}\cap\Delta\right)-\hat{\Pr}_{S_m}\left(\ER{h^*_{\mathcal{D}}}\cap\Delta\right)) - \left(\Pr_{\mathcal{D}}\left(\ER{h}\cap \Delta\right) - \Pr_{\mathcal{D}}\left(\ER{h^*_{\mathcal{D}}}\cap\Delta\right)\right)\right|.
\end{align}
Then we want to claim that the first term in the right-hand side of inequality \ref{eq:excess-risk-on-delta-t} can be bounded by $\gamma$. Notice that we have
\begin{align*}
    &\left|\hat{\Pr}_{S_m}\left(\ER{h}\cap \Delta\right) - \hat{\Pr}_{S_m}\left(\ER{h^*_{\mathcal{D}}}\cap\Delta\right)\right| \\
    = &\left|\hat{\Pr}_{S_m}\left(\ER{h}\cap \Delta\right) - \hat{\Pr}_{S_m}(\ERin{\hat{h}}\cap \Delta)+ \hat{\Pr}_{S_m}(\ERin{\hat{h}}\cap \Delta)-\hat{\Pr}_{S_m}\left(\ER{h^*_{\mathcal{D}}}\cap\Delta\right)\right|.
\end{align*}
Notice that $\hat{h}$ is the ERM predictor on $S_m\cap \Delta$, $\hat{\Pr}_{S_m}\left(\ERin{h}\cap \Delta\right) - \hat{\Pr}_{S_m}(\ERin{\hat{h}}\cap \Delta) \geq 0$ for any $h\in \hpc$. Therefore, we have
\begin{equation*}
    \left|\hat{\Pr}_{S_m}\left(\ER{h}\cap \Delta\right) - \hat{\Pr}_{S_m}\left(\ER{h^*_{\mathcal{D}}}\cap\Delta\right)\right| \leq \left|\hat{\Pr}_{S_m}\left(\ER{h}\cap \Delta\right) - \hat{\Pr}_{S_m}(\ERin{\hat{h}}\cap \Delta)\right| \leq \gamma.
\end{equation*}

On the other hand, the second term of the right-hand side of inequality \ref{eq:excess-risk-on-delta-t} is bounded by the uniform Bernstein inequality. Then, we can divide it into two cases. One is that $\gamma$ is greater than the second term, then we already get the results we want. The other is that $\gamma$ is smaller than or equal to the second term, then the right-hand side of the inequality \ref{eq:excess-risk-on-delta-t} is upper bounded by two times the upper bound from the uniform Bernstein inequality. Combining this with the Bernstein class condition, for the same reason as the upper bound of the excess risk of ERM learners under Tsybakov noise, we can get the desired upper bound. 

For the second statement, we first use the Bernstein class condition to get the upper bound of the pseudo-distance between $h$ and $h^*$. Then, by the uniform Bernstein inequality and the definition of $\gamma$, we have
\begin{equation*}
    \left|\left(\hat{\Pr}_{S_m}(\ER{h}\cap\Delta) - \hat{\Pr}_{S_m}(\ERin{h^*_{\mathcal{D}}}\cap\Delta)\right) - \left(\hat{\Pr}_{\mathcal{D}}(\ER{h}\cap\Delta) - \hat{\Pr}_{\mathcal{D}}(\ERin{h^*_{\mathcal{D}}}\cap\Delta)\right)\right| \leq 2c_0\gamma.
\end{equation*}
Then, by the triangle inequality, we have 
\begin{equation*}
    \left|\hat{\Pr}_{S_m}(\ER{h}\cap\Delta) - \hat{\Pr}_{S_m}(\ERin{h^*_{\mathcal{D}}}\cap\Delta)\right| \leq (2c_0+1)\gamma,
\end{equation*}
as desired. The complete proof is postponed to \Cref{sec:omitted-proofs} to save space. 

\begin{remark}
    Combining the second statement of this lemma with \Cref{lem:excess-risk-under-tsybakov-noise}, specifically, \Cref{rm:tsybakov-erm-bound}, and the fact that $\delta_0 > \delta$, we know that $h^*_{\mathcal{D}} \in \hpc_t$, for every $t = 0,1,2,\ldots,T-1$, which verifies the last bullet of the event $\mathcal{E}_t$. Thus, we omit the last bullet in the rest of the proof.
\end{remark}
The following lemma shows that $m_t$ samples are sufficient to build a multiplicative bound on the estimation of the pseudo-distance between any pair of concepts. The proof is also based on the uniform Bernstein inequality, and for brevity, we put the complete proof of this lemma in \Cref{sec:omitted-proofs}.
\begin{lemma}  
  \label{lem:uniform-chernoff-pseudo-distance-estimation}
Let $\hpc$ be a concept class with $\vcdim=d$. Let $\mathcal{D}$ be any distribution over $\mathcal{X}\times\{0,1\}$. Let $\delta_0 \in (0,1)$. Let $\Delta\subseteq\mathcal{X}$. Let $m = (8c_0^2/B)(d\log(1/B)+\log(1/\delta))\in\naturalnumber$. Let $S_{m}=\{(x_{i}, y_{i})\}_{i=1}^{m}\sim\mathcal{D}^{m}$ be i.i.d.\ samples. 
When $B\leq \min(\delta,1/(32c_0^2\log c_0))$,\footnote{Notice that this is a constant, and when using this lemma, the value of $B$ is a function of $n$. By taking $n$ large enough, the value of $B$ can be arbitrarily small.} for any pair $f,g\in\hpc$, we have with probability at least $1-\delta$,
\begin{align*}
    &\hat{\Pr}_{S_{m}}(\{f\neq g\}\cap\Delta) \geq B \;\;\Longrightarrow\;\; \Pr_{\mathcal{D}}(\{f\neq g\}\cap\Delta) \geq B/2 , \\
    &\hat{\Pr}_{S_{m}}(\{f\neq g\}\cap\Delta) \leq B \;\;\Longrightarrow\;\; \Pr_{\mathcal{D}}(\{f\neq g\}\cap\Delta) \leq 2B .
\end{align*}
\end{lemma}

Then, we need to upper-bound $\Delta_t$ for every $t$. Due to the previous two lemmas, the empirical excess risk and the empirical pseudo-distance are both good estimations of the population excess risk and pseudo-distance, respectively. Therefore, we can build the upper bound of $\Delta_t$ based on the population excess risk and pseudo-distance, stated as the following lemma. 
\begin{lemma}
  \label{lem:upper-bound-of-p-Delta}
Let $\gamma, B>0$. Let $\mathcal{X}$ be (some subset of) the instance space. Let $\Delta\subseteq\mathcal{X}$ be a region obtained via the following procedure:
\begin{algorithm}[H]
\caption{The Population Version of the Subroutine in each for-loop of \Cref{alg:main-algo}}
\label{alg:population-subroutine}
\begin{algorithmic}[1]  
    \State Let $V=\left\{h\in\hpc: \Pr_{\mathcal{D}}(\mathrm{ER}(h)\cap\mathcal{X}) \leq \Pr_{\mathcal{D}}(\mathrm{ER}(h^{*}_{\mathcal{D}})\cap\mathcal{X})+\gamma\right\}$ and $\Delta=\emptyset$
    \While{$\exists h_{1}, h_{2}\in V$ such that $\Pr_{\mathcal{D}}(\{h_{1}\neq h_{2}\}\cap\mathcal{X}\setminus\Delta) > B$}
        \State Let $h_{1},h_{2}=\argmax_{(h_{1}^{'}, h_{2}^{'})}\Pr_{\mathcal{D}}(\{h_{1}^{'}\neq h_{2}^{'}\}\cap\mathcal{X}\setminus\Delta)$
        \State Update $\Delta \leftarrow \Delta\cup\{h_{1}\neq h_{2}\}$
    \EndWhile
\end{algorithmic}
\end{algorithm}

Assume the Tsybakov noise model with parameter $(a,\alpha)$. We have
\begin{equation*}
    \Pr_{\mathcal{D}}(\Delta) \leq 2a'\left(\frac{2\gamma}{B}\right)^{\frac{\alpha}{1-\alpha}} ,
\end{equation*}
where $a' = (1-\alpha)(2\alpha)^{\alpha/(1-\alpha)}a^{1/(1-\alpha)}$ as defined at the beginning.
\end{lemma}

With the three lemmas above in hand, we can prove the following lemma and use it as the inductive step.
\begin{lemma}
  \label{lem:approximate-linear-BCC}
Suppose $\mathcal{E}_{t-1}$ happens, $\mathcal{E}_{t}$ will happen with probability at least $1-\delta_{t}$, where $\delta_t = (\delta/5)^{T+1-t}$.
\end{lemma}
The proof of this lemma requires the $\Delta_{t-1}$ upper bound from the assumption that $\mathcal{E}_{t-1}$ happens. From \Cref{lem:prune}, we know the first bullet is correct. The first statement of \Cref{lem:uniform-chernoff-pseudo-distance-estimation} yields the second bullet. The first bullet of event $\mathcal{E}_t$ and the second statement of \Cref{lem:uniform-chernoff-pseudo-distance-estimation}, combining with the fact that every pair found by the algorithm has a large empirical pseudo-distance, prove the required upper bound of $\Delta_t$ and $s_t$. The complete proof is deferred to \Cref{sec:omitted-proofs} for brevity. 

Finally, we prove the following lemma that the assumption we made in the proof of \Cref{thm:optimal-rates-main} holds, with probability at least $1-2\delta/5$ by induction.
\begin{lemma}
\label{lem:good-events-hold}
    With probability at least $1-2\delta/5$, for all $t$, $\mathcal{E}_t$ happens.
\end{lemma}
From \Cref{lem:approximate-linear-BCC}, it suffices to verify the base case, which can be done by following the same reasoning as in the proof of \Cref{lem:approximate-linear-BCC}, and using \Cref{lem:prune} and \Cref{lem:uniform-chernoff-pseudo-distance-estimation}. Finally, a union bound arguing that all $\mathcal{E}_t$ happen simultaneously with probability at least $1- 2\delta/5$ will complete the proof.

\section{Conclusion and Future Works}
  \label{sec:conclusion}
In this work, we proposed a new proper learner that guarantees the optimal excess risks of PAC learning under Tsybakov noise, resolving an open problem lasting for decades. Our new algorithm leverages the idea of adaptively partitioning the instance space into regions with different noise levels, tackling each region separately by running ERMs, and finally returning a hypothesis in the concept class that satisfies a specific error constraint for each region. 

We believe our techniques might find further applications in learning theory. For instance, the recent work of \cite{hanneke2024revisiting} provided the first learning algorithm that achieves the optimal first-order excess risks under the agnostic setting. While being optimal for the higher-order term, their upper bound has a sub-optimal lower-order term. We believe that the idea of noise isolation with multiple regions can be useful to refine the lower-order term in the first-order excess risk bound for agnostic PAC learning (though in this case, an improper learner is necessary). Indeed, our technique of partitioning the instance space into regions shares connections with several recent advances in non-realizable learning theory \cite{hanneke2024revisiting,hanneke2025agnostic}. It is worthwhile exploring this approach under other non-realizable learning scenarios, and even more interesting if a unified theory towards understanding this idea can be developed.

Our work also leaves several interesting research problems for future work. Note that our learning algorithm requires knowing the Tsybakov noise condition parameters $(a,\alpha)$ in advance. A natural follow-up question is whether it is possible to design an optimal learning algorithm under Tsybakov noise that automatically adapts to these parameters. Secondly, our learning algorithm is $\delta$-dependent (requiring setting the confidence in advance and taking $\delta$ as a hyperparameter). This is in contrast to the known optimal PAC learners under the realizable setting, which are all $\delta$-independent. Hence, an interesting question is, can we develop certain optimal learning algorithms under Tsybakov noise that can automatically adapt to any $\delta$. We would like to mention that, recent advances in designing optimal learners under various learning settings all admit $\delta$-dependent algorithms \cite{hanneke2024revisiting,hanneke2025agnostic}. Finally, it remains open what the optimal excess risk bound in expectation is and what algorithms achieve it. Note that our algorithm is $\delta$-dependent and thus does not automatically lead to an optimal in-expectation guarantee, whereas a positive solution to the above question of $\delta$-independent optimal learners would immediately imply an optimal in-expectation excess risk bound. 
Another natural question concerns computational efficiency.  The work of \cite{diakonikolas2021efficiently} provides polynomial-time learnability of Halfspaces under Tsybakov noise.  It is natural to consider whether it is possible to achieve the optimal error guarantee, established in the present work, with a polynomial-time method, for special classes such as Halfspaces.  Beyond such specialized results, it is also natural to ask whether, for every concept class, there is a variant of our technique expressible as an efficient \emph{reduction} to an ERM oracle.  

\newpage
\appendix
\section{Omitted proofs in Section~\ref{sec:main-results}}
  \label{sec:omitted-proofs}

\begin{lemma}[\textbf{Restate of \Cref{lem:number-of-samples}}]
\label{lem:number-of-samples-re}
    For any $n = \Omega((d+\log(1/\delta))(4(2-\alpha)/(1-\alpha))^{(8(2-\alpha)/(1-\alpha))})$, the total number of fresh samples required by \Cref{alg:main-algo} is at most $4n$, if $\mathcal{E}_t$ happens for all $t$.
\end{lemma}

\begin{proof}[Proof of \Cref{lem:number-of-samples-re} (Under the assumption that for all $t$, $\mathcal{E}_t$ happens)]

    According to \Cref{alg:main-algo}, there are three groups of fresh samples that are used in the algorithm, which are $S_1$, $S_2$, and $S_3$. Notice that we have $|S_3| = n$ and $|S_1| = \sum_{t=1}^T |S_{1,t}| = \sum_{t=1}^T n_t = \sum_{t=1}^T n2^{t-T} \leq 2n$.

    After that, we only need to show that $|S_2| = \sum_{t=1}^T s_t|S_{2,t}|$. Then by the definition, we have $|S_2| = \sum_{t=1}^T s_t m_t$ and $m_t = 8c_0^2(d\log(1/a_t\gamma_t)+\log(2/\delta_t))/(a_t\gamma_t)$. Then because $\mathcal{E}_t$ happens for all $t$, we know that $s_t \leq  2c_2 a_t^{-1/(1-\alpha)} \gamma_t^{-1}$. Therefore, we have
    \begin{align*}
        |S_2| = \sum_{t=1}^T s_t m_t
        \leq \sum_{t=1}^T \frac{16c_0^2c_2}{a_t^{\frac{2-\alpha}{1-\alpha}}\gamma_t^2}\left(d\log\left(\frac{1}{a_t\gamma_t}\right)+\log\left(\frac{2}{\delta_t}\right)\right).
    \end{align*}
    By definition, $\log(2/\delta_t) \leq T\log(1/\delta)$. Notice that $a_{t}\gamma_{t} > 2^{t-T}((d+\log(1/\delta))/n_{t})^{\alpha/(2-\alpha)}$, thus, by taking $C_3 = 16c_0^2c_2(\alpha/(2-\alpha))\log(n/(d+\log(1/\delta)))$. 
    \begin{align*}
        |S_2| \leq \sum_{t=1}^T \frac{C_3T}{a_t^{\frac{2-\alpha}{1-\alpha}}\gamma_t^2}\left(d+\log\left(\frac{1}{\delta}\right)\right)
        = C_3T\left(d+\log\left(\frac{1}{\delta}\right)\right) \sum_{t=1}^T \frac{1}{a_t^{\frac{2-\alpha}{1-\alpha}}\gamma_t^2}.
    \end{align*}
    Notice that $\delta_t < \delta$, we know that $\gamma_t > ((d+\log(1/\delta))/n_t)^{1/(2-\alpha)}$. Therefore, we have
    \begin{align*}
        |S_2| \leq &C_3T\left(d+\log\left(\frac{1}{\delta}\right)\right) \sum_{t=1}^T \frac{1}{\left(2^{t-T}(n_t/(d+\log(1/\delta)))^{\frac{1-\alpha}{2-\alpha}}\right)^{\frac{2-\alpha}{1-\alpha}}\left((d+\log(1/\delta))/n_t\right)^{\frac{2}{2-\alpha}}}\\
        = &C_3T\left(d+\log\left(\frac{1}{\delta}\right)\right)^{\frac{2-2\alpha}{2-\alpha}}\sum_{t= 1}^T 2^{(T-t)(\frac{2-\alpha}{1-\alpha})}n_t^{\frac{\alpha}{2-\alpha}}\\
        = &C_3T\left(d+\log\left(\frac{1}{\delta}\right)\right)^{\frac{2-2\alpha}{2-\alpha}}n^{\frac{\alpha}{2-\alpha}}\sum_{t= 1}^T 2^{(T-t)(\frac{4-5\alpha+2\alpha^2}{2-3\alpha+\alpha^2})}.
    \end{align*}
    The last equation comes from the definition of $n_t = n2^{t-T}$. Notice that $(4-5\alpha+2\alpha^2)/(2-3\alpha+\alpha^2) < 3$. Therefore, we have $\sum_{t= 1}^T 2^{(T-t)((4-5\alpha+2\alpha^2)/(2-3\alpha+\alpha^2))} < \sum_{t= 1}^T 2^{3(T-t)}$. Then, notice that by definition, we have $T < 2\log\log(n/(d+\log(1/\delta)))$, therefore, we have $\sum_{t= 1}^T 2^{3(T-t)} \leq 8^T = (\log(n/(d+\log(1/\delta))))^6$ and $T< \log(n/(d+\log(1/\delta)))$. By putting this term back, we get 
    \begin{equation*}
        |S_2| \leq 16c_0^2c_2\frac{\alpha}{2-\alpha}\left(d+\log\left(\frac{1}{\delta}\right)\right)^{\frac{2-2\alpha}{2-\alpha}}n^{\frac{\alpha}{2-\alpha}}\left(\log\left(\frac{n}{d+\log(1/\delta)}\right)\right)^8.
    \end{equation*}
    Divided by $n$ on both sides, we have
    \begin{equation*}
        \frac{|S_2|}{n} \leq 16c_0^2c_2\frac{\alpha}{2-\alpha} \left(\frac{d+\log(1/\delta)}{n}\right)^{\frac{2-2\alpha}{2-\alpha}}\left(\log\left(\frac{n}{d+\log(1/\delta)}\right)\right)^8.
    \end{equation*}
    Then, because $\alpha \in (0,1)$, $(2-2\alpha)/(2-\alpha)>0$. Therefore, we want to show that when $n = \Omega((d+\log(1/\delta))(4(2-\alpha)/(1-\alpha))^{(8(2-\alpha)/(1-\alpha))})$, we have $|S_2|/n \leq 1$. Notice that $n = \Omega((d+\log(1/\delta))(4(2-\alpha)/(1-\alpha))^{(8(2-\alpha)/(1-\alpha))})$, we have $\log(n/(d+\log(1/\delta)))\geq 16c_0^2c_2\alpha/(2-\alpha)$ by choosing the constant properly, this is equivalent to showing that 
    \begin{equation*}
        \left(\log\left(\frac{n}{d+\log\left(1/\delta\right)}\right)\right)^8 \leq \left(\frac{n}{d+\log\left(1/\delta\right)}\right)^{\frac{2-2\alpha}{2-\alpha}}.
    \end{equation*}
    Taking the logarithm on both sides, we have that it is equivalent to showing that 
    \begin{equation}
    \label{eq:S-2-bound}
        8\log\log\left(\frac{n}{d+\log\left(1/\delta\right)}\right) \leq {\frac{2-2\alpha}{2-\alpha}}\log\left(\frac{n}{d+\log\left(1/\delta\right)}\right).
    \end{equation}

    By the well-known fact that $x \geq 2a\log a \implies x\geq a\log x$.\footnote{This result is well-known and can be found as Lemma A.1 in Appendix A of the famous textbook \cite{shalev2014understanding}.} We know that $\log(n/(d+\log(1/\delta))))\geq 8(2-\alpha)/(1-\alpha)\log(4(2-\alpha)/(1-\alpha))$ is sufficient for inequality \ref{eq:S-2-bound}. Notice that we have $n = \Omega((d+\log(1/\delta))(4(2-\alpha)/(1-\alpha))^{(8(2-\alpha)/(1-\alpha))})$, therefore, $\log(n/(d+\log(1/\delta))))\geq 8(2-\alpha)/(1-\alpha)\log(4(2-\alpha)/(1-\alpha))$ as desired. Combining all these three bounds, we get that when $n = \Omega((d+\log(1/\delta))(4(2-\alpha)/(1-\alpha))^{(8(2-\alpha)/(1-\alpha))})$ and $\mathcal{E}_t$ happens for all $t$, the total number of fresh samples required by \Cref{alg:main-algo} is $4n$. This completes the proof.
\end{proof} 

\begin{lemma}[\textbf{Restate of \Cref{lem:prune}}]
  \label{lem:prune-re}
Let $\hpc$ be a concept class with $\vcdim=d$ and $h^*_{\mathcal{D}}\in \hpc$. Let $\mathcal{D}$ be a distribution that satisfies a Tsybakov noise condition with parameter $(a,\alpha)$ w.r.t. $\hpc$. Let $m\in\naturalnumber$, $\delta_0,\delta\in(0,1)$ and $\delta < \delta_0$. Let $\Delta\subseteq\mathcal{X}$ and $\gamma = ((d\log((m/(d+\log(1/\delta_0)))^{\alpha/(2-\alpha)}\Pr_{\mathcal{D}}(\Delta))+\log(1/\delta))/m)^{1/(2-\alpha)}$. Let $\mathrm{ERM}(\hpc;S_{m})=\{h\in\hpc: h=\argmin_{h\in\hpc}\hat{\Pr}_{S_{m}}(\mathrm{ER}(h))\}$ denote the class of empirical risk minimizers with i.i.d.\ samples $S_{m}=\{(x_{i}, y_{i})\}_{i=1}^{m}\sim\mathcal{D}^{m}$. Let $\hat{h}\in\mathrm{ERM}(\hpc;S_{m}\cap\Delta)$. Then, for any $h\in\hpc$, with probability at least $1-\delta$, 
\begin{align*}
    \hat{\Pr}_{S_m}(\ER{h}\cap\Delta) - \hat{\Pr}_{S_m}(\ERin{\hat{h}}\cap\Delta) \leq \gamma \;\;&\Longrightarrow\;\;\Pr_{\mathcal{D}}\left(\ER{h}\cap\Delta\right) - \Pr_{\mathcal{D}}\left(\ER{h^*_{\mathcal{D}}}\cap\Delta\right) \leq c\gamma\\
    \Pr_{\mathcal{D}}\left(\ER{h}\cap\Delta\right) - \Pr_{\mathcal{D}}\left(\ER{h^*_{\mathcal{D}}}\cap\Delta\right) \leq \gamma\;\;&\Longrightarrow\;\;\left|\hat{\Pr}_{S_m}(\ER{h}\cap\Delta) - \hat{\Pr}_{S_m}(\ERin{h^*_{\mathcal{D}}}\cap\Delta)\right| \leq c\gamma,
\end{align*}
for constant $c$ defined in the \Cref{alg:main-algo} as $16ac_{0}^2+2c_{0}+2$, where $c_0$ is a numerical constant.
\end{lemma}

\begin{proof}[Proof of \Cref{lem:prune-re}]
Let us prove the first statement first. To get a high level idea, the proof consists of bounding the distance between $\Pr_{\mathcal{D}}(\ERin{h}\cap\Delta)-\Pr_{\mathcal{D}}(\ERin{h^*_{\mathcal{D}}}\cap\Delta)$ and $\hat{\Pr}_{S_{m}}(\ERin{h}\cap\Delta)-\hat{\Pr}_{S_{m}}(\ERin{h^*_{\mathcal{D}}}\cap\Delta)$ using the uniform Bernstein inequality and approximating the Bayes-optimal classifier $h^{*}_{\mathcal{D}}$ by the ERM predictor $\hat{h}$. Concretely, we have with probability at least $1-\delta$,
\begin{align*}
    &\left|\Pr_{\mathcal{D}}\left(\ER{h}\cap \Delta\right)-\Pr_{\mathcal{D}}\left(\ER{h^*_{\mathcal{D}}}\cap\Delta\right)\right| \\
    \stackrel{\text{\eqmakebox[lemma-prune][c]{\text{\tiny Triangle ineq.}}}}{\leq} &\left|\hat{\Pr}_{S_m}\left(\ER{h}\cap \Delta\right) - \hat{\Pr}_{S_m}\left(\ER{h^*_{\mathcal{D}}}\cap\Delta\right)\right| \\
    & +\left|(\hat{\Pr}_{S_m}\left(\ER{h}\cap\Delta\right)-\hat{\Pr}_{S_m}\left(\ER{h^*_{\mathcal{D}}}\cap\Delta\right)) - \left(\Pr_{\mathcal{D}}\left(\ER{h}\cap \Delta\right) - \Pr_{\mathcal{D}}\left(\ER{h^*_{\mathcal{D}}}\cap\Delta\right)\right)\right| \\
    \stackrel{\text{\eqmakebox[lemma-prune][c]{\text{\tiny Lemma~\ref{lem:uniform-bernstein-inequality}}}}}{\leq} &\left|\hat{\Pr}_{S_m}\left(\ER{h}\cap \Delta\right) - \hat{\Pr}_{S_m}\left(\ER{h^*_{\mathcal{D}}}\cap\Delta\right)\right| \\
    &+\underbrace{2 c_{0}\sqrt{\frac{\Pr_{\mathcal{D}}(\{h\neq h^*_{\mathcal{D}}\}\cap\Delta)}{m}\left(d\log\left(\frac{\Pr_{\mathcal{D}}(\Delta)}{\Pr_{\mathcal{D}}(\{h\neq h^*_{\mathcal{D}}\}\cap\Delta)} \land \frac{m\Pr_{\mathcal{D}}(\Delta)}{d}\right)+\log\left(\frac{1}{\delta}\right)\right)}}_{T}\\
    \stackrel{\text{\eqmakebox[lemma-prune][c]{}}}{=} &\left|\hat{\Pr}_{S_m}\left(\ER{h}\cap\Delta\right) - \hat{\Pr}_{S_m}(\ERin{\hat{h}}\cap\Delta) + \hat{\Pr}_{S_m}(\ERin{\hat{h}}\cap\Delta) - \hat{\Pr}_{S_m}\left(\ER{h^*_{\mathcal{D}}}\cap\Delta\right)\right| + T \\
    \stackrel{\text{\eqmakebox[lemma-prune][c]{}}}{\leq} &\left|\hat{\Pr}_{S_m}\left(\ER{h}\cap\Delta\right) - \hat{\Pr}_{S_m}(\ERin{\hat{h}}\cap\Delta)\right| + T \\
    \stackrel{\text{\eqmakebox[lemma-prune][c]{}}}{\leq} &\gamma + T ,
\end{align*}
where the third inequality follows from the fact that $\hat{h}\in\mathrm{ERM}(\hpc;S_{m}\cap\Delta)$. Then, it suffices to discuss the following two cases: (i) If $T\leq\gamma$, we have $|\Pr_{\mathcal{D}}(\ERin{h}\cap \Delta)-\Pr_{\mathcal{D}}(\ERin{h^*_{\mathcal{D}}}\cap\Delta)| \leq 2\gamma$ as desired. (ii) Otherwise if $T>\gamma$, we have 
\begin{align}
    &\left|\Pr_{\mathcal{D}}\left(\ER{h}\cap \Delta\right)-\Pr_{\mathcal{D}}\left(\ER{h^*_{\mathcal{D}}}\cap\Delta\right)\right| \leq 2T \nonumber \\
    = &4c_0 \sqrt{\frac{\Pr_{\mathcal{D}}(\{h\neq h^*_{\mathcal{D}}\}\cap\Delta)}{m}\left(d\log\left(\frac{\Pr_{\mathcal{D}}(\Delta)}{\Pr_{\mathcal{D}}(\{h\neq h^*_{\mathcal{D}}\}\cap\Delta)} \land \frac{m\Pr_{\mathcal{D}}(\Delta)}{d}\right)+\log\left(\frac{1}{\delta}\right)\right)} \nonumber \\
    \leq &4c_0 \sqrt{\frac{\Pr_{\mathcal{D}}(\{h\neq h^*_{\mathcal{D}}\}\cap\Delta)}{m}\left(d\log\left(\frac{\Pr_{\mathcal{D}}(\Delta)}{\Pr_{\mathcal{D}}(\{h\neq h^*_{\mathcal{D}}\}\cap\Delta)}\right)+\log\left(\frac{1}{\delta}\right)\right)} , \label{eq:excess-r}
\end{align}
with probability at least $1-\delta$.

Following the same argument in the proof of Lemma~\ref{lem:excess-risk-under-tsybakov-noise}, we can show that $\mathcal{D}$ satisfies a Tsybakov noise condition with the same parameter $(a,\alpha)$ w.r.t. $\hpc_{\Delta}=\{h_{\Delta}:x \mapsto h(x)\mathbbm{1}(x\in\Delta), \forall h\in\hpc\}$. Based on the well-known fact that a Tsybakov noise condition implies a Bernstein class condition, we have for any $h\in\hpc$,
\begin{equation*}
    \Pr_{\mathcal{D}}(\{h \neq h^*_{\mathcal{D}}\}\cap \Delta) \leq a(\Pr_{\mathcal{D}}\left(\ER{h}\cap \Delta\right) - \Pr_{\mathcal{D}}\left(\ER{h^*_{\mathcal{D}}}\cap\Delta\right))^{\alpha} .
\end{equation*}
For notational simplicity, we denote $\mathcal{E}_{\mathcal{D}}^{\Delta}(h)=\Pr_{\mathcal{D}}(\ER{h}\cap\Delta) - \Pr_{\mathcal{D}}(\ER{h^*_{\mathcal{D}}}\cap\Delta)$ for any $h\in\hpc$. Substituting the Bernstein class condition into \Cref{eq:excess-r}, we have
\begin{equation*}
    \mathcal{E}_{\mathcal{D}}^{\Delta}(h) \leq C' \sqrt{\frac{(\mathcal{E}_{\mathcal{D}}^{\Delta}(h))^{\alpha}}{m}\left(d\log\left(\frac{\Pr_{\mathcal{D}}(\Delta)}{(\mathcal{E}_{\mathcal{D}}^{\Delta}(h))^{\alpha}}\right)+\log\left(\frac{1}{\delta}\right)\right)} ,
\end{equation*}
for constant $C' =4c_{0}\sqrt{C}$, where $C=a^{1-\alpha}/(\alpha^{\alpha}(1-\alpha)^{1-\alpha})$, with probability at least $1-\delta$. A simple calculation yields that
\begin{equation}
  \label{eq:excess-r-upper-bound}
    (\mathcal{E}_{\mathcal{D}}^{\Delta}(h))^{2-\alpha} \leq \frac{(C')^{2}}{m}\left(d\log\left(\frac{\Pr_{\mathcal{D}}(\Delta)}{(\mathcal{E}_{\mathcal{D}}^{\Delta}(h))^{\alpha}}\right)+\log\left(\frac{1}{\delta}\right)\right) .
\end{equation}
For the sake of contradiction, suppose that $\mathcal{E}_{\mathcal{D}}^{\Delta}(h)>c\gamma$, where $c\leq\max\{(C')^{2},2\}$. Then, we have 
\begin{equation}
  \label{eq:lhs}
    (\mathcal{E}_{\mathcal{D}}^{\Delta}(h))^{2-\alpha} > (c\gamma)^{2-\alpha} \geq \frac{(C')^{2}}{m}\left(d\log\left(\left(\frac{m}{d+\log\left(1/\delta_0\right)}\right)^{\frac{\alpha}{2-\alpha}}\Pr(\Delta)\right)+\log\left(\frac{1}{\delta}\right)\right) .
\end{equation}
On the other hand, because $\delta< \delta_0$, we have $\gamma\geq((d+\log(1/\delta_0))/m)^{1/(2-\alpha)}$, and thus $\mathcal{E}_{\mathcal{D}}^{\Delta}(h)>c\gamma >((d+\log(1/\delta_0)/m)^{1/(2-\alpha)}$. It follows that 
\begin{equation}
  \label{eq:rhs}
    \frac{(C')^2}{m}\left(d\log\left(\frac{\Pr_{\mathcal{D}}(\Delta)}{(\mathcal{E}_{\mathcal{D}}^{\Delta}(h))^{\alpha}}\right)+\log\left(\frac{1}{\delta}\right)\right) < \frac{(C')^2}{m}\left(d\log\left(\left(\frac{m}{d+\log\left(1/\delta_0\right)}\right)^{\frac{\alpha}{2-\alpha}}\Pr_{\mathcal{D}}(\Delta)\right)+\log\left(\frac{1}{\delta}\right)\right).
\end{equation}
Combine \Cref{eq:lhs} and \Cref{eq:rhs}, we have 
\begin{equation*}
    (\mathcal{E}_{\mathcal{D}}^{\Delta}(h))^{2-\alpha} > \frac{(C')^2}{m}\left(d\log\left(\frac{\Pr_{\mathcal{D}}(\Delta)}{(\mathcal{E}_{\mathcal{D}}^{\Delta}(h))^{\alpha}}\right)+\log\left(\frac{1}{\delta}\right)\right),
\end{equation*}
which contradicts \Cref{eq:excess-r-upper-bound}. Therefore, $\mathcal{E}_{\mathcal{D}}^{\Delta}(h)\leq c\gamma$ holds for both cases.

Then we prove the second statement. From the statement, we know the excess risk of $h$ is upper bounded by $\gamma$. In the meantime, as we know $\mathcal{D}$ satisfies the Tsybakov noise condition with parameter $(a,\alpha)$ w.r.t. $\hpc$ implies that $\mathcal{D}$ satisfies Bernstein class condition with parameter $(a,\alpha)$ w.r.t. $\hpc$, therefore, we have for any $h\in\hpc$,
\begin{equation*}
    \Pr_{\mathcal{D}}(\{h \neq h^*_{\mathcal{D}}\}\cap \Delta) \leq a(\Pr_{\mathcal{D}}\left(\ER{h}\cap \Delta\right) - \Pr_{\mathcal{D}}\left(\ER{h^*_{\mathcal{D}}}\cap\Delta\right))^{\alpha} .
\end{equation*}
Thus, for those $h\in\hpc$ whose excess risk is upper bounded by $\gamma$, we have $\Pr_{\mathcal{D}}(\{h \neq h^*_{\mathcal{D}}\}\cap \Delta) \leq a\gamma^\alpha$. 

By the uniform Bernstein inequality (\Cref{lem:uniform-bernstein-inequality}), for all $h \in \hpc$, with probability at least $1-\delta$, for some constant $c_0$,
\begin{align*}
    &\left|(\hat{\Pr}_{S_m}\left(\ER{h}\cap\Delta\right)-\hat{\Pr}_{S_m}\left(\ER{h^*_{\mathcal{D}}}\cap\Delta\right)) - \left(\Pr_{\mathcal{D}}\left(\ER{h}\cap \Delta\right) - \Pr_{\mathcal{D}}\left(\ER{h^*_{\mathcal{D}}}\cap\Delta\right)\right)\right|\\
    \leq &2 c_{0}\sqrt{\frac{\Pr_{\mathcal{D}}(\{h\neq h^*_{\mathcal{D}}\}\cap\Delta)}{m}\left(d\log\left(\frac{\Pr_{\mathcal{D}}(\Delta)}{\Pr_{\mathcal{D}}(\{h\neq h^*_{\mathcal{D}}\}\cap\Delta)} \land \frac{m\Pr_{\mathcal{D}}(\Delta)}{d}\right)+\log\left(\frac{1}{\delta}\right)\right)}.
\end{align*}
Taking $\Pr_{\mathcal{D}}(\{h \neq h^*_{\mathcal{D}}\}\cap \Delta) \leq a\gamma^\alpha$ and $\gamma = ((d\log((m/(d+\log(1/\delta_0)))^{\alpha/(2-\alpha)}\Pr_{\mathcal{D}}(\Delta))+\log(1/\delta))/m)^{1/(2-\alpha)}$ into the right-hand side of the inequality above, we have that 
\begin{equation*}
    \left|(\hat{\Pr}_{S_m}\left(\ER{h}\cap\Delta\right)-\hat{\Pr}_{S_m}\left(\ER{h^*_{\mathcal{D}}}\cap\Delta\right)) - \left(\Pr_{\mathcal{D}}\left(\ER{h}\cap \Delta\right) - \Pr_{\mathcal{D}}\left(\ER{h^*_{\mathcal{D}}}\cap\Delta\right)\right)\right| \leq 2c_0 \gamma.
\end{equation*}
From the statement, we know the excess risk of $h$ is upper bounded by $\gamma$, therefore, $|\Pr_{\mathcal{D}}(\ERin{h}\cap \Delta) - \Pr_{\mathcal{D}}(\ERin{h^*_{\mathcal{D}}}\cap\Delta)|\leq \gamma$. Then by the triangle inequality, $|(\hat{\Pr}_{S_m}(\ERin{h}\cap\Delta)-\hat{\Pr}_{S_m}(\ERin{h^*_{\mathcal{D}}}\cap\Delta))| \leq (2c_0+1) \gamma$.
Therefore, for any $h\in \hpc$ satisfies $\Pr_{\mathcal{D}}(\ERin{h}\cap\Delta) - \Pr_{\mathcal{D}}(\ERin{h^*_{\mathcal{D}}}\cap\Delta) \leq \gamma$, with probability at least $1-\delta$, $|\hat{\Pr}_{S_m}(\ER{h}\cap\Delta) - \hat{\Pr}_{S_m}(\ERin{h^*_{\mathcal{D}}}\cap\Delta)| \leq c\gamma$, because $c > 2c_0+1$. That finishes the proof.
\end{proof}

\begin{lemma}[\textbf{Restate of \Cref{lem:uniform-chernoff-pseudo-distance-estimation}}]
  \label{lem:uniform-chernoff-pseudo-distance-estimation-re}
Let $\hpc$ be a concept class with $\vcdim=d$. Let $\mathcal{D}$ be any distribution over $\mathcal{X}\times\{0,1\}$. Let $\delta_0 \in (0,1)$. Let $\Delta\subseteq\mathcal{X}$. Let $m = (8c_0^2/B)(d\log(1/B)+\log(1/\delta))\in\naturalnumber$. Let $S_{m}=\{(x_{i}, y_{i})\}_{i=1}^{m}\sim\mathcal{D}^{m}$ be i.i.d.\ samples. When $B\leq \min(\delta,1/(32c_0^2\log c_0))$, 
for any pair $f,g\in\hpc$, we have with probability at least $1-\delta$,
\begin{align*}
    &\hat{\Pr}_{S_{m}}(\{f\neq g\}\cap\Delta) \geq B \;\;\Longrightarrow\;\; \Pr_{\mathcal{D}}(\{f\neq g\}\cap\Delta) \geq B/2 , \\
    &\hat{\Pr}_{S_{m}}(\{f\neq g\}\cap\Delta) \leq B \;\;\Longrightarrow\;\; \Pr_{\mathcal{D}}(\{f\neq g\}\cap\Delta) \leq 2B .
\end{align*}
\end{lemma}

\begin{proof}[Proof of \Cref{lem:uniform-chernoff-pseudo-distance-estimation-re}]
For any pair $f,g\in\hpc$ and $\Delta\subseteq\mathcal{X}$, we have from Lemma~\ref{lem:pseudo-distance-estimation} that 
\begin{equation}
\label{eq:upper-bound-disagree}
    \left|\Pr_{\mathcal{D}}(\{f\neq g\}\cap\Delta) - \hat{\Pr}_{S_{m}}(\{f\neq g\}\cap\Delta)\right| \leq c_{0}\sqrt{\frac{\Pr_{\mathcal{D}}(\{f\neq g\}\cap\Delta)(d\log(m/d)+\log(1/\delta))}{m}} ,
\end{equation}
with probability at least $1-\delta$. Notice that $m = (8c_0^2/B)(d\log(1/B)+\log(1/\delta))$, we have 
\begin{equation*}
    \sqrt{\frac{(d\log(m/d)+\log(1/\delta))}{m}} = \sqrt{\frac{(d\log(m/d)+\log(1/\delta))}{(8c_0^2/B)(d\log(1/B)+\log(1/\delta))}}.
\end{equation*}
Notice that when $B\leq \min(\delta,1/(32c_0^2\log c_0))$, we have $m/d < 1/B^2$, therefore, we have
\begin{align*}
    \sqrt{\frac{(d\log(m/d)+\log(1/\delta))}{m}} \leq &\sqrt{\frac{(2d\log(1/B)+\log(1/\delta))}{(8/B)(d\log(1/B)+\log(1/\delta))}}\\
    \leq &\sqrt{\frac{(2d\log(1/B)+\log(1/\delta))}{(4/B)(2d\log(1/B)+\log(1/\delta))}} = \frac{\sqrt{B}}{2c_0}.
\end{align*}

Then we can prove the statement by contradiction. For the sake of contradiction, we assume that 
\begin{align}
    &\hat{\Pr}_{S_{m}}(\{f\neq g\}\cap\Delta) \geq B \;\;\text{ and }\;\; \Pr_{\mathcal{D}}(\{f\neq g\}\cap\Delta) \leq B/2 , \label{eq:statement-1}\\
    &\hat{\Pr}_{S_{m}}(\{f\neq g\}\cap\Delta) \leq B \;\;\text{ and }\;\; \Pr_{\mathcal{D}}(\{f\neq g\}\cap\Delta) \geq 2B .\label{eq:statement-2}
\end{align}
For the statement \ref{eq:statement-1}, we have that the left-hand side of the inequality \ref{eq:upper-bound-disagree} is greater than or equal to $B/2$. However, the right-hand side of the inequality \ref{eq:upper-bound-disagree} is smaller than or equal to $B/4$. Here is a contradiction.

For the statement \ref{eq:statement-2}, under this case we have the inequality \ref{eq:upper-bound-disagree} is equivalent to
\begin{equation*}
    \Pr_{\mathcal{D}}(\{f\neq g\}\cap\Delta) - \hat{\Pr}_{S_{m}}(\{f\neq g\}\cap\Delta) \leq \frac{\sqrt{B}}{2}\sqrt{\Pr_{\mathcal{D}}(\{f\neq g\}\cap\Delta)}.
\end{equation*}
Because $\Pr_{\mathcal{D}}(\{f\neq g\}\cap\Delta) > 0$, we have
\begin{equation}
\label{eq:upper-bound-disagree-2}
    \sqrt{\Pr_{\mathcal{D}}(\{f\neq g\}\cap\Delta)} - \frac{\hat{\Pr}_{S_{m}}(\{f\neq g\}\cap\Delta)}{\sqrt{\Pr_{\mathcal{D}}(\{f\neq g\}\cap\Delta)}} \leq \frac{\sqrt{B}}{2}.
\end{equation}
Then the left-hand side of the inequality \ref{eq:upper-bound-disagree-2} is greater than or equal to $\frac{\sqrt{2B}}{2}$, which is a contradiction. The proof is done.
\end{proof}
\begin{lemma}[\textbf{Restate of \Cref{lem:upper-bound-of-p-Delta}}]
  \label{lem:upper-bound-of-p-Delta-re}
Let $\gamma, B>0$. Let $\mathcal{X}$ be (some subset of) the instance space. Let $\Delta\subseteq\mathcal{X}$ be a region obtained via the procedure \ref{alg:population-subroutine}.
Assume $\mathcal{D}$ satisfies the Tsybakov noise condition with parameter $(a,\alpha)$ w.r.t. $\hpc$. We have
\begin{equation*}
    \Pr_{\mathcal{D}}(\Delta) \leq 2a'\left(\frac{2\gamma}{B}\right)^{\frac{\alpha}{1-\alpha}} .
\end{equation*}
\end{lemma}
\begin{proof}[Proof of Lemma~\ref{lem:upper-bound-of-p-Delta-re}]
Let $h^{*}_{\mathcal{D}}\in\hpc$ be the Bayes-optimal classifier w.r.t.\ the probability measure $\Pr_{\mathcal{D}}$. For any $h_{1},h_{2}$ and $\Delta\subseteq\mathcal{X}$, we have
\begin{align*}
    &\frac{1}{2}\left(\Pr_{\mathcal{D}}(\mathrm{ER}(h_{1})\cap\mathcal{X}\cap\Delta)+\Pr_{\mathcal{D}}(\mathrm{ER}(h_{2})\cap\mathcal{X}\cap\Delta)\right) - \Pr_{\mathcal{D}}(\mathrm{ER}(h^{*}_{\mathcal{D}})\cap\mathcal{X}\cap\Delta) \\
    = &\frac{1}{2}\Pr_{\mathcal{D}}\left(\{h_{1}\neq h_{2}\}\cap\mathcal{X}\cap\Delta\right) + \Pr_{\mathcal{D}}\left(\mathrm{ER}(h_{1})\cap\{h_{1}=h_{2}\}\cap\mathcal{X}\cap\Delta\right) - \Pr_{\mathcal{D}}(\mathrm{ER}(h^{*}_{\mathcal{D}})\cap\mathcal{X}\cap\Delta)\\
    = &\frac{1}{2}\Pr_{\mathcal{D}}\left(\{h_{1}\neq h_{2}\}\cap\mathcal{X}\cap\Delta\right) + \Pr_{\mathcal{D}}\left(\mathrm{ER}(h_{1})\cap\{h_{1}=h_{2}\}\cap\mathcal{X}\cap\Delta\right) \\
    &- \Pr_{\mathcal{D}}\left(\mathrm{ER}(h^{*}_{\mathcal{D}})\cap\{h_{1}=h_{2}\}\cap\mathcal{X}\cap\Delta\right) \\
    &- \Pr_{\mathcal{D}}\left(\mathrm{ER}(h^{*}_{\mathcal{D}})\cap\{h_{1}\neq h_{2}\}\cap\mathcal{X}\cap\Delta\right) \\
    \geq &\frac{1}{2}\Pr_{\mathcal{D}}\left(\{h_{1}\neq h_{2}\}\cap\mathcal{X}\cap\Delta\right) - \Pr_{\mathcal{D}}\left(\mathrm{ER}(h^{*}_{\mathcal{D}})\cap\{h_{1}\neq h_{2}\}\cap\mathcal{X}\cap\Delta\right) .
\end{align*}
Now, assume that $h_{1},h_{2}\in V$, it follows that
\begin{equation}
  \label{eq:bound-p-delta-eq1}
    \Pr_{\mathcal{D}}\left(h^{*}_{\mathcal{D}}(x)\neq y \;\big|\; \{h_{1}\neq h_{2}\}\cap\mathcal{X}\cap\Delta\right) \geq \frac{1}{2}-\frac{\gamma}{\Pr_{\mathcal{D}}\left(\{h_{1}\neq h_{2}\}\cap\mathcal{X}\cap\Delta\right)} .
\end{equation}
Assume within the while loop, we find (in order) $k$ pairs of concepts $\{(h_{1}^{(j)},h_{2}^{(j)}), 1\leq j\leq k\}$ for some positive integer $k$. Denote the status of the region $\Delta$ after merging $\{h_{1}^{(j)} \neq h_{2}^{(j)}\}$ by $\Delta^{(j)}$ for $1\leq j\leq k$ and denote $\Delta^{(0)}=\emptyset$. In other words, for each iteration of the while loop, we can write
\begin{equation*}
    \Delta^{(0)}=\emptyset ,\;\; \Delta^{(1)}=\Delta^{(0)}\cup\{h_{1}^{(1)} \neq h_{2}^{(1)}\} ,\;\; \ldots ,\;\; \Delta^{(k)}=\Delta^{(k-1)}\cup\{h_{1}^{(k)} \neq h_{2}^{(k)}\} .
\end{equation*}
Note that for each $1\leq j\leq k$, $h_{1}^{(j)},h_{2}^{(j)}\in V$ satisfy $\Pr_{\mathcal{D}}(\{h_{1}^{(j)}\neq h_{2}^{(j)}\}\cap\mathcal{X}\setminus\Delta^{(j-1)})>B$. Plugging $(h_{1},h_{2})=(h_{1}^{(j)},h_{2}^{(j)})$ and $\Delta=\{h_{1}^{(j)}\neq h_{2}^{(j)}\}\setminus\Delta^{(j-1)}$ into \eqref{eq:bound-p-delta-eq1}, we get
\begin{equation}
  \label{eq:bound-p-delta-eq2}
    \Pr_{\mathcal{D}}\left(h^{*}_{\mathcal{D}}(x)\neq y \;\big|\; \{h_{1}^{(j)}\neq h_{2}^{(j)}\}\cap\mathcal{X}\setminus\Delta^{(j-1)}\right) \geq \frac{1}{2}-\frac{\gamma}{\Pr_{\mathcal{D}}\left(\{h_{1}^{(j)}\neq h_{2}^{(j)}\}\cap\mathcal{X}\setminus\Delta^{(j-1)}\right)} > \frac{1}{2}-\frac{\gamma}{B} .
\end{equation}
Moreover, we know that eventually,
\begin{align}
  \label{eq:bound-p-delta-eq3}
    \Delta = &\Delta^{(k)} \nonumber \\
    = &\Delta^{(k-1)}\cup\{h_{1}^{(k)} \neq h_{2}^{(k)}\} \nonumber \\
    = &\left(\{h_{1}^{(k)}\neq h_{2}^{(k)}\}\setminus\Delta^{(k-1)}\right)\cup\Delta^{(k-1)} \nonumber \\
    = &\cdots \nonumber \\
    = &\bigcup_{j=1}^{k} \left(\{h_{1}^{(j)}\neq h_{2}^{(j)}\}\setminus\Delta^{(j-1)}\right)\cup\Delta^{(0)} \nonumber \\
    = &\bigcup_{j=1}^{k} \left(\{h_{1}^{(j)}\neq h_{2}^{(j)}\}\setminus\Delta^{(j-1)}\right) .
\end{align}
By the law of total probability, we have 
\begin{align*}
    &\Pr_{\mathcal{D}}\left(h^{*}_{\mathcal{D}}(x)\neq y \;\big|\; \mathcal{X}\cap\Delta\right) \\
    \stackrel{\text{\eqmakebox[lemma-upper-bound-of-p-Delta-re][c]{\eqref{eq:bound-p-delta-eq3}}}}{=} &\sum_{j=1}^{k}\Pr_{\mathcal{D}}\left(h^{*}_{\mathcal{D}}(x)\neq y \;\big|\; \{h_{1}^{(j)}\neq h_{2}^{(j)}\}\cap\mathcal{X}\setminus\Delta^{(j-1)}\right)\Pr_{\mathcal{D}}\left(\{h_{1}^{(j)}\neq h_{2}^{(j)}\}\cap\mathcal{X}\setminus\Delta^{(j-1)} \;\big|\; \mathcal{X}\cap\Delta\right)\\
    \stackrel{\text{\eqmakebox[lemma-upper-bound-of-p-Delta-re][c]{\eqref{eq:bound-p-delta-eq2}}}}{\geq} &\left(\frac{1}{2}-\frac{\gamma}{B}\right)\sum_{j=1}^{k}\Pr_{\mathcal{D}}\left(\{h_{1}^{(j)}\neq h_{2}^{(j)}\}\cap\mathcal{X}\setminus\Delta^{(j-1)} \;\big|\; \mathcal{X}\cap\Delta\right) \\
    \stackrel{\text{\eqmakebox[lemma-upper-bound-of-p-Delta-re][c]{}}}{=} &\frac{1}{2}-\frac{\gamma}{B} .
\end{align*}
On the other hand, we have for any $t\in(\gamma/B, 1/2]$,
\begin{align*}
    &\Pr_{\mathcal{D}}\left(h^{*}_{\mathcal{D}}(x)\neq y \;\big|\; \mathcal{X}\cap\Delta\right) \\
    = &\E\left[\left(\frac{1}{2}-\Big|\eta(x)-\frac{1}{2}\Big|\right) \;\bigg|\; \mathcal{X}\cap\Delta\right] \\
    \leq &\left(\frac{1}{2}-t\right)\Pr_{\mathcal{D}}\left(\Big|\eta(x)-\frac{1}{2}\Big|>t \;\bigg|\; \mathcal{X}\cap\Delta\right) + \frac{1}{2}\Pr_{\mathcal{D}}\left(\Big|\eta(x)-\frac{1}{2}\Big|\leq t \;\bigg|\; \mathcal{X}\cap\Delta\right) \\
    \leq &\frac{1}{2} - t + t\Pr_{\mathcal{D}}\left(\Big|\eta(x)-\frac{1}{2}\Big|\leq t \;\bigg|\; \mathcal{X}\cap\Delta\right) \\
    \leq &\frac{1}{2} - t + t\frac{\Pr_{\mathcal{D}}\left(\Big|\eta(x)-\frac{1}{2}\Big|\leq t\right)}{\Pr_{\mathcal{D}}(\mathcal{X}\cap\Delta)} \\
    \leq &\frac{1}{2} - t + \frac{at^{1+\frac{\alpha}{1-\alpha}}}{\Pr_{\mathcal{D}}(\mathcal{X}\cap\Delta)} .
\end{align*}
Putting together, choosing $t=2\gamma/B<1/2$, we have
\begin{equation*}
    \Pr_{\mathcal{D}}(\Delta) = \Pr_{\mathcal{D}}(\mathcal{X}\cap\Delta) \leq \frac{at^{\frac{\alpha}{1-\alpha}}}{1-\gamma/tB} \leq 2a'\left(\frac{2\gamma}{B}\right)^{\frac{\alpha}{1-\alpha}} .
\end{equation*}
\end{proof}

\begin{lemma}[\textbf{Restate of \Cref{lem:approximate-linear-BCC}}]
  \label{lem:approximate-linear-BCC-re}
Suppose $\mathcal{E}_{t-1}$ happens, $\mathcal{E}_{t}$ will happen with probability at least $1-\delta_{t}$, where $\delta_t = (\delta/5)^{T+1-t}$.
\end{lemma}
\begin{proof}[Proof of Lemma~\ref{lem:approximate-linear-BCC-re}]
    As $\mathcal{E}_{t-1}$ happens, we have 
    \begin{equation}
    \label{eq:bound-of-p-Delta-t}
        \Pr(\Delta_{t-1}) \leq c_2 a_{t-1}^{-\frac{\alpha}{1-\alpha}}.
    \end{equation}
    Then we can prove the event $\mathcal{E}_t$ happens with probability at least $1-\delta_t$, by definition, we know that for all $t = 1,\ldots,T$, $\delta_t/2 < \delta$ and $\log(2/\delta_t) = O(\log(1/\delta))$.
    
    From \Cref{lem:prune}, we know that, for sufficiently large $n$, with probability at least $(1- \delta_t/2)$, for all $h \in \hpc_{t-1}$
    \begin{align*}
        &\Pr_{\mathcal{D}}\left(\ER{h}\cap\Delta_{t-1}\right) - \Pr_{\mathcal{D}}\left(\ER{h^*_{\mathcal{D}}}\cap\Delta_{t-1}\right) \\
        \stackrel{\text{\eqmakebox[lemma-inductive][c]{}}}{\leq} &c C_1\left(\frac{d}{n_t}\log\left(\left(\frac{n_t}{d+\log\left(\frac{1}{\delta}\right)}\right)^{\frac{\alpha}{2-\alpha}}\Pr_{\mathcal{D}}\left(\Delta_{t-1}\right)\right)+\frac{1}{n_t}\log\left(\frac{2}{\delta_t}\right)\right)^{\frac{1}{2-\alpha}}\\
        \stackrel{\text{\eqmakebox[lemma-inductive][c]{\text{\tiny \Cref{eq:bound-of-p-Delta-t}}}}}{\leq} &cC_1\left(\frac{d}{n_t}\log\left(\left(\frac{n_t}{d+\log\left(\frac{1}{\delta}\right)}\right)^{\frac{\alpha}{2-\alpha}}c_2 a_{t-1}^{-\frac{\alpha}{1-\alpha}}\right)+\frac{1}{n_t}\log\left(\frac{2}{\delta_t}\right)\right)^{\frac{1}{2-\alpha}}\\
        \stackrel{\text{\eqmakebox[lemma-inductive][c]{$c \geq 2$}}}{\leq} &2c C_1\left(\frac{d}{n_t}\log\left(\left(\frac{n_t}{d+\log\left(\frac{1}{\delta}\right)}\right)^{\frac{\alpha}{2-\alpha}} a_{t-1}^{-\frac{\alpha}{1-\alpha}}\right)+\frac{1}{n_t}\log\left(\frac{2}{\delta_t}\right)\right)^{\frac{1}{2-\alpha}}\\
        \stackrel{\text{\eqmakebox[lemma-inductive][c]{\tiny Definition of $\gamma_t$}}}{=} &2c \gamma_t ,
    \end{align*}
    which implies the first bullet of the event $\mathcal{E}_t$.
    
    Notice that by taking $n$ large enough, $a_t\gamma_t$ satisfies the requirement in \Cref{lem:uniform-chernoff-pseudo-distance-estimation}, then, by taking $\delta = \delta_t/2$, $m_t = (8c_0^2/B_t)(d\log(1/B_t)+\log(2/\delta_t))$, from \Cref{lem:uniform-chernoff-pseudo-distance-estimation}, we know that for every pair of $f,g\in \hpc$ and every $\Delta \subseteq \cX$, with probability at least $1-\delta_t/2$, we have 
    \begin{align*}
        &\hat{\Pr}_{S_{2,t}}(\{f\neq g\}\cap\Delta) \geq B_{t} \;\;\Longrightarrow\;\; \Pr_{\mathcal{D}}(\{f\neq g\}\cap\Delta) \geq B_{t}/2 , \\
        &\hat{\Pr}_{S_{2,t}}(\{f\neq g\}\cap\Delta) \leq B_{t} \;\;\Longrightarrow\;\; \Pr_{\mathcal{D}}(\{f\neq g\}\cap\Delta) \leq 2B_{t} ,
    \end{align*}
    where 
    \begin{align*}
        &a_{t} = 2^{t-T}\left(\frac{n_t}{d+\log\left(1/\delta\right)}\right)^{\frac{1-\alpha}{2-\alpha}} \;\;,\;\; \gamma_{t} = C_1\left(\frac{d}{n_t}\log\left(\left(\frac{n_t}{d+\log(1/\delta)}\right)^{\frac{\alpha}{2-\alpha}}(a_{t-1})^{-\frac{\alpha}{1-\alpha}}\right)+\frac{1}{n_t}\log\left(\frac{2}{\delta_t}\right)\right)^{\frac{1}{2-\alpha}}\; ,\\
        &\mathrm{and}\;\; B_{t}=a_{t}\gamma_{t}.
    \end{align*}
    Then according to \Cref{alg:main-algo}, when the \textbf{while} loop terminates, we have that for every pair of $f,g \in \hpc_{t-1}$, 
    \begin{equation*}
        \hat{\Pr}_{S_{2,t}}\left(\{f\neq g\}\cap\left(\Delta_{t-1}\setminus\Delta_{t}\right)\right) \leq a_{t}\gamma_{t} = B_t.
    \end{equation*}
    Thus, we have that, with probability at least $1-\delta_t/2$, for every pair of $f,g\in \hpc_{t-1}$, 
    \begin{equation*}
        \Pr_{\mathcal{D}}(\{f\neq g\}\cap (\Delta_{t-1}\setminus\Delta_t)) \leq 2B_t = 2 a_t\gamma_t.
    \end{equation*}
    That shows the second bullet of the event $\mathcal{E}_t$. 
    
    Then we need to prove that $\Pr_{\mathcal{D}}(\Delta_t) \leq c_2(a_t)^{-\alpha/(1-\alpha)}$. First, let $\Delta'_t$ be the region generated by Subroutine~\ref{alg:population-subroutine} with $\gamma = 2c\gamma_t$ and $B = B_t/2$. Then from \Cref{lem:upper-bound-of-p-Delta}, we know that 
    \begin{equation*}
        \Pr_{\mathcal{D}}(\Delta'_t) \leq 2a'\left(\frac{4c\gamma_t}{B_t/2}\right)^{\frac{\alpha}{1-\alpha}} = 2a'\left(\frac{8c\gamma_t}{a_t\gamma_t}\right)^{\frac{\alpha}{1-\alpha}} = 2a'(8c)^{\frac{\alpha}{1-\alpha}}a_t^{-\frac{\alpha}{1-\alpha}} \leq c_2(a_t)^{-\frac{\alpha}{1-\alpha}}, 
    \end{equation*}
    because $c_2 \geq 2a'(8c)^{\alpha/(1-\alpha)}$. 
    From union bound and the previous part of this proof, we know that, with probability at least $1-\delta_t$, we have the following facts hold simultaneously:
    \begin{itemize}
        \item For all $h \in \hpc_{t-1}$, $\Pr_{\mathcal{D}}\left(\ER{h}\cap\Delta_{t-1}\right) - \Pr_{\mathcal{D}}\left(\ER{h^*_{\mathcal{D}}}\cap\Delta_{t-1}\right) \leq 2c \gamma_t$.
        \item For every pair $f,g \in \hpc_{t-1}$ and every $\Delta\subseteq \Delta_{t-1}$, if $\hat{\Pr}_{S_{m}}(\{f\neq g\}\cap\Delta) \geq B_{t}$, then $\Pr_{\mathcal{D}}(\{f\neq g\}\cap\Delta) \geq B_{t}/2$.
    \end{itemize}
    Thus, $\Delta_t \subseteq \Delta'_t$ with probability at least $1-\delta_t$ and thus, with probability at least $1-\delta_t$, we have $\Pr_{\mathcal{D}}(\Delta_t) \leq \Pr_{\mathcal{D}}(\Delta'_t) \leq c_2(a_t)^{-\alpha/(1-\alpha)}$. That shows the third bullet of the event $\mathcal{E}_t$. 
    
    At last, because every time we find a pair of $f,g$, the region $\Delta_t$ will increase at least $B_t/2$ and the probability of the region $\Delta_t$ is upper-bounded by $c_2(a_t)^{-\alpha/(1-\alpha)}$. We have $s_t \leq 2\Pr_{\mathcal{D}}(\Delta_t)/B_t \leq 2c_2 a_t^{-1/(1-\alpha)} \gamma_t^{-1}$, which yields the desired inequality. That shows the fourth bullet of the event $\mathcal{E}_t$.
    
    Therefore, under the assumption that $\mathcal{E}_{t-1}$ happens, with probability at least $1-\delta_t$, $\mathcal{E}_t$ happens.
\end{proof}
\begin{lemma}[\textbf{Restate of \Cref{lem:good-events-hold}}]
\label{lem:good-events-hold-re}
    With probability at least $1-2\delta/5$, for all $t$, $\mathcal{E}_t$ happens.
\end{lemma}
\begin{proof}[Proof of \Cref{lem:good-events-hold-re}]
    We prove this lemma by induction. Start from the base case, where $t = 1$, and we need to prove the following three facts happened with probability at least $1- \delta_1$:
    \begin{itemize}
        \item $\Pr_{\mathcal{D}}(\mathrm{ER}(h)) - \Pr_{\mathcal{D}}(\mathrm{ER}(h^{*}_{\mathcal{D}}))\leq 2c\gamma_1, \;\; \forall h\in\hpc_1$.
        \item $\Pr_{\mathcal{D}}(\{f \neq g\}\cap(\cX\setminus\Delta_{1})) \leq 2a_1\gamma_1, \;\; \forall f,g\in\hpc_1$.
        \item $\Pr(\Delta_1) \leq 2a'(8c)^{\alpha/(1-\alpha)}(a_1)^{-\frac{\alpha}{1-\alpha}}$.
        \item $s_1 \leq 2c_2 a_1^{-\frac{1}{1-\alpha}} \gamma_1^{-1}$, where $s_t$ is the number of pairs $f$ and $g$ found on the termination of \textbf{while} loop.
    \end{itemize}
    Here, $a_1 = (n_1/(d+\log(1/\delta)))^{(1-\alpha)/(2-\alpha)}(\log(n/d))^{-1}$ and $\gamma_1= C_1((d\log(n_1/(d+\log(1/\delta)))^{\alpha/(2-\alpha)}+\log(2/\delta_{1}))/n_1)^{1/(2-\alpha)}$. From \Cref{lem:prune}, by taking $\Delta = \cX$ and $\delta_0 = \delta$, we know the first bullet holds with probability at least $1-\delta_1/2$. 
    
    By taking $n$ large enough, we have $a_1\gamma_1$ satisfies the requirement, then from \Cref{lem:uniform-chernoff-pseudo-distance-estimation} and $m_1 = 8c_0^2(d\log(1/(a_1\gamma_1))+\log(2/\delta_1))/(a_1\gamma_1)$, with probability at least $1-\delta_1/2$, for all pair of hypothesis $f,g \in \hpc_1$, the following statement holds:
    \begin{align*}
        &\hat{\Pr}_{S_1}(\{f\neq g\}\cap\Delta_1) \geq B_{1} \;\;\Longrightarrow\;\; \Pr_{\mathcal{D}}(\{f\neq g\}\cap\Delta_1) \geq B_{1}/2 , \\
        &\hat{\Pr}_{S_1}(\{f\neq g\}\cap(\cX\setminus\Delta_1)) \leq B_{1} \;\;\Longrightarrow\;\; \Pr_{\mathcal{D}}(\{f\neq g\}\cap(\cX\setminus\Delta_1)) \leq 2B_{1} ,
    \end{align*}
    where $B_1 = a_1\gamma_1 = \tilde{\Theta}(((d+\log(1/\delta))/n)^{\alpha/(2-\alpha)})$.

    Combining with \Cref{lem:upper-bound-of-p-Delta}, and for the same reason as the proof of \Cref{lem:approximate-linear-BCC}, we have $\Pr(\Delta_1) \leq 2a'(8c)^{\alpha/(1-\alpha)}(a_1)^{-\alpha/(1-\alpha)}$ and $s_1 \leq 2c_2 a_1^{-1/(1-\alpha)} \gamma_1^{-1}$. These facts yield the event $\mathcal{E}_1$ happens with probability at least $1-\delta_1$.
    
    Then from \Cref{lem:approximate-linear-BCC}, we know that if $\mathcal{E}_i$ happens for all $i < t$, $\mathcal{E}_t$ will happen with probability at least $1-\delta_t$. Therefore, by the chain rule, we have
    \begin{equation}
        \Pr_{\mathcal{D}}\left(\bigcap_{t = 1}^{T}\mathcal{E}_t\right) = \prod_{t = 1}^T \Pr_{\mathcal{D}}\left(\mathcal{E}_t\left|\bigcap_{i = 1}^{t-1}\mathcal{E}_i\right.\right) \geq \prod_{t = 1}^T (1-\delta_t) \geq 1-\sum_{t = 1}^T \delta_t.
    \end{equation}
    The last inequality shares the same reason as the union bound. Notice that $\delta_t = (\delta/5)^{T-t+1}$, for $t = 1,2,\ldots,T$, thus $\sum_{t = 1}^T \delta_t \leq 2\delta/5$, that yields $\Pr_{\mathcal{D}}(\bigcap_{t = 1}^{T}\mathcal{E}_t) \geq 1-2\delta/5$ as desired. 
\end{proof}

\section{Concentration Inequalities}
  \label{sec:concentration-inequalities}

\begin{definition}  [\textbf{Disagreement coefficient}]
  \label{def:disagreement-coefficient}
Let $\mathcal{X}$ be an instance space and $\hpc$ be a concept class. We define the \underline{region of disagreement} of $\hpc$ as $\mathrm{DIS}(\hpc)=\{x\in\mathcal{X}:\exists h,g\in\hpc \text{ s.t. } h(x)\neq g(x)\}$. Let $h^{*}_{\mathcal{D}}\in\hpc$ be the Bayes-optimal classifier w.r.t.\ distribution $\mathcal{D}$. For any $r>0$, define $B(h^{*}_{\mathcal{D}},r)=\{h\in\hpc: \Pr_{\mathcal{D}}(x:h(x) \neq h^{*}_{\mathcal{D}}(x)) \leq r\}$. Then, the \underline{disagreement coefficient w.r.t.\ $\hpc$ under $\mathcal{D}$} is defined, for any $r_{0}\geq0$, as
\begin{equation*}
    \theta_{\hpc}(r_{0}) = \sup_{r>r_{0}}\frac{\Pr_{\mathcal{D}}(\mathrm{DIS}(B(h^{*}_{\mathcal{D}},r)))}{r} \lor 1 .
\end{equation*}
\end{definition}

\begin{lemma}  [\textbf{ERM Bound under Massart Noise Condition}]
  \label{lem:excess-risk-under-linear-BCC}
Let $\Delta\subseteq\mathcal{X}$ be some subset of the instance space. Let $h^{*}_{\mathcal{D}}\in\hpc$ be the Bayes-optimal classifier w.r.t.\ $\mathcal{D}$. Let $d=\vcdim$ and $a>0$. Assume that
\begin{equation*}
    \Pr_{\mathcal{D}}(\{h \neq h^{*}_{\mathcal{D}}\}\cap\Delta) \leq a\left(\Pr_{\mathcal{D}}(\mathrm{ER}(h)\cap\Delta)-\Pr_{\mathcal{D}}(\mathrm{ER}(h^{*}_{\mathcal{D}})\cap\Delta)\right), \;\; \forall h\in\hpc .
\end{equation*}
Let $m\in\naturalnumber$. Let $\mathrm{ERM}(\hpc;S_{m})=\{h\in\hpc: h=\argmin_{h\in\hpc}\hat{\Pr}_{S_{m}}(\mathrm{ER}(h))\}$ denote the class of empirical risk minimizers with $m$ i.i.d.\ samples $S_{m}=\{(x_{i}, y_{i})\}_{i=1}^{m}\sim\mathcal{D}^{m}$. Then, for any $\delta\in(0,1)$, we have
\begin{equation*}
    \sup_{h\in\mathrm{ERM}(\hpc;S_{m})}\Pr_{\mathcal{D}}\left(\mathrm{ER}(h)\cap\Delta\right) - \Pr_{\mathcal{D}}\left(\mathrm{ER}(h^{*}_{\mathcal{D}})\cap\Delta\right) \leq \frac{C_{0}a}{m}\left(d\log\left(\frac{m\Pr_{\mathcal{D}}(\Delta)}{a^{2}d}\right)+\log\left(\frac{1}{\delta}\right)\right) ,
\end{equation*}
for some universal constant $C_{0}>0$, with probability at least $1-\delta$.
\end{lemma}

\begin{proof}[Proof of Lemma~\ref{lem:excess-risk-under-linear-BCC}]
We recall the following upper bound on the excess risk of (the worst-case) ERM for learning under any distribution $\mathcal{D}$ satisfies a Bernstein class condition with parameters $(a,\alpha)$ w.r.t.\ $\hpc$ \citep[c.f.][]{massart2006risk,gine2006concentration,hanneke2019surrogate,hanneke2016refined}. We have, with probability at least $1-\delta$,
\begin{equation}
  \label{eq:erm-guanrantee-BCC-with-disagreement-coefficient}
    \sup_{h\in\mathrm{ERM}(\hpc;S_{m})}\Pr_{\mathcal{D}}\left(\mathrm{ER}(h)\right) - \inf_{h\in\hpc}\Pr_{\mathcal{D}}\left(\mathrm{ER}(h)\right) \lesssim \left(\frac{a}{m}\left(d\log\left(\theta_{\hpc}\left(a\left(\frac{ad}{m}\right)^{\frac{\alpha}{2-\alpha}}\right)\right)+\log\left(\frac{1}{\delta}\right)\right)\right)^{\frac{1}{2-\alpha}} ,
\end{equation}
where $\theta_{\hpc}(\cdot)$ is the disagreement coefficient w.r.t.\ $\hpc$ under $\mathcal{D}$ as defined in Definition~\ref{def:disagreement-coefficient}. Clearly, when $\alpha=1$ and $h^{*}_{\mathcal{D}}\in\hpc$ being the Bayes-optimal classifier w.r.t.\ $\mathcal{D}$, this yields that, with probability at least $1-\delta$,
\begin{equation*}
    \sup_{h\in\mathrm{ERM}(\hpc;S_{m})}\Pr_{\mathcal{D}}\left(\mathrm{ER}(h)\right) - \Pr_{\mathcal{D}}\left(\mathrm{ER}(h^{*}_{\mathcal{D}})\right) \leq \frac{C_{0}a}{m}\left(d\log\left(\theta_{\hpc}\left(\frac{a^{2}d}{m}\right)\right)+\log\left(\frac{1}{\delta}\right)\right) ,
\end{equation*}
for some universal constant $C_{0}>0$. For any measurable set $\Delta\subseteq\mathcal{X}$, we define $\hpc_{\Delta}=\{h_{\Delta}:x \mapsto h(x)\mathbbm{1}(x\in\Delta), \forall h\in\hpc\}$. Now, we consider learning $\hpc_{\Delta}$ under distribution $\mathcal{D}$ with ERM. Note that since $h^{*}_{\mathcal{D}}\in\hpc$ is Bayes-optimal, $h^{*}_{\mathcal{D},\Delta}=h^{*}_{\mathcal{D}}\cdot\mathbbm{1}(\Delta)\in\hpc_{\Delta}$ is also Bayes-optimal w.r.t.\ $\mathcal{D}$. Moreover, per the lemma statement, we know that $\mathcal{D}$ satisfies a Bernstein class condition with parameters $(a,1)$ w.r.t.\ $\hpc_{\Delta}$. Therefore, we can apply the above bound and obtain that
\begin{align*}
    &\sup_{h\in\mathrm{ERM}(\hpc;S_{m})}\Pr_{\mathcal{D}}\left(\mathrm{ER}(h)\cap\Delta\right) - \Pr_{\mathcal{D}}\left(\mathrm{ER}(h^{*}_{\mathcal{D}})\cap\Delta\right) \\
    = &\sup_{h_{\Delta}\in\mathrm{ERM}(\hpc_{\Delta};S_{m})}\Pr_{\mathcal{D}}\left(\mathrm{ER}(h_{\Delta})\right) - \Pr_{\mathcal{D}}\left(\mathrm{ER}(h^{*}_{\mathcal{D},\Delta})\right) \\
    \leq &\frac{C_{0}a}{m}\left(\mathrm{VC}(\hpc_{\Delta})\log\left(\theta_{\hpc_{\Delta}}\left(\frac{a^{2}\mathrm{VC}(\hpc_{\Delta})}{m}\right)\right)+\log\left(\frac{1}{\delta}\right)\right) ,
\end{align*}
with probability at least $1-\delta$. Note that the disagreement coefficient shown above is dependent on $h^{*}_{\mathcal{D},\Delta}$ and $\mathcal{D}$. It is straightforward that $\Pr_{\mathcal{D}}(\mathrm{DIS}(B(h^{*}_{\mathcal{D},\Delta},r))) \leq \Pr_{\mathcal{D}}(\Delta)$. This immediately implies
\begin{equation*}
    \theta_{\hpc_{\Delta}}\left(\frac{a^{2}\mathrm{VC}(\hpc_{\Delta})}{m}\right) = \sup_{r>a^{2}\mathrm{VC}(\hpc_{\Delta})/m}\frac{\Pr_{\mathcal{D}}\left(\mathrm{DIS}\left(B(h^{*}_{\mathcal{D},\Delta},r)\right)\right)}{r} \lor 1 \leq \frac{m\Pr_{\mathcal{D}}(\Delta)}{a^{2}\mathrm{VC}(\hpc_{\Delta})} .
\end{equation*}
Finally, the lemma follows from the fact that $\mathrm{VC}(\hpc_{\Delta})\leq d$.
\end{proof}

\begin{lemma}  [\textbf{Uniform Bernstein Inequality}]
  \label{lem:uniform-bernstein-inequality}
Let $\hpc$ be any measurable class of functions mapping from $\mathcal{X}$ to $\{0,1\}$ with $\vcdim=d<\infty$. Let $\mathcal{D}$ be any probability measure on $\mathcal{X}\times\{0,1\}$. Let $\Delta\subseteq\mathcal{X}$ be some subset of the instance space. For any $m\in\naturalnumber$ and $\delta\in(0,1)$, let $S_{m}=\{(X_{i},Y_{i})\}_{i=1}^{m}\sim\mathcal{D}^{m}$. Then, with probability at least $1-\delta$, every pair of $f,g \in\hpc$ satisfies
\begin{align*}
    &\Big|\big(\Pr_{\mathcal{D}}(\mathrm{ER}(f)\cap\Delta)-\Pr_{\mathcal{D}}(\mathrm{ER}(g)\cap\Delta)\big) - \big(\hat{\Pr}_{S_{m}}(\mathrm{ER}(f)\cap\Delta)-\hat{\Pr}_{S_{m}}(\mathrm{ER}(g)\cap\Delta)\big)\Big| \\ 
    \leq &c_{0}\sqrt{\frac{\Pr_{\mathcal{D}}(\{f\neq g\}\cap\Delta)}{m}\left(d\log\left(\frac{\Pr_{\mathcal{D}}(\Delta)}{\Pr_{\mathcal{D}}(\{f\neq g\}\cap\Delta)} \land \frac{m\Pr_{\mathcal{D}}(\Delta)}{d}\right)+\log\left(\frac{1}{\delta}\right)\right)} \\
    &+ \frac{c_{0}}{m}\left(d\log\left(\frac{\Pr_{\mathcal{D}}(\Delta)}{\Pr_{\mathcal{D}}(\{f\neq g\}\cap\Delta)} \land \frac{m\Pr_{\mathcal{D}}(\Delta)}{d}\right)+\log\left(\frac{1}{\delta}\right)\right) ,
\end{align*}
for some universal constant $c_{0}>0$.
\end{lemma}

\begin{proof}[Proof of Lemma~\ref{lem:uniform-bernstein-inequality}]
This result refines the classical Bernstein inequality \citep{bernstein1924modification} or the Bennett inequality \citep{bennett1962probability}. Results of this type have been established in the literature for standard i.i.d.\ data \citep{koltchinskii2006local,gine2006concentration,hanneke2025universal,hanneke2025agnostic} as well as for non-i.i.d.\ data \citep{hanneke2022no}. The proof techniques usually involve first deriving expectation bounds based on chaining an envelope function \citep{van-der-Vaart96,van2011local}, and then using the Bousquet’s inequality \citep{bousquet2002bennett} to achieve high probability bounds together with localization \citep{koltchinskii2006local}. The presenting bound here differs from the existing ones by restricting to a sub-region $\Delta\subseteq\mathcal{X}$. For completeness, we include this main technical modification. 

Let $\mathcal{F}$ be a measurable class of measurable functions with a (measurable) envelope function $F$, i.e.\ $|f|\leq F$ for every $f\in\mathcal{F}$, such that $F\leq1$ and $\mathcal{F}^{2}$ is measurable. Let $X_{1},\ldots,X_{m}$ be i.i.d.\ samples drawn from $\mathcal{D}$. \cite{van2011local} proves that, if $\mathcal{D} f^{2}<\eta^{2}\mathcal{D} F^{2}$ for every $f$ and some $\eta\in(0,1)$, then 
\begin{equation*}
    \E_{\mathcal{D}}\left[\sup_{f\in\mathcal{F}}\bigg|\frac{1}{\sqrt{m}}\sum_{i=1}^{m}\left(f(X_{i})-\mathcal{D} f\right)\bigg|\right] \lesssim J\left(\eta,\mathcal{F},L_{2}\right)\left(1+\frac{J\left(\eta,\mathcal{F},L_{2}\right)}{\eta^{2}\sqrt{m}\lVert F\rVert_{\mathcal{D},2}}\right)\lVert F\rVert_{\mathcal{D},2} ,
\end{equation*}
where $\lVert f\rVert_{\mathcal{D},2}$ denotes the norm of a function $f$ in $L_{2}(\mathcal{D})$,
\begin{equation*}
    J\left(\eta,\mathcal{F},L_{2}\right) = \sup_{\mathcal{D}}\int_{0}^{\eta}\sqrt{1+\log N(\epsilon\lVert F\rVert_{\mathcal{D},2},\mathcal{F},L_{2}(\mathcal{D}))} \mathrm{d}\epsilon 
\end{equation*}
is called the uniform entropy integral (at $\eta$), and $N(\epsilon,\mathcal{F},L_{2}(\mathcal{D}))$ is called the covering number, i.e.\ the minimal number of balls of radius $\epsilon$ in $L_{2}(\mathcal{D})$ needed to cover $\mathcal{F}$. A direct implication of this result is the following. For any $f,g: \mathcal{X}\rightarrow\{0,1\}$, let $\ell_{f,g|_{\Delta}}(x,y)=\mathbbm{1}(\{f(x)\neq y\}\cap\Delta)-\mathbbm{1}(\{g(x)\neq y\}\cap\Delta)$. Now, we define $\mathcal{F}=\{\ell_{f,g|_{\Delta}}: f,g\in\hpc\}$. Note that $\mathcal{F}$ is a measurable class of functions from $\mathcal{X}\times\{0,1\}$ to $\{0,\pm1\}$ and for any $f,g\in\hpc$, the following hold
\begin{equation*}
    \E_{\mathcal{D}}\left[\left(\ell_{f,g}(X,Y)\right)\right] = \Pr_{\mathcal{D}}(\mathrm{ER}(f)\cap\Delta)-\Pr_{\mathcal{D}}(\mathrm{ER}(g)\cap\Delta) \;\;\mathrm{and}\;\; \E_{\mathcal{D}}\left[\left(\ell_{f,g}(X,Y)\right)^{2}\right] = \Pr_{\mathcal{D}}(\{f\neq g\}\cap\Delta) .
\end{equation*}
A standard route is then to apply the stated uniform entropy inequality to the following class of zero-mean functions: fix some $r>0$ (choosing various levels of $r$ for the purpose of localization), define
\begin{equation*} 
    \mathcal{F}_{r} = \left\{(x,y) \mapsto (1/2)(f(x,y)-\E_{\mathcal{D}}[f(X,Y)]): f \in\mathcal{F},\;\; \E_{\mathcal{D}}[(f(X,Y))^{2}]\leq 4r\right\} .
\end{equation*}
Note that $\mathcal{F}_{r}$ is measurable with an envelope function $F_{\Delta}=\mathbbm{1}(\Delta)$ satisfying $\lVert F\rVert_{\mathcal{D},2}=\sqrt{\Pr_{\mathcal{D}}(\Delta)}$ (note that we can only have an envelope function $F=1$ if without restricting to $\Delta$). Now, applying the above uniform entropy inequality with $\eta=\sqrt{r/\Pr_{\mathcal{D}}(\Delta)}$ together with the Bousquet's inequality would give us, with probability at least $1-r\delta$,
\begin{align*}
    &\sup_{\substack{f,g\in\hpc:\\\Pr_{\mathcal{D}}(\{f\neq g\}\cap\Delta)\leq4r}}\Big|\big(\Pr_{\mathcal{D}}(\mathrm{ER}(f)\cap\Delta)-\Pr_{\mathcal{D}}(\mathrm{ER}(g)\cap\Delta)\big) - \big(\hat{\Pr}_{S_{m}}(\mathrm{ER}(f)\cap\Delta)-\hat{\Pr}_{S_{m}}(\mathrm{ER}(g)\cap\Delta)\big)\Big|  \\ 
    \lesssim &\sqrt{\frac{r}{m}\left(d\log\left(\frac{\Pr_{\mathcal{D}}(\Delta)}{r} \land \frac{m\Pr_{\mathcal{D}}(\Delta)}{d}\right)+\log\left(\frac{1}{\delta}\right)\right)}+\frac{1}{m}\left(d\log\left(\frac{\Pr_{\mathcal{D}}(\Delta)}{r} \land \frac{m\Pr_{\mathcal{D}}(\Delta)}{d}\right)+\log\left(\frac{1}{\delta}\right)\right) .
\end{align*}
Finally, the lemma follows from applying standard localization technique.
\end{proof}

\begin{lemma}  [\textbf{ERM Bound under Tsybakov Noise Condition}]
  \label{lem:excess-risk-under-tsybakov-noise}
Let $h^{*}_{\mathcal{D}}\in\hpc$ be the Bayes-optimal classifier w.r.t.\ $\mathcal{D}$. Let $d=\vcdim$ and $a>0$. Assume that $\mathcal{D}$ satisfies a Tsybakov noise condition with parameters $(a,\alpha)$ w.r.t.\ $\hpc$ (Definition~\ref{def:Tsybakov-noise-condition}). Let $\Delta\subseteq\mathcal{X}$ be some subset of the instance space. Let $m\in\naturalnumber$. Let $\mathrm{ERM}(\hpc;S_{m})=\{h\in\hpc: h=\argmin_{h\in\hpc}\hat{\Pr}_{S_{m}}(\mathrm{ER}(h))\}$ denote the class of empirical risk minimizers with $m$ i.i.d.\ samples $S_{m}=\{(x_{i}, y_{i})\}_{i=1}^{m}\sim\mathcal{D}^{m}$. Then, for any $\delta\in(0,1)$, we have
\begin{equation*}
    \sup_{h\in\mathrm{ERM}(\hpc;S_{m})}\Pr_{\mathcal{D}}\left(\mathrm{ER}(h)\cap\Delta\right) - \Pr_{\mathcal{D}}\left(\mathrm{ER}(h^{*}_{\mathcal{D}})\cap\Delta\right) \leq 2c_0\left(\frac{ad}{m}\log\left(\left(\frac{m}{d}\right)^{\frac{\alpha}{2-\alpha}}\Pr_{\mathcal{D}}(\Delta)\right)+\frac{a}{m}\log\left(\frac{1}{\delta}\right)\right)^{\frac{1}{2-\alpha}} ,
\end{equation*}
for some numerical constant $c_0$, with probability at least $1-\delta$.
\end{lemma}

\begin{proof}[Proof of Lemma~\ref{lem:excess-risk-under-tsybakov-noise}]
This bound is a direct implication of Equation~\eqref{eq:erm-guanrantee-BCC-with-disagreement-coefficient}. This is because a Tsybakov noise condition with parameter $(a,\alpha)$ implies a Bernstein class condition with exponent $\alpha$, and the proof follows from applying the same argument for bounding the disagreement coefficient in terms of $\Pr_{\mathcal{D}}(\Delta)$.

As a complement, we show that this bound is consistent with our uniform Bernstein inequality. Since $h^{*}_{\mathcal{D}}\in\hpc$, we can choose $f=h$ to be any ERM learner and $g=h^{*}_{\mathcal{D}}$, and derive from Lemma~\ref{lem:uniform-bernstein-inequality} (by dropping the lower-order term) that
\begin{align}
  \label{eq:proof-lem-excess-risk-under-tsybakov-noise-step1}
    &\sup_{h\in\mathrm{ERM}(\hpc;S_{m})}\Pr_{\mathcal{D}}\left(\mathrm{ER}(h)\cap\Delta\right) - \Pr_{\mathcal{D}}\left(\mathrm{ER}(h^{*}_{\mathcal{D}})\cap\Delta\right) \nonumber \\
    \leq &2c_{0}\sqrt{\frac{\Pr_{\mathcal{D}}(\{h\neq h^{*}_{\mathcal{D}}\}\cap\Delta)}{m}\left(d\log\left(\frac{\Pr_{\mathcal{D}}(\Delta)}{\Pr_{\mathcal{D}}(\{h\neq h^{*}_{\mathcal{D}}\}\cap\Delta)}\right)+\log\left(\frac{1}{\delta}\right)\right)} ,
\end{align}
for the same $c_{0}>0$ therein, with probability at least $1-\delta$.

Now, we follow the same argument of considering the concept class $\hpc_{\Delta}=\{h_{\Delta}:x \mapsto h(x)\mathbbm{1}(x\in\Delta), \forall h\in\hpc\}$. We first show that $\mathcal{D}$ also satisfies a Tsybakov noise condition with the same parameters $(a,\alpha)$ w.r.t.\ $\hpc_{\Delta}$. Since $h^{*}_{\mathcal{D}}\in\hpc$ is Bayes-optimal, $h^{*}_{\mathcal{D},\Delta}=h^{*}_{\mathcal{D}}\cdot\mathbbm{1}(\Delta)\in\hpc_{\Delta}$ is also Bayes-optimal w.r.t.\ $\mathcal{D}$. This is because all concepts agree in $\mathcal{X}\setminus\Delta$. Moreover, for any $t>0$, we still have $\Pr_{\mathcal{D}}(|\eta_{\mathcal{D}}(x)-1/2| \leq t) \leq at^{\frac{\alpha}{1-\alpha}}$. A well-known result is that a Tsybakov noise condition implies a Bernstein class condition, i.e., 
\begin{equation*}
    \Pr_{\mathcal{D}}\left(\{h_{\Delta} \neq h^{*}_{\mathcal{D},\Delta}\}\right) \leq a\left(\Pr_{\mathcal{D}}(\mathrm{ER}(h_{\Delta}))-\Pr_{\mathcal{D}}(\mathrm{ER}(h^{*}_{\mathcal{D},\Delta}))\right)^{\alpha} ,
\end{equation*}
for any $h_{\Delta}\in\hpc_{\Delta}$.  
For any distribution $\mathcal{D}$ satisfies a Bernstein class condition with parameters $(a,\alpha)$ w.r.t.\ $\hpc$, we have
\begin{equation*}
    \sup_{h\in\mathrm{ERM}(\hpc;S_{m})}\Pr_{\mathcal{D}}\left(\mathrm{ER}(h)\right) - \inf_{h\in\hpc}\Pr_{\mathcal{D}}\left(\mathrm{ER}(h)\right) \lesssim \left(\frac{a}{m}\left(d\log\left(\theta_{\hpc}\left(a\left(\frac{ad}{m}\right)^{\frac{\alpha}{2-\alpha}}\right)\right)+\log\left(\frac{1}{\delta}\right)\right)\right)^{\frac{1}{2-\alpha}} ,
\end{equation*}
with probability at least $1-\delta$. 
For the sake of applying union bound, we have with probability at least $1-\delta/2$,
\begin{equation*}
    \sup_{h_{\Delta}\in\mathrm{ERM}(\hpc_{\Delta};S_{m})}\Pr_{\mathcal{D}}\left(\mathrm{ER}(h_{\Delta})\right) - \Pr_{\mathcal{D}}\left(\mathrm{ER}(h^{*}_{\mathcal{D},\Delta})\right) \lesssim \left(\frac{d}{m}\log\left(\left(\frac{m}{d}\right)^{\frac{\alpha}{2-\alpha}}\Pr_{\mathcal{D}}(\Delta)\right)+\frac{1}{m}\log\left(\frac{2}{\delta}\right)\right)^{\frac{1}{2-\alpha}} .
\end{equation*}
Putting together, we have that for any $h\in\mathrm{ERM}(\hpc;S_{m})$, with probability at least $1-\delta/2$,
\begin{align*}
    \Pr_{\mathcal{D}}(\{h\neq h^{*}_{\mathcal{D}}\}\cap\Delta) = & \Pr_{\mathcal{D}}\left(\{h_{\Delta} \neq h^{*}_{\mathcal{D},\Delta}\}\right) \leq a\left(\Pr_{\mathcal{D}}(\mathrm{ER}(h_{\Delta}))-\Pr_{\mathcal{D}}(\mathrm{ER}(h^{*}_{\mathcal{D},\Delta}))\right)^{\alpha} \\
    \lesssim & \left(\frac{a}{m}\left(d\log\left(\left(\frac{1}{a}\right)^{\frac{2}{2-\alpha}}\left(\frac{m}{d}\right)^{\frac{\alpha}{2-\alpha}}\Pr_{\mathcal{D}}(\Delta)\right)+\log\left(\frac{2}{\delta}\right)\right)\right)^{\frac{\alpha}{2-\alpha}} ,\\
    \lesssim & \left(\frac{a}{m}\left(d\log\left(\left(\frac{m}{d}\right)^{\frac{\alpha}{2-\alpha}}\Pr_{\mathcal{D}}(\Delta)\right)+\log\left(\frac{1}{\delta}\right)\right)\right)^{\frac{\alpha}{2-\alpha}} ,
\end{align*}
for any sufficiently large $m$.
Finally, substituting the above bound into \eqref{eq:proof-lem-excess-risk-under-tsybakov-noise-step1} with also $\delta/2$ failure probability, we obtain
\begin{align*}
    &\sup_{h\in\mathrm{ERM}(\hpc;S_{m})}\Pr_{\mathcal{D}}\left(\mathrm{ER}(h)\cap\Delta\right) - \Pr_{\mathcal{D}}\left(\mathrm{ER}(h^{*}_{\mathcal{D}})\cap\Delta\right) \\
    \lesssim &\sqrt{\frac{1}{m}\left(\frac{d}{m}\log\left(\left(\frac{m}{d}\right)^{\frac{\alpha}{2-\alpha}}\Pr_{\mathcal{D}}(\Delta)\right)+\frac{1}{m}\log\left(\frac{1}{\delta}\right)\right)^{\frac{\alpha}{2-\alpha}}\left(d\log\left(\left(\frac{m}{d}\right)^{\frac{\alpha}{2-\alpha}}\Pr_{\mathcal{D}}(\Delta)\right)+\log\left(\frac{1}{\delta}\right)\right)} \\
    \leq &2c_0\left(\frac{ad}{m}\log\left(\left(\frac{m}{d}\right)^{\frac{\alpha}{2-\alpha}}\Pr_{\mathcal{D}}(\Delta)\right)+\frac{a}{m}\log\left(\frac{1}{\delta}\right)\right)^{\frac{1}{2-\alpha}} ,
\end{align*}
for a numerical constant $c_0$, with probability at least $1-\delta$ from union bound.
\end{proof}

\begin{remark}
\label{rm:tsybakov-erm-bound}
We present a slightly refined version of Lemma~\ref{lem:excess-risk-under-tsybakov-noise} for the purpose of our analysis. Specifically, the high probability bound in Lemma~\ref{lem:excess-risk-under-tsybakov-noise} can be replaced with 
\begin{equation*}
    \sup_{h\in\mathrm{ERM}(\hpc;S_{m})}\hspace{-0.5cm}\Pr_{\mathcal{D}}\left(\mathrm{ER}(h)\cap\Delta\right) - \Pr_{\mathcal{D}}\left(\mathrm{ER}(h^{*}_{\mathcal{D}})\cap\Delta\right) \lesssim \left(\frac{ad}{m}\log\left(\left(\frac{m}{d+\log(1/\delta)}\right)^{\frac{\alpha}{2-\alpha}}\Pr_{\mathcal{D}}(\Delta)\right)+\frac{a}{m}\log\left(\frac{1}{\delta}\right)\right)^{\frac{1}{2-\alpha}} .
\end{equation*}
To justify such a refinement, we recall from \cite[Theorem 7]{hanneke2019surrogate} that it suffices to have sample size
\begin{equation*}
    m \gtrsim \frac{1}{\epsilon^{2-\alpha}}\left(d\log\left(\theta_{\hpc}\left(a\epsilon^{\alpha}\right)\right)+\log\left(\frac{1}{\delta}\right)\right)
\end{equation*}
to achieve $\epsilon>0$ excess risk by ERM with probability at least $1-\delta$ under Tsybakov noise. Our next goal is to find a refined bound on $\epsilon$ in terms of $d,m$ and $\delta$ such that the above holds. First, observe that it suffices to have
\begin{equation*}
    \epsilon \gtrsim \left(\frac{d}{m}\log\left(\theta_{\hpc}\left(a\epsilon^{\alpha}\right)\right)+\frac{1}{m}\log\left(\frac{1}{\delta}\right)\right)^{\frac{1}{2-\alpha}} .
\end{equation*}
Note that a trivial implication of the above is 
\begin{equation*}
    \epsilon \gtrsim \left(\frac{d}{m}+\frac{1}{m}\log\left(\frac{1}{\delta}\right)\right)^{\frac{1}{2-\alpha}} .
\end{equation*}
Now, due to the monotonicity of the disagreement coefficient \cite{hanneke2009theoretical,hanneke2014theory}, it further suffices to have
\begin{equation*}
    \epsilon \gtrsim \left(\frac{d}{m}\log\left(\theta_{\hpc}\left(a\left(\frac{d+\log(1/\delta)}{m}\right)^{\frac{\alpha}{2-\alpha}}\right)\right)+\frac{1}{m}\log\left(\frac{1}{\delta}\right)\right)^{\frac{1}{2-\alpha}} .
\end{equation*}
This implies that having $m$ (solving the above equation) samples suffices to yield an excess risk of 
\begin{equation*}
    \epsilon = \left(\frac{d}{m}\log\left(\theta_{\hpc}\left(a\left(\frac{d+\log(1/\delta)}{m}\right)^{\frac{\alpha}{2-\alpha}}\right)\right)+\frac{1}{m}\log\left(\frac{1}{\delta}\right)\right)^{\frac{1}{2-\alpha}} .
\end{equation*}
Together with the argument for bounding the disagreement coefficient, this bound yields the desired rates.
\end{remark}

\begin{lemma}  
  \label{lem:pseudo-distance-estimation}
Let $\hpc$ be a concept class with $\vcdim=d<\infty$. Let $\Delta\subseteq\mathcal{X}$ be some subset of the instance space. Let $\mathcal{D}$ be any distribution over $\mathcal{X}\times\{0,1\}$. Let $m\in\naturalnumber$ and $S_{m}=\{(x_{i}, y_{i})\}_{i=1}^{m}\sim\mathcal{D}^{m}$. There exists a universal constant $c_{0}>0$ such that, for any $\delta\in(0,1)$ and any pair $f,g\in\hpc$, we have
\begin{equation*}
    \left|\Pr_{\mathcal{D}}(\{f\neq g\}\cap\Delta) - \hat{\Pr}_{S_{m}}(\{f\neq g\}\cap\Delta)\right| \leq c_{0}\sqrt{\frac{\Pr_{\mathcal{D}}(\{f\neq g\}\cap\Delta)(d\log(m/d)+\log(1/\delta))}{m}} ,
\end{equation*}
with probability at least $1-\delta$.
\end{lemma}

\begin{proof}[Proof of Lemma~\ref{lem:pseudo-distance-estimation}]
This result is a uniform variant of the multiplicative Chernoff bound, which is a direct implication of the uniform Bernstein inequality. Let us assume w.l.o.g.\ that the label space is $\mathcal{Y}=\{-1,1\}$. Consider the hypothesis class $\tilde{\hpc}=\hpc\oplus\hpc$ consisting of all concepts $h_{f,g}(x)=f(x)g(x):\mathcal{X}\mapsto\{-1,1\}$ for any pair of $f,g\in\hpc$. Note that $\mathrm{VC}(\tilde{\hpc})=O(d)$. Furthermore, note that for any distribution $\mathcal{D}$, $\Pr_{\mathcal{D}}(\{h_{f,g}(x)\neq 1\}\cap\Delta)=\Pr_{\mathcal{D}}(\{f\neq g\}\cap\Delta)$. Indeed, for any input domain $\mathcal{X}$, concept class $\hpc$ with $\vcdim=d$, distribution $\mathcal{D}$ over $\mathcal{X}\times\{-1,1\}$, $m\in\naturalnumber$ and $\delta\in(0,1)$, \cite{hanneke2024revisiting} proves that, with probability at least $1-\delta$ over samples $S_{m}\sim\mathcal{D}^{m}$, every $h\in\hpc$ satisfies 
\begin{align*}
    &\left|\Pr_{\mathcal{D}}(\mathrm{ER}(h))-\hat{\Pr}_{S_{m}}(\mathrm{ER}(h))\right| \\
    \lesssim &\sqrt{\frac{\Pr_{\mathcal{D}}(\mathrm{ER}(h))(d\log((1/\Pr_{\mathcal{D}}(\mathrm{ER}(h)) \land (m/d)))+\log(1/\delta))}{m}} + \frac{d\log(m/d)+\log(1/\delta)}{m} .
\end{align*}
We invoke the above bound for $\tilde{\hpc}$ and a distribution $\tilde{\mathcal{D}}$ with $\tilde{\mathcal{D}}_{\mathcal{X}}=\mathcal{D}_\mathcal{X}$ and $\Pr_{\tilde{\mathcal{D}}}(Y=1|X)=1$. We will apply it to the samples from the conditional distribution given $\Delta$ (the same technique has been used in \cite{hanneke2025agnostic} for deriving a conditional uniform convergence bound of \cite{talagrand1994sharper}) and get, for any pair of $f,g\in\hpc$,
\begin{align*}
    &\left|\Pr_{\mathcal{D}}(\{f\neq g\}\cap\Delta) - \hat{\Pr}_{S_{m}}(\{f\neq g\}\cap\Delta)\right| = \left|\Pr_{\tilde{\mathcal{D}}}(\mathrm{ER}(h_{f,g})\cap\Delta)-\hat{\Pr}_{\tilde{\mathcal{D}}_{S_{m}}}(\mathrm{ER}(h_{f,g})\cap\Delta)\right| \\
    \leq &c_{0}\sqrt{\frac{\Pr_{\mathcal{D}}(\{f\neq g\}\cap\Delta) \log(m\Pr_{\mathcal{D}}(\Delta)/d)+\log(1/\delta))}{m}} + \frac{c_{0}(d\log(m\Pr_{\mathcal{D}}(\Delta)/d)+\log(1/\delta))}{m} \\
\end{align*}
for some universal constant $c_{0}>0$, with probability at least $1-\delta$. We would like to point out that this lemma can also be proved based on our uniform Bernstein inequality (Lemma~\ref{lem:uniform-bernstein-inequality}) with the same constant $c_{0}$ therein. 
\end{proof}

\newpage
\bibliographystyle{alpha}
\bibliography{optimal_tsybakov}

@article{tsybakov2004optimal,
  title={Optimal aggregation of classifiers in statistical learning},
  author={Tsybakov, Alexander B},
  journal={The Annals of Statistics},
  volume={32},
  number={1},
  pages={135--166},
  year={2004}
}

@article{mammen1999smooth,
  title={Smooth discrimination analysis},
  author={Mammen, Enno and Tsybakov, Alexandre B},
  journal={The Annals of Statistics},
  volume={27},
  number={6},
  pages={1808--1829},
  year={1999},
  publisher={Institute of Mathematical Statistics}
}

@article{massart2006risk,
  title={Risk bounds for statistical learning},
  author={Massart, Pascal and N\'{e}d\'{e}lec, \'{E}lodie},
  journal={The Annals of Statistics},
  volume={34},
  number={5},
  pages={2326--2366},
  year={2006}
}

@article{gine2006concentration,
  title={Concentration Inequalities and Asymptotic Results for Ratio Type Empirical Processes},
  author={Gin{\'e}, Evarist and Koltchinskii, Vladimir},
  journal={The Annals of Probability},
  pages={1143--1216},
  year={2006}
}

@article{koltchinskii2006local,
  title={Local Rademacher Complexities and Oracle Inequalities in Risk Minimization},
  author={Koltchinskii, Vladimir},
  journal={The Annals of Statistics},
  pages={2593--2656},
  year={2006}
}

@article{hanneke2016refined,
  title={Refined error bounds for several learning algorithms},
  author={Hanneke, Steve},
  journal={Journal of Machine Learning Research},
  volume={17},
  number={135},
  pages={1--55},
  year={2016}
}

@inproceedings{hanneke2025universal,
  title={Universal Rates of ERM for Agnostic Learning},
  author={Hanneke, Steve and Xu, Mingyue},
  booktitle={Proceedings of the 38th Annual Conference on Learning Theory},
  pages={2666--2703},
  volume={291},
  year={2025},
  organization={PMLR}
}

@article{hanneke2025agnostic,
  title={Agnostic Active Learning Is Always Better Than Passive Learning},
  author={Hanneke, Steve},
  journal={Advances in Neural Information Processing Systems},
  year={2025}
}

@article{hanneke2022no,
  title={A no-free-lunch theorem for multitask learning},
  author={Hanneke, Steve and Kpotufe, Samory},
  journal={The Annals of Statistics},
  volume={50},
  number={6},
  pages={3119--3143},
  year={2022}
}

@book{van-der-Vaart96,
  author={van der Vaart, A. W. and Wellner, J. A.},
  title={Weak Convergence and Empirical Processes},
  publisher={Springer},
  year={1996}
}

@article{van2011local,
  title={A local maximal inequality under uniform entropy},
  author={van der Vaart, A. W. and Wellner, J. A.},
  journal={Electronic Journal of Statistics},
  volume={5},
  number={2011},
  pages={192},
  year={2011}
}

@article{bousquet2002bennett,
  title={A \uppercase{B}ennett concentration inequality and its application to suprema of empirical processes},
  author={Bousquet, Olivier},
  journal={Comptes Rendus Mathematique},
  volume={334},
  number={6},
  pages={495--500},
  year={2002}
}

@article{bernstein1924modification,
  title={On a modification of Chebyshev’s inequality and of the error formula of Laplace},
  author={Bernstein, Sergei},
  journal={Ann. Sci. Inst. Sav. Ukraine, Sect. Math},
  volume={1},
  number={4},
  pages={38--49},
  year={1924}
}

@article{bennett1962probability,
  title={Probability inequalities for the sum of independent random variables},
  author={Bennett, George},
  journal={Journal of the American Statistical Association},
  volume={57},
  number={297},
  pages={33--45},
  year={1962}
}

@inproceedings{hanneke2024revisiting,
  title={Revisiting agnostic pac learning},
  author={Hanneke, Steve and Larsen, Kasper Green and Zhivotovskiy, Nikita},
  booktitle={2024 IEEE 65th Annual Symposium on Foundations of Computer Science (FOCS)},
  pages={1968--1982},
  year={2024},
  organization={IEEE}
}

@article{talagrand1994sharper,
  title={Sharper bounds for Gaussian and empirical processes},
  author={Talagrand, Michel},
  journal={The Annals of Probability},
  pages={28--76},
  year={1994},
  publisher={JSTOR}
}

@article{auer2007new,
  title={A new PAC bound for intersection-closed concept classes},
  author={Auer, Peter and Ortner, Ronald},
  journal={Machine Learning},
  volume={66},
  number={2},
  pages={151--163},
  year={2007},
  publisher={Springer}
}

@article{bousquet2021fast,
  title={Fast classification rates without standard margin assumptions},
  author={Bousquet, Olivier and Zhivotovskiy, Nikita},
  journal={Information and Inference: A Journal of the IMA},
  volume={10},
  number={4},
  pages={1389--1421},
  year={2021}
}

@inproceedings{puchkin2021exponential,
  title={Exponential savings in agnostic active learning through abstention},
  author={Puchkin, Nikita and Zhivotovskiy, Nikita},
  booktitle={Proceedings of the 34th Annual Conference on Learning Theory},
  pages={3806--3832},
  year={2021},
  organization={PMLR}
}

@article{boucheron2005theory,
  title={Theory of classification: A survey of some recent advances},
  author={Boucheron, St{\'e}phane and Bousquet, Olivier and Lugosi, G{\'a}bor},
  journal={ESAIM: Probability and Statistics},
  volume={9},
  pages={323--375},
  year={2005},
  publisher={EDP Sciences}
}

@article{valiant1984theory,
  title={A theory of the learnable},
  author={Valiant, Leslie G},
  journal={Communications of the ACM},
  volume={27},
  number={11},
  pages={1134--1142},
  year={1984},
  publisher={ACM New York, NY, USA}
}

@article{vapnik1974theory,
  title={Theory of pattern recognition},
  author={Vapnik, Vladimir and Chervonenkis, Alexey},
  year={1974},
  journal={Nauka, Moscow}
}

@article{vapnik1964class,
  title={A class of algorithms for pattern recognition learning},
  author={Vapnik, Vladimir and Chervonenkis, A Ya},
  journal={Avtomat. i Telemekh},
  volume={25},
  number={6},
  pages={937--945},
  year={1964}
}

@article{blumer1989learnability,
  title={Learnability and the Vapnik-Chervonenkis dimension},
  author={Blumer, Anselm and Ehrenfeucht, Andrzej and Haussler, David and Warmuth, Manfred K},
  journal={Journal of the ACM (JACM)},
  volume={36},
  number={4},
  pages={929--965},
  year={1989},
  publisher={ACM New York, NY, USA}
}

@article{ehrenfeucht1989general,
  title={A general lower bound on the number of examples needed for learning},
  author={Ehrenfeucht, Andrzej and Haussler, David and Kearns, Michael and Valiant, Leslie},
  journal={Information and Computation},
  volume={82},
  number={3},
  pages={247--261},
  year={1989},
  publisher={Elsevier}
}

@article{hanneke2015minimax,
  title={Minimax analysis of active learning.},
  author={Hanneke, Steve and Yang, Liu},
  journal={Journal of Machine Learning Research},
  volume={16},
  number={1},
  pages={3487--3602},
  year={2015}
}

@article{haussler1994predicting,
  title={Predicting $\{$0, 1$\}$-functions on randomly drawn points},
  author={Haussler, David and Littlestone, Nick and Warmuth, Manfred K},
  journal={Information and Computation},
  volume={115},
  number={2},
  pages={248--292},
  year={1994},
  publisher={Elsevier}
}

@inproceedings{natarajan1987learning,
  title={On learning boolean functions},
  author={Natarajan, Balaubramaniam Kausik},
  booktitle={Proceedings of the 19th Annual ACM Symposium on Theory of Computing},
  pages={296--304},
  year={1987}
}

@inproceedings{warmuth2004optimal,
  title={The optimal PAC algorithm},
  author={Warmuth, Manfred K},
  booktitle={Proceedings of the 17th International Conference on Computational Learning Theory},
  pages={641--642},
  year={2004}
}

@article{hanneke2016optimal,
  title={The optimal sample complexity of PAC learning},
  author={Hanneke, Steve},
  journal={Journal of Machine Learning Research},
  volume={17},
  number={38},
  pages={1--15},
  year={2016}
}

@inproceedings{bousquet2020proper,
  title={Proper learning, Helly number, and an optimal SVM bound},
  author={Bousquet, Olivier and Hanneke, Steve and Moran, Shay and Zhivotovskiy, Nikita},
  booktitle={Proceedings of the 33rd Annual Conference on Learning Theory},
  pages={582--609},
  year={2020},
  organization={PMLR}
}

@inproceedings{simon2015almost,
  title={An almost optimal PAC algorithm},
  author={Simon, Hans U},
  booktitle={Proceedings of the 28th Annual Conference on Learning Theory},
  pages={1552--1563},
  year={2015},
  organization={PMLR}
}

@inproceedings{larsen2023bagging,
  title={Bagging is an optimal PAC learner},
  author={Larsen, Kasper Green},
  booktitle={Proceedings of the 36th Annual Conference on Learning Theory},
  pages={450--468},
  year={2023},
  organization={PMLR}
}

@inproceedings{aden2024majority,
  title={Majority-of-three: The simplest optimal learner?},
  author={Aden-Ali, Ishaq and H{\o}andgsgaard, Mikael M{\o}ller and Larsen, Kasper Green and Zhivotovskiy, Nikita},
  booktitle={Proceedings of the 37th Annual Conference on Learning Theory},
  pages={22--45},
  year={2024},
  organization={PMLR}
}

@inproceedings{awasthi2015efficient,
  title={Efficient learning of linear separators under bounded noise},
  author={Awasthi, Pranjal and Balcan, Maria-Florina and Haghtalab, Nika and Urner, Ruth},
  booktitle={Proceedings of the 28th Annual Conference on Learning Theory},
  pages={167--190},
  year={2015},
  organization={PMLR}
}

@inproceedings{aden2023optimal,
  title={Optimal pac bounds without uniform convergence},
  author={Aden-Ali, Ishaq and Cherapanamjeri, Yeshwanth and Shetty, Abhishek and Zhivotovskiy, Nikita},
  booktitle={2023 IEEE 64th Annual Symposium on Foundations of Computer Science (FOCS)},
  pages={1203--1223},
  year={2023},
  organization={IEEE}
}

@inproceedings{sloan1988types,
  title={Types of noise in data for concept learning},
  author={Sloan, Robert},
  booktitle={Proceedings of the 1st Annual Workshop on Computational Learning Theory},
  pages={91--96},
  year={1988}
}

@article{bartlett2006convexity,
  title={Convexity, classification, and risk bounds},
  author={Bartlett, Peter L and Jordan, Michael I and McAuliffe, Jon D},
  journal={Journal of the American Statistical Association},
  volume={101},
  number={473},
  pages={138--156},
  year={2006},
  publisher={Taylor \& Francis}
}

@article{bartlett2006empirical,
  title={Empirical minimization},
  author={Bartlett, Peter L and Mendelson, Shahar},
  journal={Probability Theory and Related Fields},
  volume={135},
  number={3},
  pages={311--334},
  year={2006},
  publisher={Springer}
}

@inproceedings{balcan2007margin,
  title={Margin based active learning},
  author={Balcan, Maria-Florina and Broder, Andrei and Zhang, Tong},
  booktitle={Proceedings of the 20th International Conference on Computational Learning Theory},
  pages={35--50},
  year={2007},
  organization={Springer}
}

@article{hanneke2011rates,
  title={Rates of convergence in active learning},
  author={Hanneke, Steve},
  journal={The Annals of Statistics},
  pages={333--361},
  year={2011},
  publisher={JSTOR}
}

@article{zhang2020efficient,
  title={Efficient active learning of sparse halfspaces with arbitrary bounded noise},
  author={Zhang, Chicheng and Shen, Jie and Awasthi, Pranjal},
  journal={Advances in Neural Information Processing Systems},
  volume={33},
  pages={7184--7197},
  year={2020}
}

@article{van2015fast,
  title={Fast Rates in Statistical and Online Learning},
  author={Van Erven, Tim and Gr{\"u}nwald, Peter D and Mehta, Nishant A and Reid, Mark D and Williamson, Robert C},
  journal={Journal of Machine Learning Research},
  volume={16},
  pages={1793--1861},
  year={2015}
}

@inproceedings{bartlett2004local,
  title={Local complexities for empirical risk minimization},
  author={Bartlett, Peter L and Mendelson, Shahar and Philips, Petra},
  booktitle={Proceedings of the 17th International Conference on Computational Learning Theory},
  pages={270--284},
  year={2004}
}

@article{hanneke2019surrogate,
  title={Surrogate losses in passive and active learning},
  author={Hanneke, Steve and Yang, Liu},
  journal={Electronic Journal of Statistics},
  volume={13},
  pages={4646--4708},
  year={2019}
}

@article{hanneke2014theory,
  title={Theory of disagreement-based active learning},
  author={Hanneke, Steve},
  journal={Foundations and Trends in Machine Learning},
  volume={7},
  number={2-3},
  pages={131--309},
  year={2014},
  publisher={Emerald Publishers Limited}
}

@book{hanneke2009theoretical,
  title={Theoretical foundations of active learning},
  author={Hanneke, Steve},
  year={2009},
  publisher={Carnegie Mellon University}
}

@inproceedings{hanneke2026more,
  title={When More Data Doesn't Help: Limits of Adaptation in Multitask Learning},
  author={Hanneke, Steve and Xu, Mingyue},
  booktitle={Proceedings of the 43rd International Conference on Machine Learning},
  year={2026}
}

@article{hanneke2019value,
  title={On the value of target data in transfer learning},
  author={Hanneke, Steve and Kpotufe, Samory},
  journal={Advances in Neural Information Processing Systems},
  volume={32},
  year={2019}
}

@article{angluin1988learning,
  title={Learning from noisy examples},
  author={Angluin, Dana and Laird, Philip},
  journal={Machine Learning},
  volume={2},
  number={4},
  pages={343--370},
  year={1988},
  publisher={Springer}
}

@article{haussler1992decision,
  title={Decision theoretic generalizations of the PAC model for neural net and other learning applications},
  author={Haussler, David},
  journal={Information and Computation},
  volume={100},
  number={1},
  pages={78--150},
  year={1992},
  publisher={Elsevier}
}

@inproceedings{kearns1992toward,
  title={Toward efficient agnostic learning},
  author={Kearns, Michael J and Schapire, Robert E and Sellie, Linda M},
  booktitle={Proceedings of the 5th Annual Workshop on Computational Learning Theory},
  pages={341--352},
  year={1992}
}

@inproceedings{diakonikolas2021efficiently,
  title={Efficiently learning halfspaces with tsybakov noise},
  author={Diakonikolas, Ilias and Kane, Daniel M and Kontonis, Vasilis and Tzamos, Christos and Zarifis, Nikos},
  booktitle={Proceedings of the 53rd Annual ACM SIGACT Symposium on Theory of Computing},
  pages={88--101},
  year={2021}
}

@article{diakonikolas2019distribution,
  title={Distribution-independent pac learning of halfspaces with massart noise},
  author={Diakonikolas, Ilias and Gouleakis, Themis and Tzamos, Christos},
  journal={Advances in Neural Information Processing Systems},
  volume={32},
  year={2019}
}

@inproceedings{diakonikolas2020learning,
  title={Learning halfspaces with massart noise under structured distributions},
  author={Diakonikolas, Ilias and Kontonis, Vasilis and Tzamos, Christos and Zarifis, Nikos},
  booktitle={Proceedings of the 33rd Annual Conference on Learning Theory},
  pages={1486--1513},
  year={2020},
  organization={PMLR}
}

@article{audibert2007fast,
  title={Fast Learning Rates for Plug-In Classifiers},
  author={Audibert, Jean-Yves and Tsybakov, Alexandre B},
  journal={The Annals of Statistics},
  pages={608--633},
  year={2007}
}

@book{shalev2014understanding,
  title={Understanding machine learning: From theory to algorithms},
  author={Shalev-Shwartz, Shai and Ben-David, Shai},
  year={2014},
  publisher={Cambridge university press}
}

\end{document}